\documentclass{sig-alternate-05-2015}

\usepackage{times}
\usepackage{graphicx} 
\usepackage[usenames,dvipsnames,svgnames,table]{xcolor}
\usepackage{framed}
\usepackage{subfigure} 
\usepackage{enumitem}
\usepackage{multirow}
\usepackage[numbers,sort&compress]{natbib}
\usepackage{balance}

\usepackage{amsmath}
\usepackage{amssymb}
\usepackage{amsthm}

\usepackage{algorithm}
\usepackage{algorithmic}

\newcommand{\D}{\mathcal{D}}
\newcommand{\w}{\boldsymbol{w}}
\newcommand{\wbarb}{\bar{\boldsymbol{w}}}
\newcommand{\wbarbR}{\bar{\boldsymbol{w}}^{R}}

\newcommand{\x}{\boldsymbol{x}}

\newcommand{\yh}{\hat{y}}
\newcommand{\z}{\boldsymbol{z}}

\newcommand{\p}{\boldsymbol{p}}

\newcommand{\0}{\boldsymbol{0}}

\newcommand{\g}{\boldsymbol{g}}

\newcommand{\eps}{\epsilon}
\newcommand{\phib}{\bar{\phi}}
\newcommand{\bphi}{\boldsymbol{\phi}}
\newcommand{\bphib}{\bar{\bphi}}

\newcommand{\Hc}{\mathcal{H}}

\newcommand{\R}{\mathbb{R}}

\newcommand{\prox}{\mathbf{prox}}

\newcommand{\brho}{\boldsymbol{\rho}}

\newcommand{\brhob}{\bar{\brho}}

\newcommand{\X}{\mathcal{X}}
\newcommand{\Y}{\mathcal{Y}}

\newcommand{\bdelta}{\boldsymbol{\delta}}
\newcommand{\A}{\mathcal{A}}
\newcommand{\cR}{\mathcal{R}}

\newcommand{\e}{\boldsymbol{e}}

\newcommand{\gbarb}{\bar{\boldsymbol{g}}}

\newcommand{\etab}{\boldsymbol{\eta}}

\newcommand{\etabarb}{\bar{\boldsymbol{\eta}}}

\definecolor{gray2}{rgb}{0.6,0.6,0.6}

\newtheorem{definition}{Definition}
\newtheorem{assumption}{Assumption}

\newtheorem{theorem}{Theorem}

\begin{document}

\setcopyright{acmcopyright}

\CopyrightYear{2016} 
\setcopyright{acmlicensed}
\conferenceinfo{KDD '16,}{August 13 - 17, 2016, San Francisco, CA, USA}
\isbn{978-1-4503-4232-2/16/08}\acmPrice{\$15.00}
\doi{http://dx.doi.org/10.1145/2939672.2939794}

%

\title{Revisiting Random Binning Features: Fast Convergence and Strong Parallelizability}

\numberofauthors{4}

\author{
\alignauthor
Lingfei Wu \titlenote{Both authors contributed equally to this manuscript}\\
       \affaddr{College of William and Mary}\\
       \affaddr{Williamsburg, VA 23185}\\
       \email{lfwu@cs.wm.edu}
\alignauthor       
Ian E.H. Yen {\Large \textsuperscript{*}}\\
       \affaddr{Unversity of Texas at Austin}\\
       \affaddr{Austin, TX 78712}\\
       \email{ianyen@cs.utexas.edu} 
\and  
\alignauthor       
Jie Chen\\
       \affaddr{IBM Research}\\
       \affaddr{Yorktown Heights, NY 10598}\\
       \email{chenjie@us.ibm.com}         
\alignauthor       
Rui Yan\\
       \affaddr{Baidu Inc.}\\
       \affaddr{Beijing 100085, China}\\
       \email{yanrui02@baidu.com}        
}

\maketitle
\begin{abstract}
Kernel method has been developed as one of the standard approaches for nonlinear learning, which however, does not scale to large data set due to its quadratic complexity in the number of samples. A number of kernel approximation methods have thus been proposed in the recent years, among which the random features method gains much popularity due to its simplicity and direct reduction of nonlinear problem to a linear one. Different random feature functions have since been proposed to approximate a variety of kernel functions. Among them the Random Binning (RB) feature, proposed in the first random-feature paper \cite{rahimi2007random}, has drawn much less attention than the Random Fourier (RF) feature proposed also in \cite{rahimi2007random}. In this work, we observe that the RB features, with right choice of optimization solver, could be orders-of-magnitude more efficient than other random features and kernel approximation methods under the same requirement of accuracy. We thus propose the first analysis of RB from the perspective of optimization, which by interpreting RB as a Randomized Block Coordinate Descent in the infinite-dimensional space, gives a faster convergence rate compared to that of other random features. In particular, we show that by drawing $R$ random grids with at least $\kappa$ number of non-empty bins per grid in expectation, RB method achieves a convergence rate of $O(1/(\kappa R))$, which not only sharpens its $O(1/\sqrt{R})$ rate from Monte Carlo analysis, but also shows a $\kappa$ times speedup over other random features under the same analysis framework. In addition, we demonstrate another advantage of RB in the L1-regularized setting, where unlike other random features, a RB-based Coordinate Descent solver can be parallelized with guaranteed speedup proportional to $\kappa$. Our extensive experiments demonstrate the superior performance of the RB features over other random features and kernel approximation methods. Our code and data is available at { \url{https://github.com/teddylfwu/RB_GEN}}.

\end{abstract}

\printccsdesc



\section{Introduction}
\label{section:introduction}
Kernel methods have great promise for learning non-linear model from simple data input representations and have been demonstrated successful for solving various learning problems such as regression, classification, feature extraction, clustering and dimensionality reduction \cite{LearnKernels2001,Taskar2003NIPS}. However, they are typically not first choice for large-scale nonlinear learning problems, since large number of samples ($N$) presents significant challenges in terms of computation and memory consumptions to Kernel methods for computing the dense kernel matrix $K \in \mathcal{R}^{N \times N}$ which requires at least a $O(N^2)$ complexity. To scale up the kernel methods, there have been great efforts addressing this challenge from various perspectives such as numerical linear algebra, sampling approximation, optimization and functional analysis \cite{rahimi2007random,Dai2014ScalableKernel,yen2014sparse,Si2014MEKA,Drineas2005Nystrom,Fine2002JMLR}.

A line of research \cite{Seeger2000Nystrom,Smola2000SparseGreedy,Fine2002JMLR,Si2014MEKA} approximates the kernel matrix $K$ using low-rank factorizations, $K \approx Z^TZ$, where $Z \in \mathcal{R}^{N \times R}$ matrix with $R \ll N$. Among them, Nystr\"{o}m method \cite{Seeger2000Nystrom,Drineas2005Nystrom,Kumar2012SamplingNystrom,Gittens2013RevisitNystrom,Si2014MEKA} is probably one of the most popular approaches, which reduces the total computational costs to $O(NRd + NR^2 + R^3)$ or $O(NRd + NRm)$ depending whether the algorithm performs on $K$ explicitly or implicitly through $Z$, where $d$ and $m$ are the input data dimension and the number of iterations of an iterative solver respectively. However, the convergence of the low-rank approximation is proportional to $O(1/\sqrt{R} + 1/\sqrt{N})$ \cite{Drineas2005Nystrom,yen2014sparse}, which implies that the rank $R$ may need to be near-linear to the number of data points in order to achieve comparable generalization error compared to the vanilla kernel method. For large-scale problems, the low-rank approximation could become almost as expensive as the exact kernel method to maintain competitive performance \cite{Stein2014BlockDiag}. 

Another popular approach for scaling up kernel method is random features approximation \cite{rahimi2007random,rahimi2008RKS}. Unlike previous approach that approximates kernel matrix, Random Features approximate the kernel function directly via sampling from an explicit feature map. Random Fourier (RF) is one of the feature maps that attracted considerable interests due to its easy implementation and fast execution time \cite{rahimi2008RKS,  Le14Fastfood, Yang15LearnFastKernel, Dai2014ScalableKernel, bengio2009binary}, which has total computational cost and storage requirement as $O(NRd + NRm)$ and $O(NR)$ respectively, for computing feature matrix $Z$ and operating the subsequent algorithms on $Z$. A \emph{Fastfood} approach and its extension \cite{Le14Fastfood,Yang15LearnFastKernel} was proposed to reduce the time of computing Fourier features from $O(Rd)$ to $O(R\log{d})$ by leveraging Hadamard basis functions, which improves the efficiency for prediction but not necessarily for training if $d\ll m$. Although RF has been successfully applied to speech recognition and vision classifications on very large datasets \cite{JCLW2016ICCASP,Huang2014Kernel,Lu14ScaleKernel}, a drawback is that a significant large number of random features are needed to achieve a comparable performance to exact kernel method. This is not surprising since the convergence of approximation error is in the order $O(1/\sqrt{R} + 1/\sqrt{N})$ \cite{rahimi2008RKS,Dai2014ScalableKernel}, which is the same as that of low-rank kernel approximations.

Mercer's theorem [19] guarantees that any positive-definite kernel permits a feature-map decomposition. However, the decomposition is not unique. One may find different feature maps to construct the same kernel function \cite{rahimi2007random,yen2014sparse}. Therefore, we ask following question: do some of the feature maps lead to faster convergence than the others in terms of approximation? In this paper, we address this question by reconsidering the Random Binning (RB) feature map, which was proposed in the first Random-Feature paper \cite{rahimi2007random} but has drawn much less attentions since then compared to the RF feature. Our main contributions are fourfold.

First, we propose the first analysis of RB from the perspective of optimization. By interpreting RB as a \emph{Randomized Block Coordinate Descent} (RBCD) in the infinite-dimensional space induced from the kernel, we prove that RB enjoys faster convergence than other random features. Specifically, by drawing $R$ grids with expected number of non-empty bins per grid lower bounded by $\kappa$, RB can achieve a solution comparable to exact kernel method with $O(1/(\kappa R))$ precision in terms of the objective function, which is not only better than the existing $O(1/\sqrt{R})$ rate from Monte Carlo analysis \cite{rahimi2007random}, but also shows a $\kappa$ times speedup over the rate of other random features under the same analysis framework \cite{yen2014sparse}.


Second, we exploit the sparse structure of the feature matrix $Z$, which is the key to rapidly transform the data features into a very high-dimension feature space that is linearly separately by any regressors and classifiers. In addition, we discuss how to efficiently perform the computation for a large, sparse matrix by using state-of-the-art iterative solvers and advanced matrix storage techniques. As a result, the computational complexity and storage requirements in training are still $O(NRd + NRm)$ and $O(NR)$, respectively. 

Third, we show that Random Binning features is particularly suitable for \emph{Parallel Coordinate Descent} solver. Unlike other random features, RB guarantees a speedup proportional to $\kappa$ due to a sparse feature matrix. This is particularly useful in the \emph{Sparse Random Feature} setting \cite{yen2014sparse}, where L1 regularization is used to induce a compact nonlinear predictor and \emph{Coordinate Descent} is presumably the state-of-the-art solver in such setting.

Finally, we provide extensive experiments to demonstrate the faster convergence and better parallelizability of RB in practice. Compared to other popular low-rank approximations, RB shows superior performance on both regression and classification tasks under the same computational budgets, and achieves same performance with one to three orders of magnitude reduction in time and memory consumptions. When combined with Coordinate Descent to solve an L1-regularized objective, RB shows an almost linear speedup, in contrast to RF that has almost no speedup. 


\section{Random Binning Feature as Kernel Approximation}
\label{section:kernel_approximation}



In this work, we consider the problem of fitting a nonlinear prediction function $f:\X\rightarrow \Y$ in Reproducing Kernel Hilbert Space $\Hc$ from training data pairs $\{(\x_n,y_n)\}_{n=1}^N$ via regularized Empirical Risk Minimization (ERM)
\begin{equation}\label{problem}
\begin{aligned}
&f^* = \underset{f\in\Hc}{argmin} &\frac{\lambda}{2} \|f\|_{\Hc}^2  + \frac{1}{N} \sum_{n=1}^N L(f(\x_n),y_n),
\end{aligned}
\end{equation}
where $L(z,y)$ is a convex loss function with Lipschitz-continuous derivative satisfying $|L'(z_1,y)-L'(z_2,y)|\leq \beta |z_1-z_2|$, which includes several standard loss functions such as the \emph{square-loss} $L(z,y)=\frac{1}{2}(z-y)^2$, \emph{square-hinge loss} $L(z,y)=\max(1-zy,0)^2$ and \emph{logistic loss} $L(z,y)=\log(1+\exp(-yz))$.

\subsection{Learning with RKHS}

The RKHS $\Hc$ can be defined via a positive-definite (PD) kernel function $k(\x_1,\x_2)$ that measures similarity between samples as
\begin{equation}\label{kernel_H}
\Hc = \left\{ f(\cdot)=\sum_{i=1}^K \alpha_i k(\x_i,\cdot)  \textit{ }|\textit{ }  \alpha_i\in\R, \x_i\in\X \right\}.
\end{equation}
One can also define the RKHS via a possibly infinite-dimensional feature map $\{\phib_h(\x)\}_{h\in H}$ with each $h\in H$ defining a feature function $\phib_h(\x):\X\rightarrow \R$. The space can be expressed as 
\begin{equation} \label{feature_map_H}
\Hc = \left\{ f(\cdot)=\int_{h\in H} w(h)\phib_h(\cdot) dh = \langle \w,\bphib(\cdot) \rangle_{\Hc} \textit{  }|\textit{  } \|f\|_{\Hc}^2 < \infty \right\},
\end{equation}
where $w(h)$ specifies weights over the set of features $\{\phi_h(\x)\}_{h\in\Hc}$. The Mercer's theorem \cite{mercer1909functions} connects the above two formulations of RKHS by stating that every PD kernel $k(.,.)$ can be expressed as an integration over some basis functions $\{\phi_h(.)\}_{h\in H}$
\begin{equation} \label{feature_map}
k(\x_1,\x_2)=\int_{h\in H} p(h) \phi_h(\x_1) \phi_h(\x_2) dh = \langle \bphib(\x_1), \bphib(\x_2)\rangle_{\Hc},
\end{equation}
However, the decomposition \eqref{feature_map} is not unique, so one can find different feature maps $\{\phib_h(.)\}_{h\in H}$ satisfying \eqref{feature_map} for the same kernel $k(.,.)$. In particular, as an example used extensively in this work, the Laplacian Kernel
\begin{equation}\label{laplacian_kernel}
k(\x_1,\x_2)=\exp\left(-\frac{\|\x_1-\x_2\|_1}{\sigma}\right),
\end{equation}
allows decomposition based on (i) Fourier basis map \cite{rahimi2007random}, (ii) RB map \cite{rahimi2007random}, and also (iii) map based on infinite number of decision trees \cite{lin2008support} to name a few. On the other hand, different kernels can be constructed using the same set of basis function $\{\phi_h(.)\}$ with different distribution $p(h)$. For example, the RB feature map can be used to construct any shift-invariant kernel of the form \cite{rahimi2007random} 
\begin{equation} \label{shift_inv_kernel}
K(\x_1,\x_2) = K(\x_1-\x_2)= \prod_{j=1}^d k_j(x_{1j}-x_{2j}),
\end{equation}
by sampling the "width" of bins $\delta_j$ for each feature $j$ from a distribution proportional to $\delta k_j''(\delta)$, where $k_j''(\delta)$ is the second derivative of $k_j(\delta)$, assuming the kernel has a non-negative second derivative.

\subsection{Random Binning Features}

\begin{figure}
    \centering
    \includegraphics[scale=0.33]{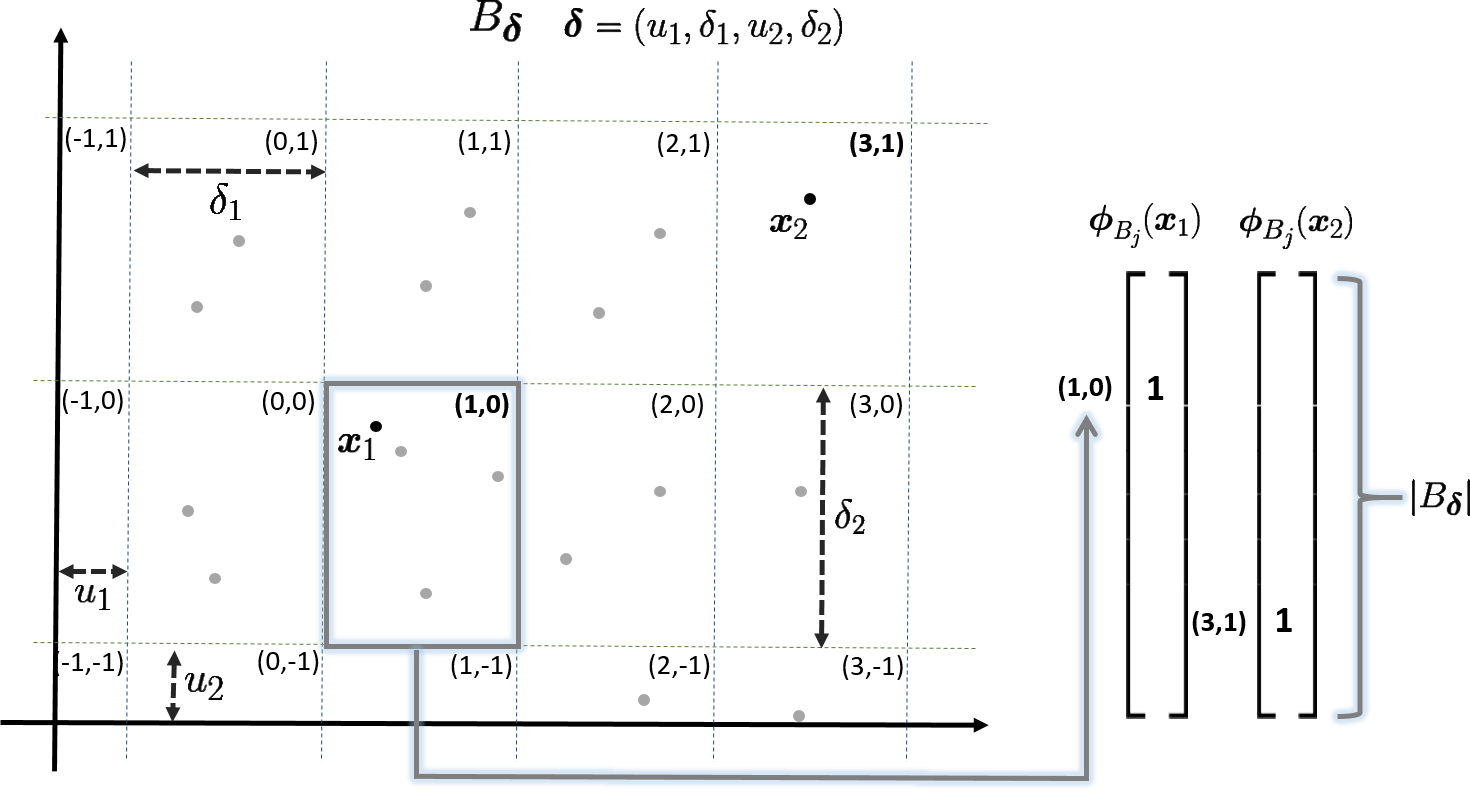}
    \caption{Generating process of RB features.}
    \label{fig:RB_gen}
\end{figure}

In this section, we describe the Random Binning (RB) feature map, which has decomposition of the form
\begin{equation}\label{RB_feature_map}
K(\x_1,\x_2)=\int_{\bdelta} p(\bdelta) \bphi_{B_{\bdelta}}(\x_1)^T\bphi_{B_{\bdelta}}(\x_2) \;d\bdelta
\end{equation}
where $B_{\bdelta}$ is a grid parameterized by $\bdelta=(\delta_1,u_1,...,\delta_d,u_d)$ that specifies the \emph{width} and \emph{bias} of the grid w.r.t. the $d$ dimensions, and $\bphi_{B_{\bdelta}}(\x)$ is a vector which has  
$$
\phi_{b}(\x)=1, \;\textit{if}\; b=(\lfloor \frac{x_{1}-u_1}{\delta_1}\rfloor, ..., \lfloor \frac{x_{d}-u_d}{\delta_d}\rfloor),
$$ 
 and $\phi_{b}(\x)=0$ otherwise for any $b\in B_{\bdelta}$. Note for each grid $B_{\bdelta}$, the number of bins $|B_{\bdelta}|$ is countably infinite, so $\bphi_{B_{\bdelta}}(\x)$ has infinite dimension but only $1$ non-zero entry (at the bin $\x$ lies in). Figure \ref{fig:RB_gen} illustrates an example when the raw dimension $d=2$. The kernel $K(\x_1,\x_2)$ is thus interpreted as the \emph{collision probability} that two data points $\x_1$, $\x_2$ fall in the same bin, when the grid is generated from distribution $p(\bdelta)$. In \cite{rahimi2007random}, it is pointed out for any kernel of form \eqref{shift_inv_kernel} with \emph{nonnegative second derivative} $k_j''(\delta)$, one can derive distribution $p(\bdelta)=\prod_{j=1}^d p_j(\delta_j)U(u_j;0,\delta_j)$, where $p_j(\delta_j)\propto \delta k_j''(\delta_j)$ and $U(\cdot,a,b)$ is uniform distribution in the range $[a,b]$.

To obtain a kernel approximation scheme from the feature map \eqref{RB_feature_map}, a simple Monte Carlo method can be used to approximate \eqref{RB_feature_map} by averaging over $R$ grids $\{B_{\bdelta_r}\}_{r=1}^R$ with each grid's parameter $\bdelta_r$ drawn from $p(\bdelta)$. The procedure for generating $R$ RB features from raw data $\{\x_n\}_{n=1}^N$ is given in Algorithm \ref{alg:RB}.

Using a Monte-Carlo analysis, one can show the approximation to \eqref{RB_feature_map} yields approximation error of order $O(1/\sqrt{R})$. From the Representer theorem, one can further bound error of the learned predictor
\begin{equation*}
\left|\w_{RF}^T\z(\x)-f^*(\x)\right|=\left|\sum_{n=1}^N \alpha^{RF}_n\z(\x_n)^T\z(\x) - \sum_{n=1}^N \alpha_n^* k(\x_n,\x)\right|
\end{equation*}
as shown in \cite{rahimi2007random} (appendix C). Unfortunately, the rate of convergence suggests that to achieve small approximation error $\eps$, one needs significant amount of random features proportional to $\Omega(1/\epsilon^2)$, and furthermore, the Monte-Carlo analysis does not explain why empirically RB feature achieves faster convergence than other random feature map like Fourier basis by orders of magnitude. 

\begin{algorithm}[t]
    \caption{ Random Binning Features }
    \label{alg:RB}
\begin{algorithmic} 
    \STATE Given a kernel function $k(\x_1,\x_2)=\prod_{j=1}^d k_j(|x_{1j}-x_{2j}|)$. Let $p_j(\delta) \propto \delta k_j''(\delta)$ be a distribution over $\delta$.
    \FOR {$r=1...R$}
        \STATE 1. Draw $\delta_{rj} \sim p_j(\delta)$, $\forall j\in [d]$. $u_{rj}\in [0,\delta_{rj}],\forall j\in[d]$. 
        \STATE 2. Compute feature $\z_{r}(\x_n)$ as the the indicator vector of bin index $(\lfloor \frac{x_{n1}-u_1}{\delta_1}\rfloor, ..., \lfloor \frac{x_{nd}-u_d}{\delta_d}\rfloor )$, for $\forall n\in[N]$.
    \ENDFOR.
    \STATE Return $\z(\x_n)=\frac{1}{\sqrt{D}}[\z_1(\x_n);...;\z_D(\x_n)]$ $\forall n\in[N]$ as the data with RB Features.
\end{algorithmic}
\end{algorithm}

\section{Faster Convergence of Random Binning} 
\label{section:faster_convergence}
In this section, we first illustrate the sparse structure of the feature matrix $Z$ of RB and discuss how to make efficient computation and storage format of $Z$. Then by interpreting RB features as Randomized Block Coordinate Descent in the infinite-dimensional space, we prove that RB has a faster convergence rate than other random features. We illustrate them accordingly in the following sections. 

\subsection{Sparse Feature Matrix \& Iterative Solvers} \label{section:iterative}
\begin{figure}
\centering
\includegraphics[width = 3.3in]{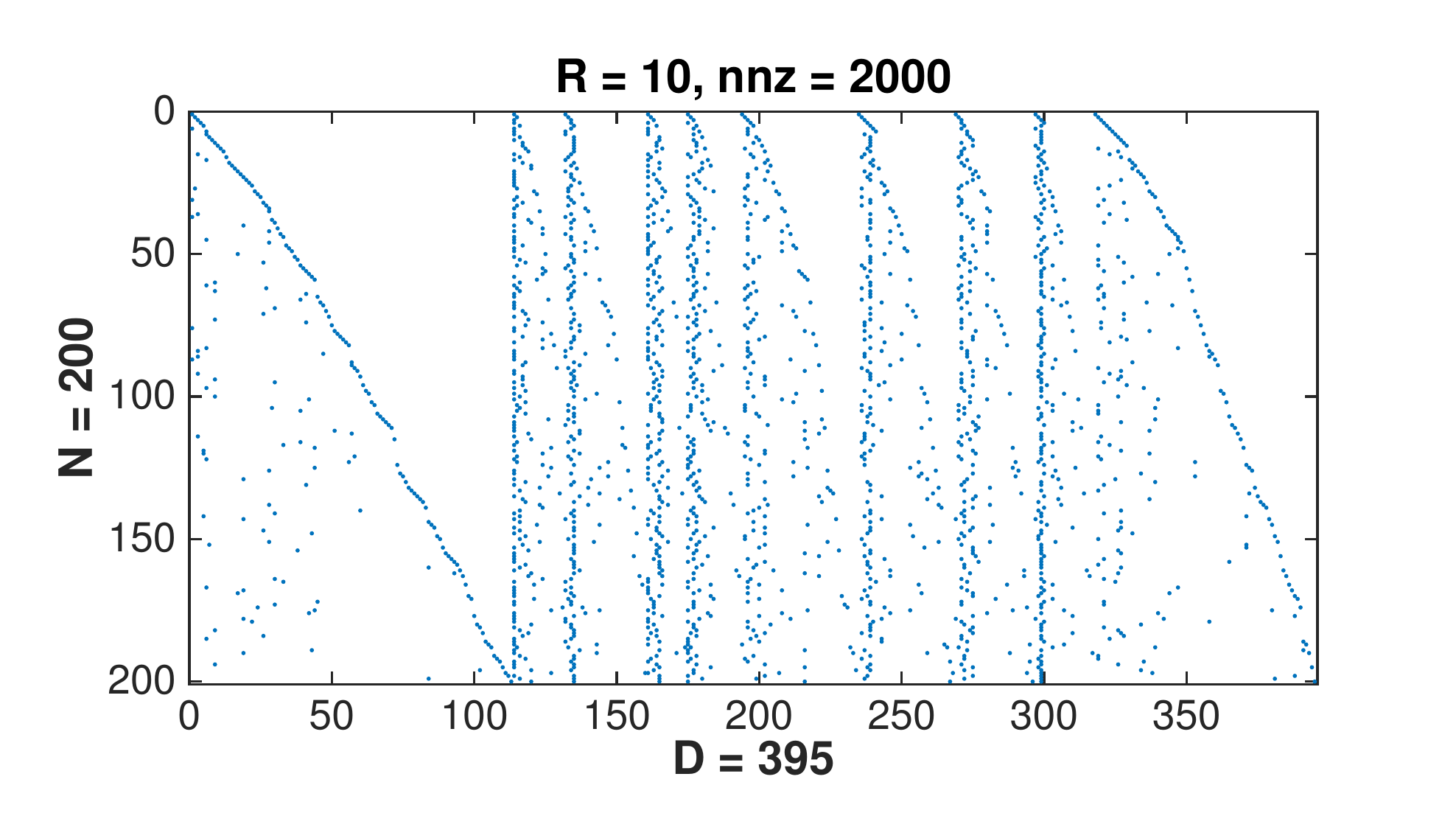}
\caption{Example of the sparse feature matrix $Z_{N \times D}$ generated by RB. In this special case,  $Z$ has the number of rows $N=200$ and columns $D=395$, respectively. The number of grids $R=10$ and the $nnz(Z)=2000$. Note that for $i$th row of $Z$, $nnz(Z(i,:))=R$ and $R \leq D \leq NR$. 
}
\label{fig:RB_matrix}
\end{figure}
 
A special characteristic of RB compared to other low-rank approximations is the fact that the feature matrix generated by RB is typically a large, sparse binary matrix $Z \in \R^{N \times D}$, where the value of $D$ is determined by both number of grids $R$ and the kernel width parameter (ex. $\sigma$ in the case of Laplacian Kernel). Different from other random features, $D$, rather than $R$, is the actual number of columns of $Z$. A direct connection between $D$ and $R$ is that the matrix has each row $i$ satisfying $nnz(Z(i,:))=R$ and therefore $R \leq D \leq NR$. Intuitively speaking, RB has more expressive power than RF since it generates a large yet sparse feature matrix to rapidly transform the data space to a very high dimension space, where data could become almost linearly separable by the classifiers. Fig. \ref{fig:RB_matrix} gives an example to illustrate the sparse structure of $Z$. 

In the case of \emph{Kernel Ridge Regression} (L2-regularization with square loss), if using RB feature to approximate the RKHS, one can solve \eqref{problem} directly in its primal form. The weighting vector is simply the solution of the linear system:
\begin{equation}
    (Z^TZ + \lambda I) \w_{RB} = Z^T y.
\end{equation}
Note since $Z$ is a large sparse matrix, there is no need to explicitly compute the covariance matrix $Z^TZ$, which is much denser than $Z$ itself. One can apply state-of-the-art sparse iterative solvers such as Conjugate Gradient (CG) and GMRES to directly operate on $Z$ \cite{Saad2003SparseLS}. The main computation in CG or GMRES is the sparse matrix-vector products. Let $m$ be the number of iterations, then the total computational complexity of iterative solver is  $O(m\, nnz(Z))=O(m NR)$. In addition, since most elements in $Z$ are zeros, the Compressed Sparse Row type matrix storage format should be employed for economically storing $Z$ \cite{Golub1996MatrixComp}, which gives computational cost and memory requirement as $O(mNR)$ and $O(NR)$ respectively, a similar cost to that of other low-rank approximations despite its much higher dimension. In testing phase, each point produces a sparse feature vector $z(x) \in \mathcal{R}^D$ based on the grids stored during training, yielding a sparse vector $z(x)$ with $nnz(z(x)))=R$ and computing the decision function $z(x)^T\w_{RB}$ only requires $O(dR + R)$. 

When the ERM is smooth but not quadratic, a \emph{Newton-CG} method that solves smooth problem via a series of local quadratic approximation gives the same complexity per CG iteration \cite{lin2008trust}, and note that most of state-of-the-art linear classification algorithms have complexity linear to $nnz(Z)$, the number of nonzeros of feature matrix \cite{fan2008liblinear}. In section \ref{section:parallel}, we further discuss cases of L1-regularized problem, where a Coordinate Descent algorithm of cost $O(nnz(Z))$ per iteration is discussed.

\subsection{Random Binning Features as Block Coordinate Descent}

In \cite{yen2014sparse}, a new approach of analysis was proposed, which interpreted Random Features as Randomized Coordinate Descent in the infinite dimensional space, and gives a better $O(1/R)$ rate in the convergence of objective function. In this section, we extend the approach of \cite{yen2014sparse} to show that, RB Feature can be interpreted as RBCD in the infinite-dimensional space, which by drawing \emph{a block of features} at a time, produces a number of features $D$ significantly more than the number of blocks $R$, resulting a provably faster convergence rate than other RF. While at the same time, by exploiting state-of-the-art iterative solvers introduced in section \ref{section:iterative}, the computational complexity of RB does not increase with number of features $D$ but only with the number of blocks $R$. Consequently, to achieve the same accuracy, RB requires significantly less training and prediction time compared to other RF.

A key quantity to our analysis is an upper bound on the \emph{collision probability} $\nu_{\bdelta}$ which specifies how unlikely data points will fall into the same bin, and its inverse $\kappa_{\bdelta}:=1/\nu_{\bdelta}$ which lower bounds the number of bins containing at least one data point. We define them as follows.
\begin{definition}\label{def:collision_prob}
Define collision probability of data $\D$ on bin $b\in B_{\bdelta}$ as 
\begin{equation}\label{collision_prob}
     \nu_{b}:=\frac{ |\{n\in[N] \;|\; \phi_{b}(\x_n)=1 \}| }{N}.
\end{equation}
Let $\nu_{\bdelta} := \max_{b\in B_{\bdelta}} \nu_b$ be an upper bound on \eqref{collision_prob}, and $\kappa_{\bdelta}:=1/\nu_{\bdelta}$ be a lower bound on the number of nonempty bins of grid $\bdelta$. 
\begin{equation}\label{expected_col_prob}
    \kappa:= E_{\bdelta}[\kappa_{\bdelta}]=E_{\bdelta}[1/\nu_{\bdelta}]
\end{equation}
is denoted as the lower bound on the expected number of (used) bins w.r.t the distribution $p(\bdelta)$.
\end{definition}
In the RB matrix, the empirical collision probability is simply the average number of non-zeros per column, divided by $N$, a number much smaller than $1$ as in the example of Fig. \ref{fig:RB_matrix}. Our analysis assumes a smooth loss function satisfying the following criteria.
\begin{assumption}\label{as:smooth}
The loss function $L(z,y)$ is smooth w.r.t. response $z$ so difference between function difference and its linear approximation can be bounded as
$$
L(z_2,.)-L(z_1,.) \leq \nabla L(z_1,.)(z_2-z_1) + \frac{\beta}{2}(z_2-z_2)^2.
$$
for some constant $0\leq\beta\leq \infty$.
\end{assumption}
This assumption is satisfied for a wide range of loss such as square loss $(\beta=1)$, logistic loss $(\beta=1/4)$ and L2-hinge loss ($\beta=1$).

We interpret RB as a \emph{Fully Corrective Randomized Block Coordinate Descent (FC-RBCD)} on the objective function
\begin{equation}\label{optimized_obj}
\begin{aligned}
    &\min_{\wbarb} && F(\wbarb):=\cR(\wbarb) + Loss(\wbarb;\bphi)\\
\end{aligned}
\end{equation}
where $Loss(\wbarb;\phi)=\frac{1}{N}\sum_{n=1}^N L(\langle \wbarb, \bphi(\x_n) \rangle, y_n)$ and 
$$
\bphib:=\sqrt{p}\circ \bphi=(\sqrt{p(\bdelta)}\bphi_{B_{\bdelta}}(.))_{\bdelta\in H}
$$
with "$\circ$" denoting the component-wise product. The goal is to show that, by performing $R$ steps of FC-RBCD on \eqref{optimized_obj}, one can obtain a $\wbarbR$ with comparable regularized loss to that from optimal solution of \eqref{problem}. Note one advantage of analysis from this optimization perspective is: it does not rely on Representer theorem, and thus $R(\w)$ can be L2 regularizer $\frac{\lambda}{2}\|\w\|^2$ or L1 regularizer $\lambda\|\w\|_1$, where the latter has advantage of giving sparse predictor of faster prediction \cite{yen2014sparse}. The FC-RBCD algorithm maintains an active set of blocks $\A^{(r)}$ which is expanded for $R$ iterations. At each iteration $r$, the FC-RBCD does the following:
\begin{itemize}
    \item[1.] Draw $\bdelta$ from $p(\bdelta)$ ( derived from the kernel $k(.,.)$ ).
    \item[2.] Expand active set $\A^{(r+1)}:= \A^{(r)} \cup B_{\bdelta}$.
    \item[3.] Minimize \eqref{optimized_obj} subject to a limited support $supp(\wbarb)\subseteq \A^{(r+1)}$.
\end{itemize}

Note this algorithm is only used for analysis. In practice, one can draw $R$ blocks of features at a time, and solve \eqref{optimized_obj} by any optimization algorithm such as those mentioned in section \ref{section:iterative} or the CD method we introduce in section \ref{section:parallel}.

Due to space limit, here we prove the case when $\cR(.)$ is the non-smooth L1 regularizer $\lambda\|\wbarb\|_1$. The smooth case for $\cR(\w)=\frac{\lambda}{2}\|\wbarb\|^2$ can be shown in a similar way. Note the objective function \eqref{optimized_obj} can be written as 
\begin{equation}\label{surrogate}
\bar{F}(\w):=F(\sqrt{\p}\circ\w)=\cR(\sqrt{\p}\circ \w) + Loss(\w,\bar{\bphi}).
\end{equation}
by a scaling of variable $\wbarb=\sqrt{p}\circ\w$.

The below theorem states that, running FC-RBCD for $R$ iterations, it generates a solution $\wbarb^{R}$ close to any reference solution $\w^*$ in terms of objective \eqref{surrogate} with their difference bounded by $O(\frac{1}{\kappa R})$.

\begin{theorem}\label{thm:RBconverge}
Let $R$ be the number of blocks (grids) generated by FC-RBCD, and $\w^*$ be any reference solution, we have
\begin{equation}\label{RBconverge}
E[\bar F(\w^{(R)})] - \bar F(\w^*) \leq \frac{\beta\|\w^*\|^2}{\kappa R'}
\end{equation}
for $R':=R-c>0$, where $c=\lceil\frac{2\kappa(\bar F(\0)-\bar F(\w^*))}{\beta\|\w^*\|^2}\rceil$.
\end{theorem}
\begin{proof}
Firstly, we obtain an expression for the progress made by each iteration of FC-RBCD. Let $B:=B_{\bdelta^{(r)}}$ be the block drawn at step 1 of FC-RBCD, and $\wbarb^{(r+1)}$ be the minimizer of \eqref{optimized_obj} subject to support $supp(\wbarb)\subseteq \A^{(r+1)}$ given by the step 3. Since $B \subseteq\A^{(r+1)}$, we have
\begin{equation}\label{FC}
F(\wbarb^{(r+1)})-F(\wbarb^{(r)}) \leq F(\wbarb^{(r)}+\etab_{B})-F(\wbarb^{(r)})
\end{equation}
for any $\etab_{B}:supp(\etab)\subseteq B$. Then denote $b_i$ as the bin $\x_i$ falling in and $L'_i = \nabla L(\wbarb^{(r)T}\bphi(\x_i),y_i)$, by smoothness of the loss (Assumption \ref{as:smooth}), we have 
\begin{equation}\label{loss_upper_bound}
\begin{aligned}
&Loss(\wbarb^{(r)}+\etab_{B}) - Loss(\wbarb^{(r)}) \leq \frac{1}{N}\sum_{i=1}^N  L'_i \phi_{b_i} \eta +\frac{\beta}{2}(\eta_{b_i}\phi_{b_i})^2  \\
&\leq \langle \g_{B},\etab_{B}\rangle + \frac{\beta\nu_{\bdelta^{(r)}} }{2}\|\etab_{B}\|^2 
\end{aligned}
\end{equation}
where the second inequality uses the fact $\phi_{b_i}=1$ and
$$
\g_{B}:=\nabla_{B} Loss(\wbarb^{(r)},\bphi).
$$ 
Now consider the regularization term, note since block $B$ is drawn from an inifinite-dimensional space, the probability that $B$ is in active set is $0$. Therefore, we have $B\cap \A^{(r)}=\emptyset$, $\wbarb^{(r)}_{B}=\0$ and $\cR_{B}(\wbarb^{(r)}_{B})=0$. As a result,
\begin{equation}\label{tmp1}
\begin{aligned}
&F(\wbarb^{(r)}+\etab_{B})-F(\wbarb^{(r)}) \\
&\leq \cR_{B}(\etab_{B})+\langle \g_{B},\etab_{B}\rangle + \frac{\beta\nu_{\bdelta^{(r)}} }{2}\|\etab_{B}\|^2 
\end{aligned}
\end{equation}
Let $\etab_{B}$ be the minimizer of RHS of \eqref{tmp1}. It satisfies $\brho_{B}+\g_{B}+\beta v_{\bdelta^{(r)}} \etab_{B}=\0$ for some $\brho_{B}\in \partial\cR(\etab_B)$, and thus,
\begin{equation}\label{tmp4}
\begin{aligned}
&F(\wbarb^{(r)}+\etab_{B})-F(\wbarb^{(r)}) \\
&\leq \langle\brho_B,\etab_B\rangle + \langle \g_{B},\etab_{B}\rangle + \frac{\beta\nu_{\bdelta^{(r)}} }{2}\|\etab_{B}\|^2\\
&=  -\frac{1}{2\beta\nu_{\bdelta^{(r)}} }\|\brho_B+\g_B\|^2
\end{aligned}
\end{equation}
Now taking expectation w.r.t. $p(\bdelta)$ on both sides of \eqref{tmp4}, we have
\begin{equation}\label{tmp5}
\begin{aligned}
E[F(\wbarb^{(r)}+\etab_{B})]-F(\wbarb^{(r)}) &\leq -\frac{1}{2\beta} E\left[\frac{1}{\nu_{\bdelta^{(r)}}}\|\brho_B+\g_B\|^2 \right] \\
                                             &\leq -\frac{1}{2\beta} E\left[\frac{1}{\nu_{\bdelta^{(r)}}}\right]E\left[\|\brho_B+\g_B\|^2 \right]\\
                                             &\leq -\frac{\kappa}{2\beta} \|\brhob_B+\gbarb_B\|^2
\end{aligned}
\end{equation}
where $\brhob:=\sqrt{\p}\circ\brho$, $\;\gbarb:=\sqrt{\p} \circ \g$, and the second inequality uses the fact that the number of used bins $\kappa_{\bdelta^{(r)}}=1/\nu_{\bdelta^{(r)}}$ has non-negative correlation with the discriminative power of block $B$ measured by the magnitude of gradient with soft-thresholding $\|\brhob_B+\gbarb_B\|$ (i.e. fewer collisions on grid $B$ implies $B$ to be a better block of features ). 

The result of \eqref{tmp5} expresses descent amount in terms of the proximal gradient of the reparameterized objective \eqref{surrogate}. Note for $B:B\cap\A^{(r)}=\emptyset$, we have $\w_B^{(r)}=\0$, and $\cR_B(\etabarb)-\cR_B(\0)= \langle\brhob, \etabarb\rangle$; on the other hand, for $B\subseteq \A^{(r)}$, we have
\begin{equation*}
\0\in arg\min_{\etabarb_{B}}\;\cR_{B}(\sqrt{p_B}\w_B+\etabarb_B)+\langle \gbarb_B, \sqrt{p_B}\w_{B}+\etabarb_B\rangle
\end{equation*}
since they are solved to optimality in the previous iteration. Then
\begin{equation}\label{tmp2}
\begin{aligned}
&E[F(\wbarb^{(r)}+\etab)]-F(\wbarb^{(r)}) \\
&\leq -\frac{\kappa}{2\beta} \|\brhob_B+\gbarb_B\|^2 = \langle\brhob,\etabarb\rangle + \langle \gbarb,\etabarb \rangle + \frac{\beta}{2\kappa}\|\etabarb_{\bar{\A}^{(r)}}\|^2\\
&= \cR(\sqrt{\p} \circ (\w^{(r)}+\etabarb)  )-\cR(\sqrt{\p}\circ \w^{(r)}) + \langle \gbarb,\etabarb \rangle + \frac{\beta}{2\kappa}\|\etabarb_{\bar{\A}^{(r)}}\|^2 
\end{aligned}
\end{equation}
where $\etabarb_{\bar{\A}^{(r)}}:=(\etabarb_B)_{B:B\cap\A^{(r)}=\emptyset }$ and $\etabarb:=\sqrt{\p}\circ\etab$. Thus the final step is to show the descent amount given by RHS of \eqref{tmp2} decreases the suboptimality $\bar{F}(\w^{(r)})-\bar{F}(\w^*)$ significantly. This can be achieved by considering $\etabarb$ of the form $\alpha(\w^{*}-\w^{(r)})$ for some $\alpha\in[0,1]$ as follows:
\begin{equation}\label{tmp3}
\begin{aligned}
&E[\bar{F}(\w^{(r)}+\etabarb)]-\bar{F}(\w^{(r)}) \\
&\leq \min_{\etabarb} \cR(\sqrt{\p} \circ (\w^{(r)}+\etabarb)  )-\cR(\sqrt{\p}\circ \w^{(r)})+\langle \gbarb ,\etabarb \rangle + \frac{\beta }{2\kappa}\|\etabarb_{\bar{\A}^{(r)}}\|^2 \\
&\leq \min_{\etabarb} \bar{F}(\w^{(r)}+\etabarb)-\bar{F}(\w^{(r)}) + \frac{\beta }{2\kappa}\|\etabarb_{\bar{\A}^{(r)}}\|^2  \\
&\leq \min_{\alpha\in[0,1]} \bar{F}((1-\alpha)\w^{(r)}+\alpha\w^*)-\bar{F}(\w^{(r)}) + \frac{\beta \alpha^2}{2\kappa}\|\w^*\|^2\\
&\leq \min_{\alpha\in[0,1]} -\alpha(\bar{F}(\w^{(r)})-\bar{F}(\w^*)) + \frac{\beta \alpha^2}{2\kappa}\|\w^*\|^2,
\end{aligned}
\end{equation}
where the second and fourth inequalities are from convexity of $\bar{F}(.)$. The $\alpha$ minimizing \eqref{tmp3} is $\alpha^*:=\min(\frac{\kappa(\bar{F}(\w^{(r)})-\bar{F}(\w^*))}{\beta\|\w^*\|^2},1)$, which leads to
\begin{equation}\label{tmp6}
\begin{aligned}
&E[\bar{F}(\w^{(r)}+\etabarb)]-\bar{F}(\w^{(r)}) \leq -\frac{\kappa(\bar{F}(\w^{(r)})-\bar{F}(\w^*))^2}{2\beta\|\w^*\|^2}
\end{aligned}
\end{equation}
if $\bar{F}(\w^{(r)})-\bar{F}(\w^*)\leq \frac{\beta}{\kappa}\|\w^*\|^2$; otherwise, we have $E[\bar{F}(\w^{(r)}+\etabarb)]-\bar{F}(\w^{(r)})\leq -\frac{\beta}{2\kappa}\|\w^*\|^2$. Note the latter case cannot happen more than $c=\lceil\frac{2\kappa(\bar F(\0)-\bar F(\w^*))}{\beta\|\w^*\|^2}\rceil$ times since FC-RBCD is a descent method. Therefore, for $r':=r-c>0$, solving the recursion \eqref{tmp6} leads to the conclusion. \qedsymbol
\end{proof}
Note we have $\|\sqrt{\p}\circ \w^*\|_1$ $\leq$ $\|\sqrt{\p}\|\|\w^*\|$ $=\|\w^*\|$ in the L1-regularized case, and thus the FC-RBCD guarantees convergence of the L1-norm objective to the (non-square) L2-norm objective. The convergence result of Theorem \ref{thm:RBconverge} is of the same form to the rate proved in \cite{yen2014sparse} for other random features, however, with an additional multiplicative factor $\kappa\geq 1$ that speeds up the rate by $\kappa$ times. Recall that $\kappa$ is the lower bound on the expected number of bins being used by data samples for each block of features $B_{\bdelta}$, which in practice is a factor much larger than $1$, as shown in the Figure \ref{fig:RB_matrix} and also in our experiments. In particular, in case each grid $B_{\bdelta}$ has similar number of bins being used, we have $D\approx \kappa R$, and thus obtain a rate of the form
\begin{equation}\label{D_rate}
    E[\bar F(\w^{(R)})] - \bar F(\w^*) \lesssim \frac{\beta\|\w^*\|^2}{D}.
\end{equation}
Note for a fixed $R$, the total number of features $D$ is increasing with kernel parameter $1/\sigma$ in the case of Laplacian Kernel, which means the less smooth the kernel, the faster convergence of RB. A simple extreme case is when $\sigma\rightarrow 0$, where one achieves $0$ training loss, and the RB, by putting each sample in a separate bin, converges to $0$ loss with $R=1$, $D=N$. On the other hand, other random features, such as Fourier, still require large $R$ for convergence to $0$ loss. In practice, there are many data that require a small kernel bandwidth $\sigma$ to avoid underfitting, for which RB has dramatically faster convergence than other RF.


\section{Strong Parallelizability of Random Binning Features} \label{section:parallel}
\label{section:parallel}

In this section, we study another strength of RB Features in the context of \emph{Sparse Random Feature} \cite{yen2014sparse}, where one aims to train a sparse nonlinear predictor that has faster prediction and more compact representation through an L1-regularized objective. In this case, the CD method is known as state-of-the-art solver \cite{yuan2010comparison,richtarik2014iteration}, and we aim to show that the structure of RB allows CD to be parallelized with much more speedup than that of other random features.

\subsection{Coordinate Descent Method}

Given the $N\times D$ data matrix produced by the RB Algorithm \ref{alg:RB}, a RCD Method solves 
\begin{equation}\label{L1_obj}
\min_{\w\in\R^D} \;\lambda\|\w\|_1 + \frac{1}{N}\sum_{n=1}^N L(\w^T\z_i, y_i)
\end{equation}
by minimizing \eqref{L1_obj} w.r.t. a single coordinate $j$
\begin{equation}\label{cd_obj}
    \min_{d_j} \; \lambda |w_j+d_j| + g_jd_j + \frac{M_j}{2}d_j^2
\end{equation}
at a time, where 
\begin{equation}\label{cd_grad}
    g_j:=\frac{1}{N}\sum_{n=1}^N (\nabla_j L(\w^T\z_i,y_i)) z_{ij}
\end{equation}
is the gradient of loss term in \eqref{L1_obj} w.r.t. the $j$-th coordinate, and $M_j:=\beta\frac{1}{N}\sum_{i=1}^N z_{ij}^2$ is an upper bound on $\nabla_{jj}L(.)$. Note, by focusing on single coordinate, \eqref{cd_obj} has a tighter quadratic upper bound than other algorithms such as Proximal Gradient Method, and allows simple closed-form solution 
\begin{equation}\label{cd_sol}
    d_j^* := \prox_{R/M_j}( w_j-\frac{g_j}{M_j} )-w_j
\end{equation}
where 
\begin{equation*}
\prox_{R}(v_j):=\left\{\begin{array}{ll}
0  ,         & |v_j| \leq \lambda \\
v_j-\lambda, & v_j > \lambda \\
v_j+\lambda, & v_j < \lambda 
\end{array}\right. .
\end{equation*}
To have efficient evaluation of the gradient \eqref{cd_grad}, a practical implementation maintain the responses
\begin{equation}\label{response}
    \yh_i := \w^Tz_i
\end{equation}
after each update $\w^{t+1}:= \w^{t} + d^*_j\e_j$, so the cost for each coordinate-wise minimization takes $O(nnz(\z_j))$ time for both gradient evaluation and maintenance of \eqref{response}, where $\z_j:=(z_{ij})_{i\in[N]}$. The algorithm is summarized in Alg. \ref{alg:PRCD-RB}, which just like the iterative solver introduced in section \ref{section:iterative}, has cost $O(nnz(Z))$ for one pass of all variables $j\in[D]$.

\begin{algorithm}[t]
    \caption{Sparse Random Binning Features via Parallel RCD}
    \label{alg:PRCD-RB}
\begin{algorithmic} 
    \STATE 0. Generate RB feature matrix $Z$ by Algorithm \ref{alg:RB}
    \STATE 1. $\z^{1}=\0$, $\w^{1}=\0$.
    \FOR {t=1......T (with $\tau$ threads in parallel)}
        \STATE 2. Draw $j$ from $[D]$ uniformly at random.
        \STATE 3. Compute $d^*_j$ by \eqref{cd_sol}.
        \STATE 4. $\w^{t+1}:= \w^{t} + d^*_j\e_j$.
        \STATE 5. Maintain $\yh_i,\forall i\in[N]$ to satisfy \eqref{response}.
    \ENDFOR
\end{algorithmic}
\end{algorithm}

\subsection{Parallel Randomized Coordinate Descend on Random Binning Features}

The RCD, however, is hard to parallelize \cite{richtarik2014iteration}. It is known that simultaneous updates of two coordinates $j_1$, $j_2$ could lead to divergence, and although one can enforce convergence by shortening the step size $\frac{1}{M_{p}}\ll \frac{1}{M_j}$, the convergence rate will not be improved with parallelization without additional assumption \cite{bradley2011parallel,richtarik2015parallel}.

On the other hand, in \cite{richtarik2015parallel}, it is shown that a function with \emph{partially separable} smooth term plus a separable non-smooth term 
\begin{equation}\label{partial_separable}
\min_{\w\in \R^{D}} \; F(\w):=\Omega(\w) + \sum_{i=1}^N f_{i}(\w)
\end{equation}
can be parallelized with guaranteed speedup in terms of overall complexity, where $\Omega(\w)$ is a non-smooth separable function and each function $f_i(\w)$ is a smooth depends only on at most $\omega$ number of variables. The form \eqref{partial_separable}, fortunately, fits our objective \eqref{L1_obj} with features $\z_i$ generated by RB. In particular, the generating process of RB guarantees that, for each block of feature $B_{\bdelta}$, the $i$-th sample can fall in exactly one bin $b=(\lfloor \frac{x_{n1}-u_1}{\delta_1}\rfloor, ..., \lfloor \frac{x_{nd}-u_d}{\delta_d}\rfloor)$, therefore each sample inolves at most $R$ features out of $D$. Specifically, let $\Omega(\w):=\lambda\|\w\|_1$ and
$$
f_{i}(\w):=\frac{1}{N}L(\w^T\z_i,y_i),
$$
we have $\omega=R$. Then by Theorem 19 of \cite{richtarik2015parallel}, a parallel RCD of $\tau$ threads that selects coordinate $j$ uniformly at random achieves a speed-up (i.e. time-of-sequential/time-of-parallel) of
\begin{equation}\label{parallel_rate_1}
\textit{speedup-ratio}=\frac{\tau}{1+\frac{(R-1)(\tau-1)}{D-1}}.
\end{equation}
When $D, R\gg 1$, and $\tau=a\bar\kappa+1$ where $\bar\kappa:=D/R$, \eqref{parallel_rate_1} becomes 
\begin{equation}\label{parallel_rate_2}
\textit{speedup-ratio}=\frac{a\bar\kappa+1}{1+a},
\end{equation}
which equals $(\bar\kappa+1)/2$ when $a=1$ and approaches $\bar\kappa$ when $a\rightarrow \infty$. Therefore, it is guaranteed in theory that parallelization can speedup RCD significantly as long as $\bar\kappa=D/R\gg 1$. We give our sparse RB Features algorithm based on parallel RCD in Alg. \ref{alg:PRCD-RB}. Note for other Random Features, there is no speedup guaranteed and our experiment shows that Parallel RCD performed on Random Fourier features could even have no speedup.

Note that the speedup achieved in this section is orthogonal to the faster convergence rate achieved in section \ref{section:faster_convergence}, so by increasing $\kappa$, the advantage of RB over other Random Features is super-linearly increasing if a parallel RCD is used. Note also that the results \eqref{parallel_rate_1}, \eqref{parallel_rate_2} also apply to algorithms that utilize Coordinate Descent as subproblem solvers such as Proximal (Quasi) Newton Method \cite{lee2012proximal,zhong2014proximal}. Those methods are typically employed for computationally expensive loss functions.

\section{Experiments}
\label{section:experiments}

In this section, we present extensive sets of experiments to demonstrate the efficiency and effectiveness of RB. The datasets are chosen to overlap with those in other papers in the literature, where the details are shown in the table \ref{tb: info of datasets}. All sets except census are available at LIBSVM data set \cite{LIBSVM}. All computations are carried out on a DELL dual socket with Intel Xeon processors at 2.93GHz for a total of 16 cores and 250 GB of memory running the SUSE Linux operating system. We implemented all methods in C++ and all dense matrix operations are performed by using the optimized BLAS and LAPACK routines provided in the OpenBLAS library. Due to the limited space, we only choose subsets of our results to present in each subsection. However, these results are objective and unbiased. 

\begin{table}[htbp]
\centering
\caption{Properties of the datasets.} 
\label{tb: info of datasets}
\small
\begin{center}
    \begin{tabular}{ |c|c|c|c|c|}
    \hline
    Name 		& $C$: Classes & $d$: Features& $N$: Train & $M$: Test \\ \hline 
    cadata 		& 1	 & 8   & 16,512   & 4,128      \\ \hline
    census 		& 1  & 119  & 18,186  & 2,273    \\ \hline
    ijcnn1 		& 2  & 22  & 35,000  & 91,701   \\ \hline
    cod\_rna    & 2  & 8  & 49,437  & 271,617  \\ \hline
    covtype     & 2  & 54 & 464,809 & 116,203  \\ \hline
    SUSY        & 2  & 18 & 4,000,000 & 1,000,000    \\ \hline
    mnist       & 10 & 780 & 60,000 & 10,000   \\ \hline
    acoustic    & 3  & 50 & 78,823 & 19,705    \\ \hline
    letter      & 26 & 16 & 10,500 & 5,000   \\ \hline
    \end{tabular}
\end{center}
\end{table}

\subsection{Effects of $\sigma$ and $R$ on Random Binning}
We perform experiments to investigate the characteristics of RB by varying the kernel parameter $\lambda$ and the rank $R$, respectively. We use a regularization $\lambda = 0.01$ to make sure the reasonable performance of RB and other low-rank kernels, although we found that RB is not sensitive to this parameter. We increase the $\sigma$ in the large interval from $1e-2$ to $1e2$ so that the optimal $\sigma$ locates within the interval. We apply CG iterative solver to operate on $Z$ directly. In order to make fair runtime comparison in each run, we set the $tol = 1e-15$ to force similar CG iterations with different $\sigma$. 

We evaluate the training and testing performance of regression and classification, when varying $\sigma$ with fixed $R$. In \cite{rahimi2007random}, it does not consider the effect of $\sigma$ in their analysis, which however has a large impact on the performance since $D$ depends on the number of bins which is controlled by $\sigma$. Fig. \ref{fig:TrTe_perf_time_lambda} shows that the training and testing performance coincidentally decrease (increase) before they diverge when $D$ grows by increasing $\sigma$. This confirms with our analysis in Theorem \ref{thm:RBconverge} that the larger $\kappa$, the faster convergence of RB Feature (recall that the convergence rate is $O(1/(\kappa R))$). 

Second, one should not be surprised that the empirical training time increases with $D$. The operations involving the weighting vector $w_{RB}$ could become as expensive as a sparse matrix-vector operation in an iterative solver. However, the total computational costs are still bounded by $O(NR)$ but the constant factor may vary with different datasets. Fortunately, in most of cases, the training time corresponding to the peak performance is just slightly higher than the smallest one. In practice, there are several ways to improve the computation costs by exploiting more advanced sparse matrix techniques such as preconditioning and efficient storage scheme, which is out scope of this paper and left for future study. 

Finally, we evaluate the training and testing performance when varying $R$ with fixed $\sigma$. Fig.\ref{fig:TrTe_perf_R} shows that the training and testing performance converge almost linearly with $D$, which again confirms our analysis in Theorem \ref{thm:RBconverge}. In addition, we observe that RB has strong overfit ability which turns out to be a strong attribute, especially when the hypothesis space has not yet saturated. 

\begin{figure*}[!htb]
\centering
\subfigure[cadata]{\includegraphics[width = 1.55in]{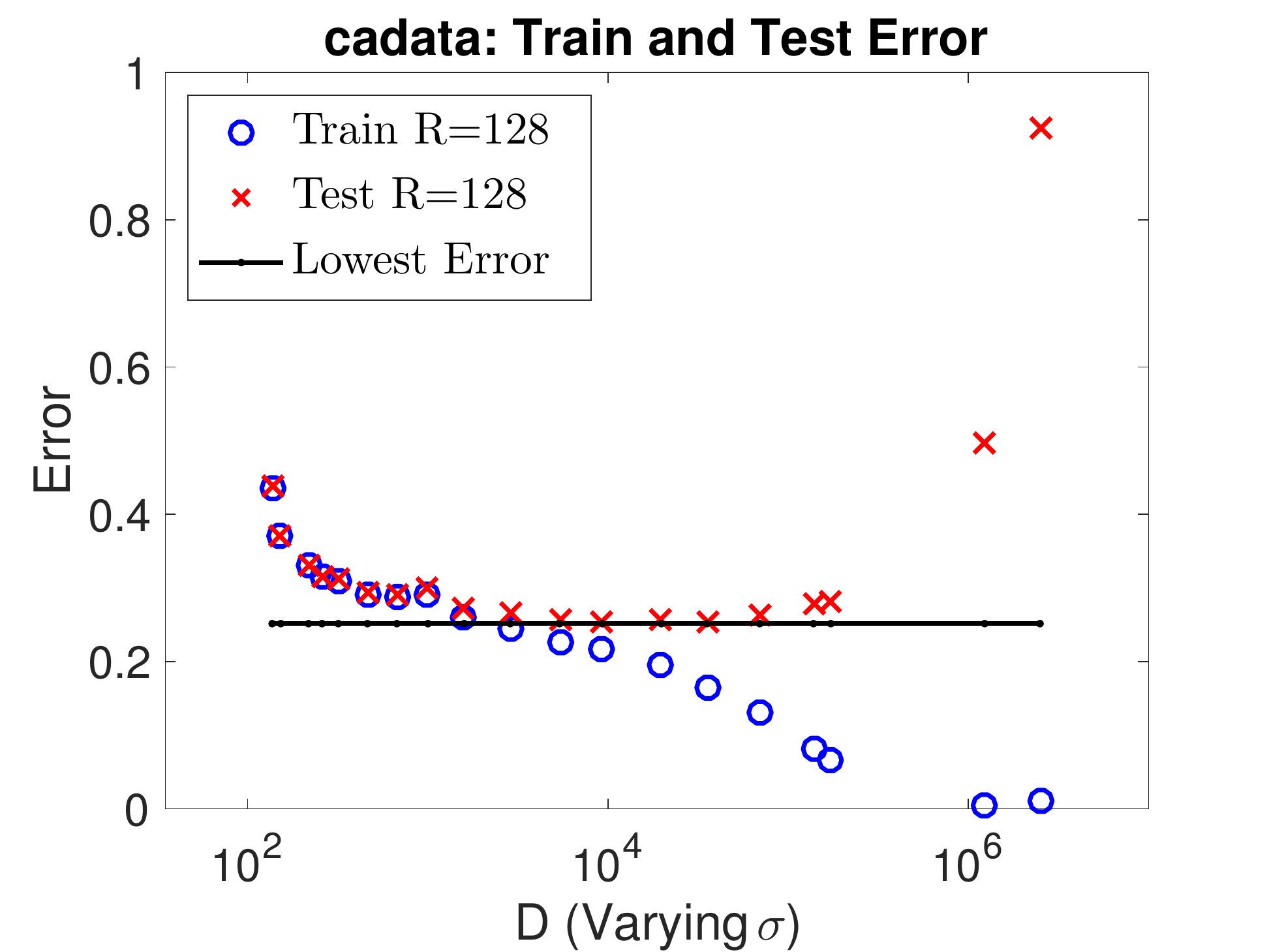}}
\subfigure[cadata]{\includegraphics[width = 1.55in]{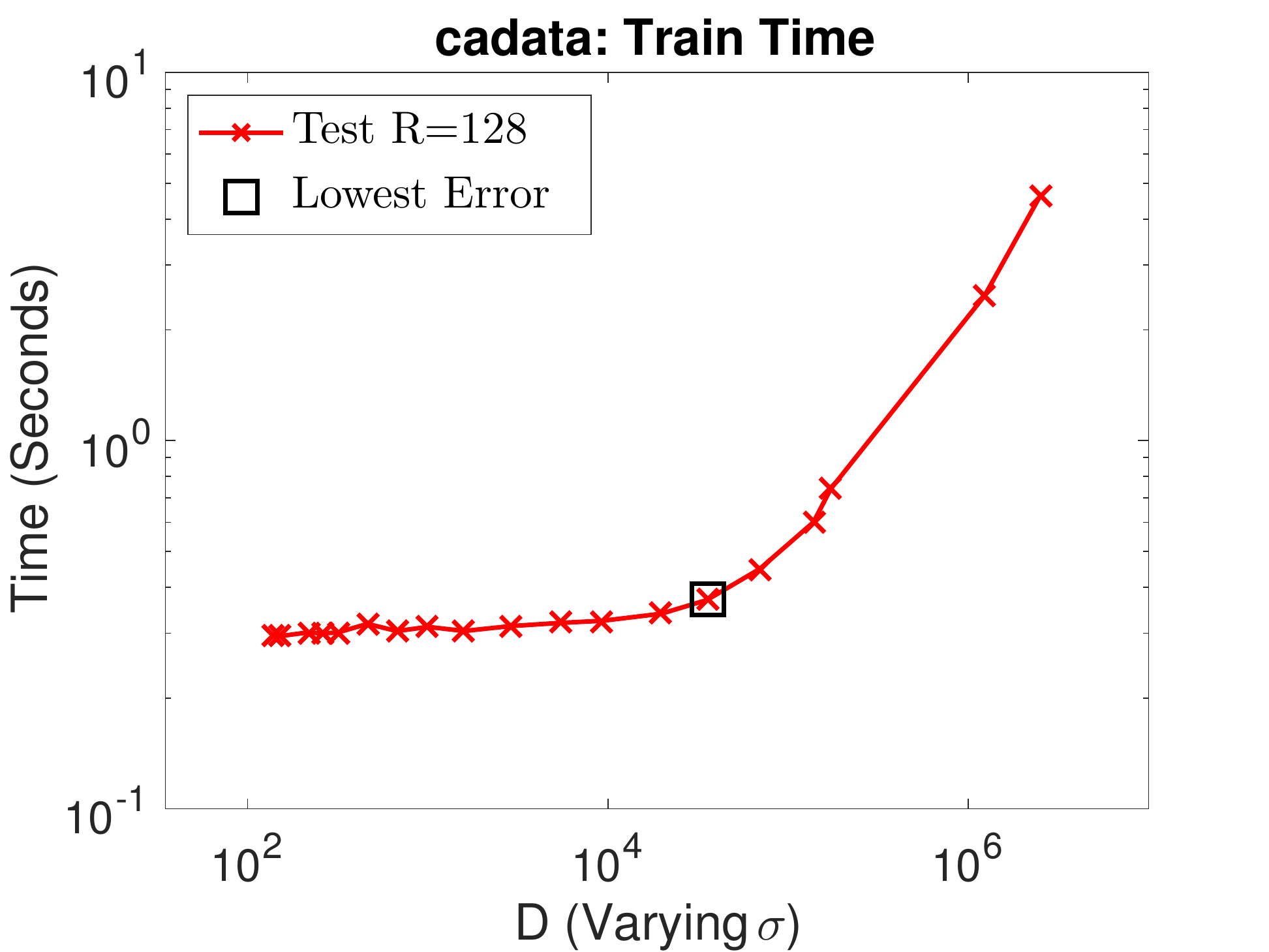}}
\subfigure[census]{\includegraphics[width = 1.55in]{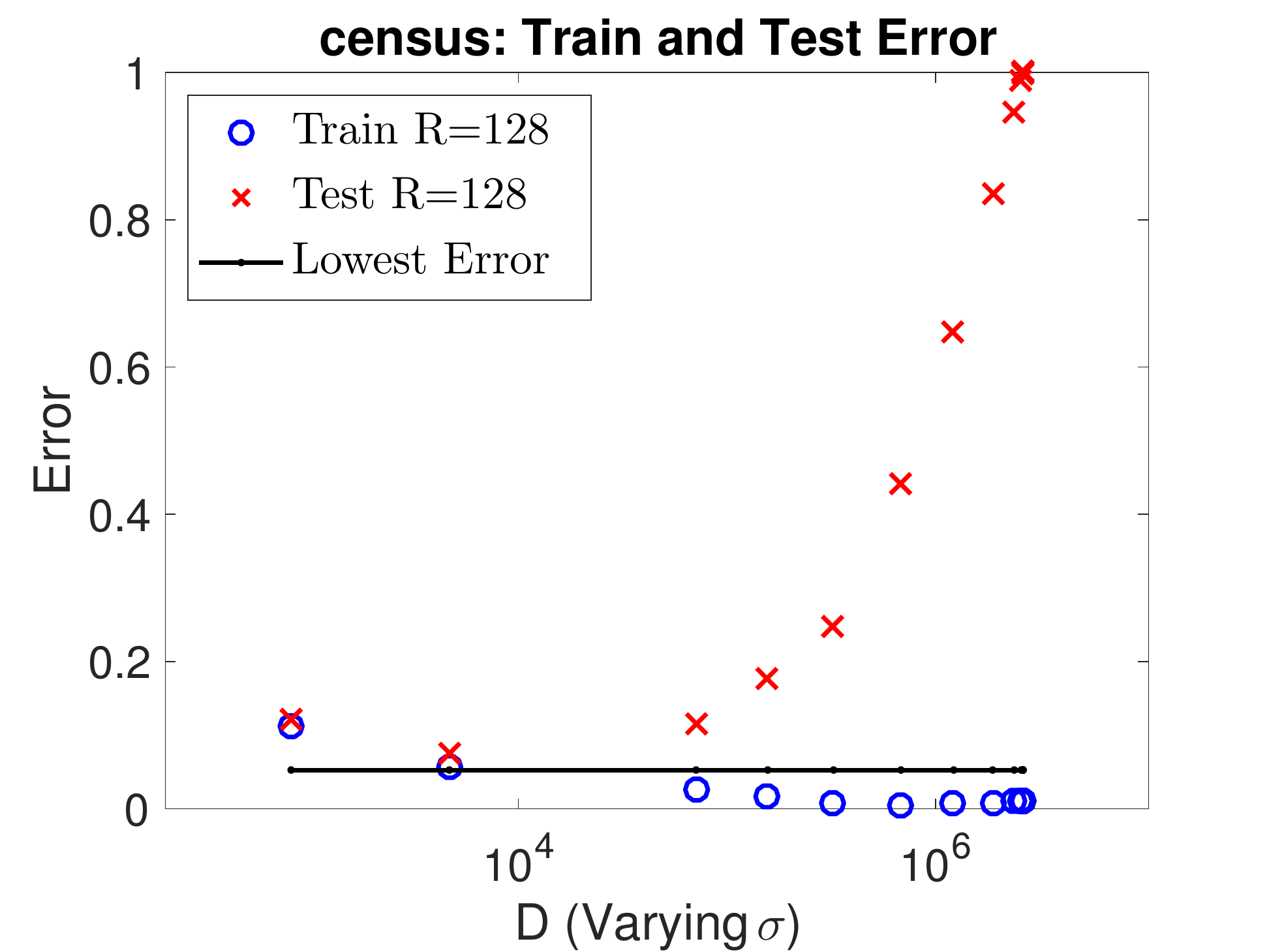}}
\subfigure[census]{\includegraphics[width = 1.55in]{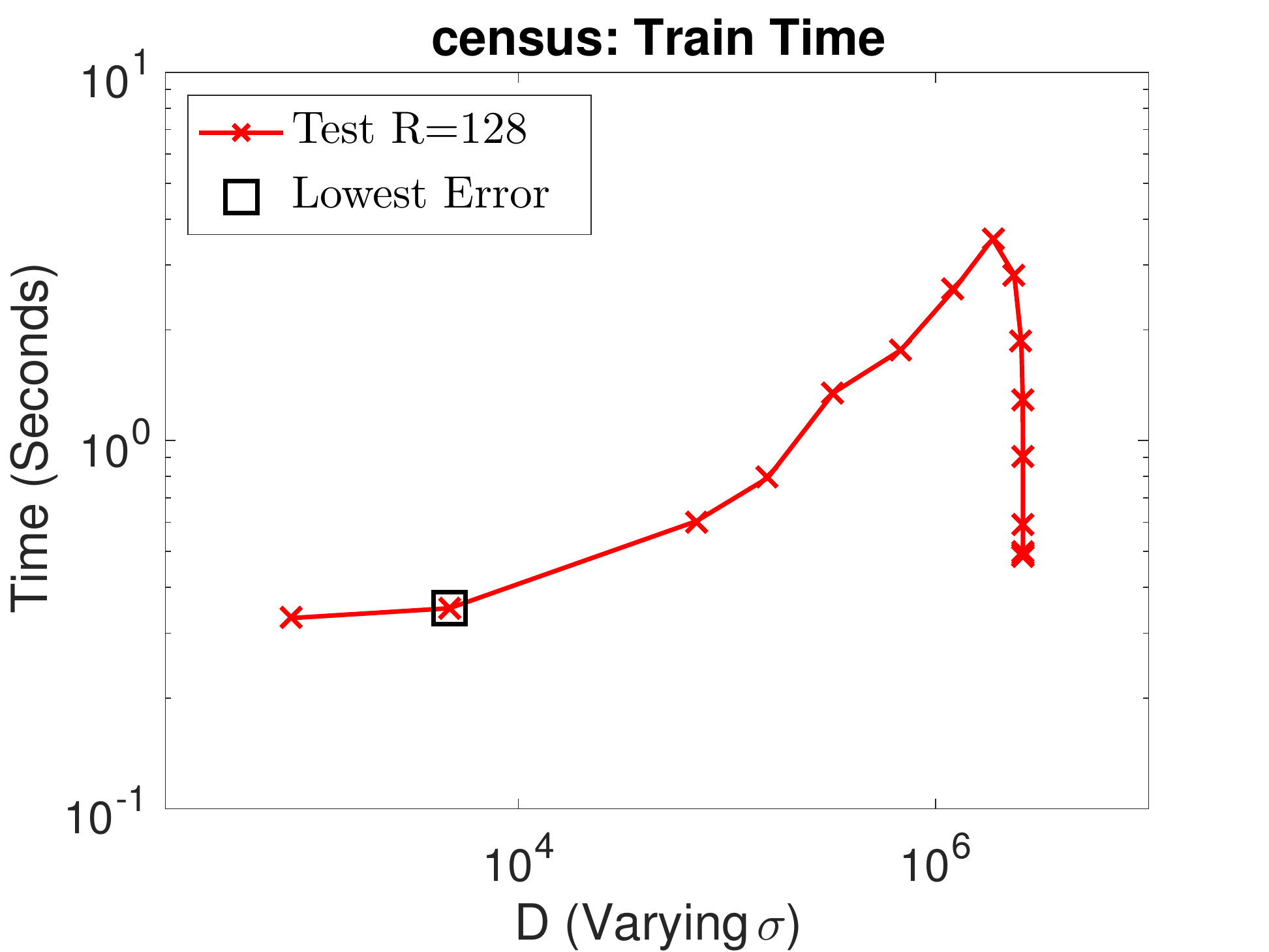}}
\subfigure[ijcnn1]{\includegraphics[width = 1.55in]{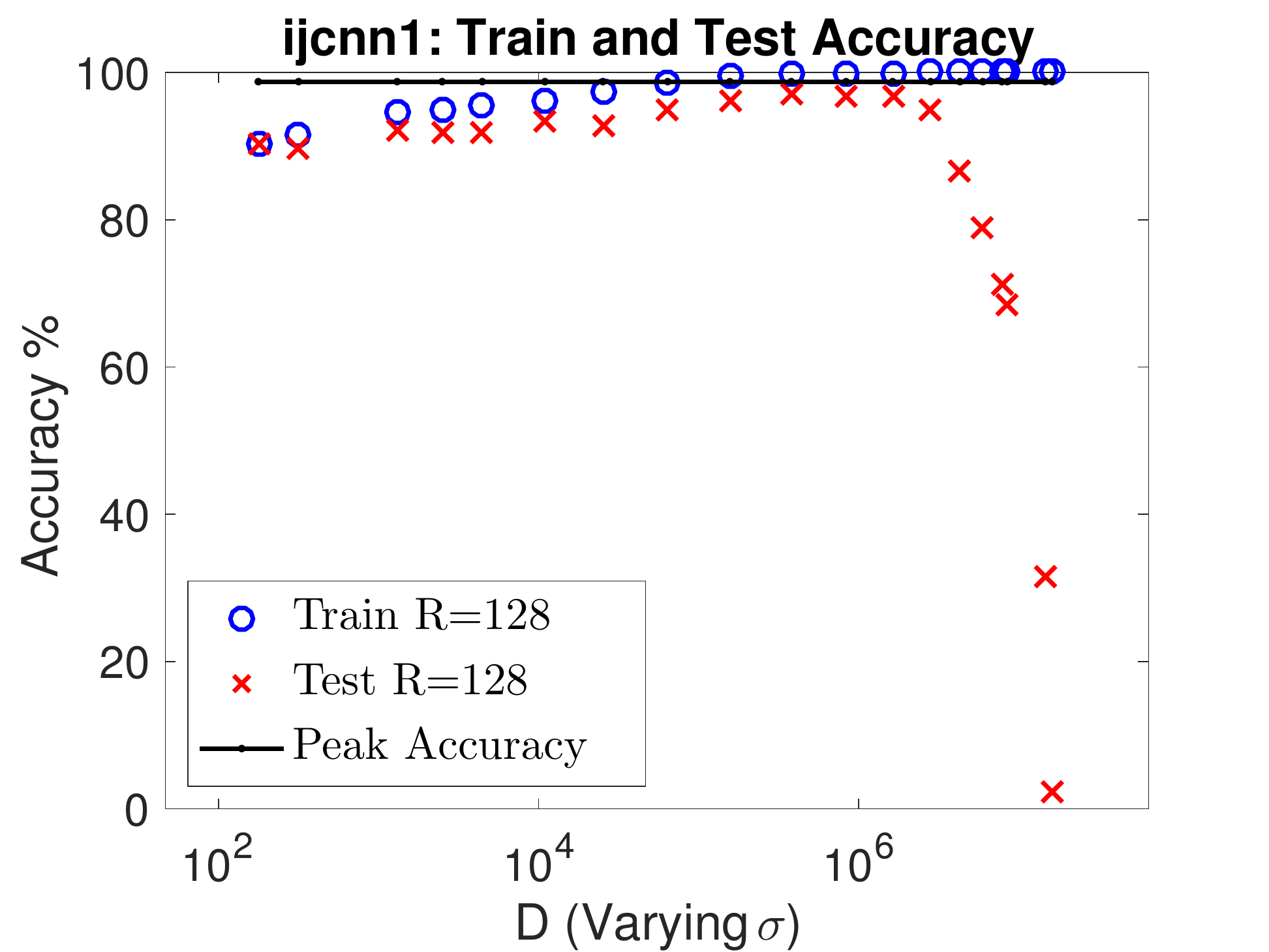}}
\subfigure[ijcnn1]{\includegraphics[width = 1.55in]{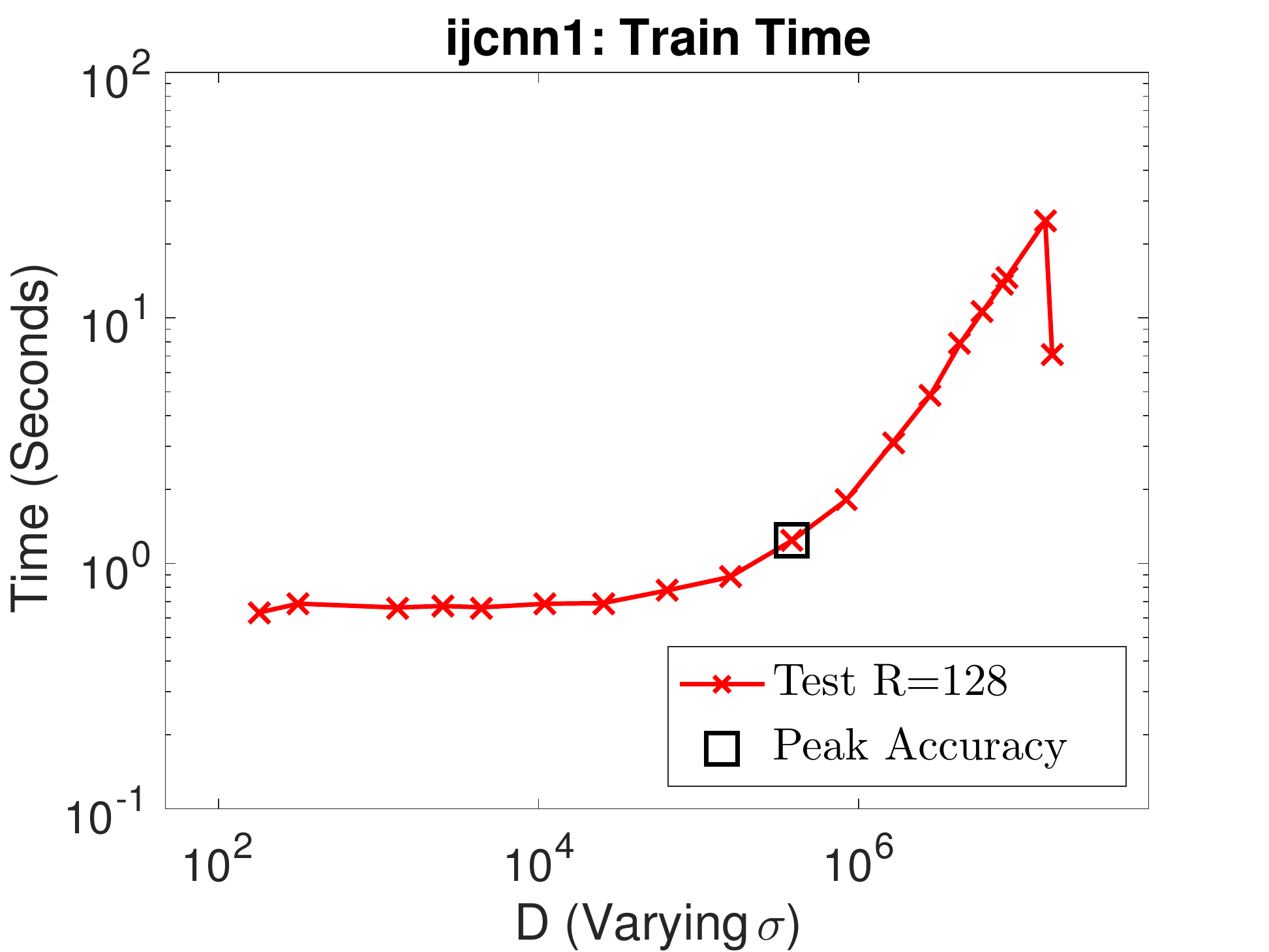}}
\subfigure[covtype]{\includegraphics[width = 1.55in]{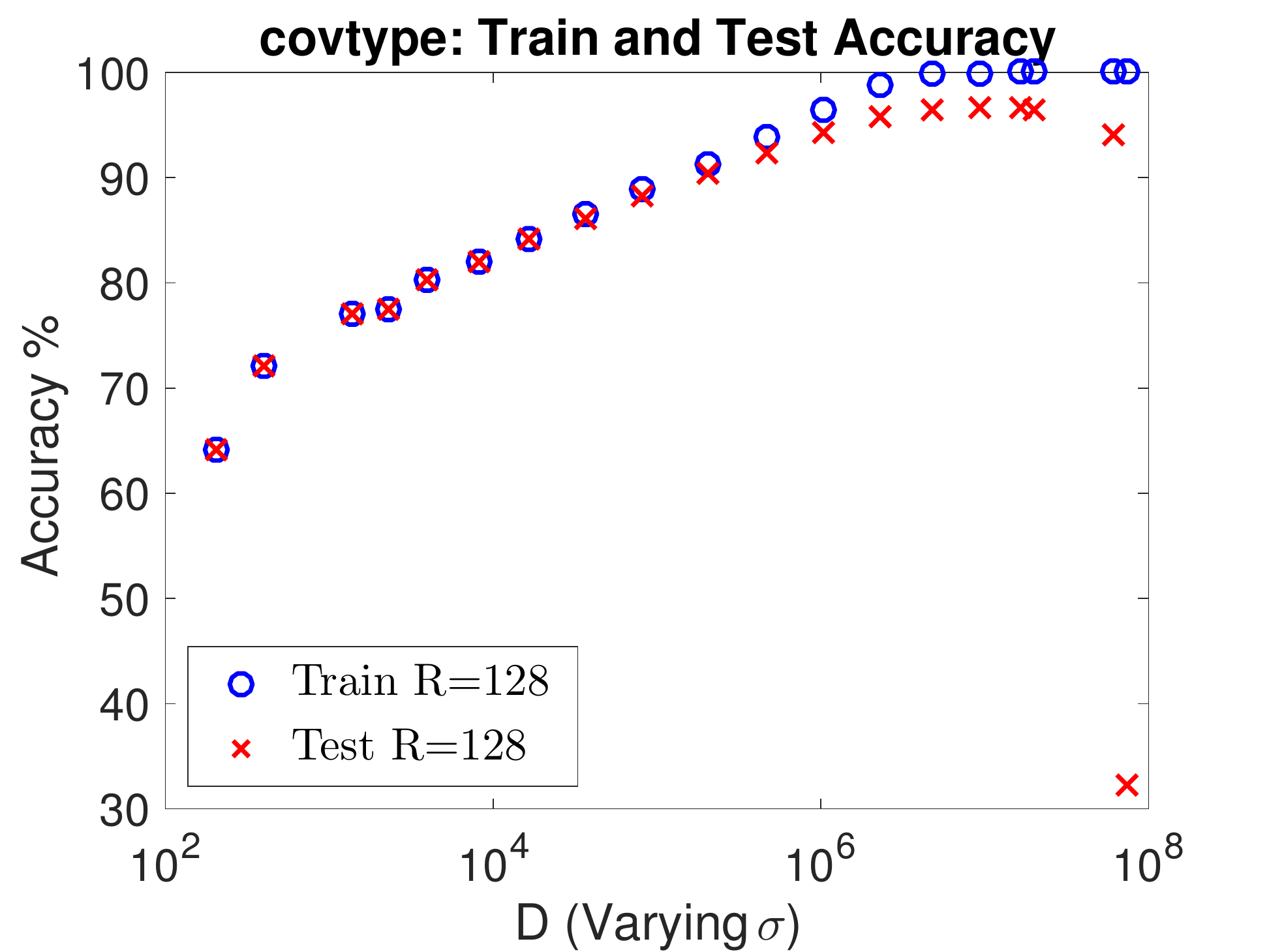}}
\subfigure[covtype]{\includegraphics[width = 1.55in]{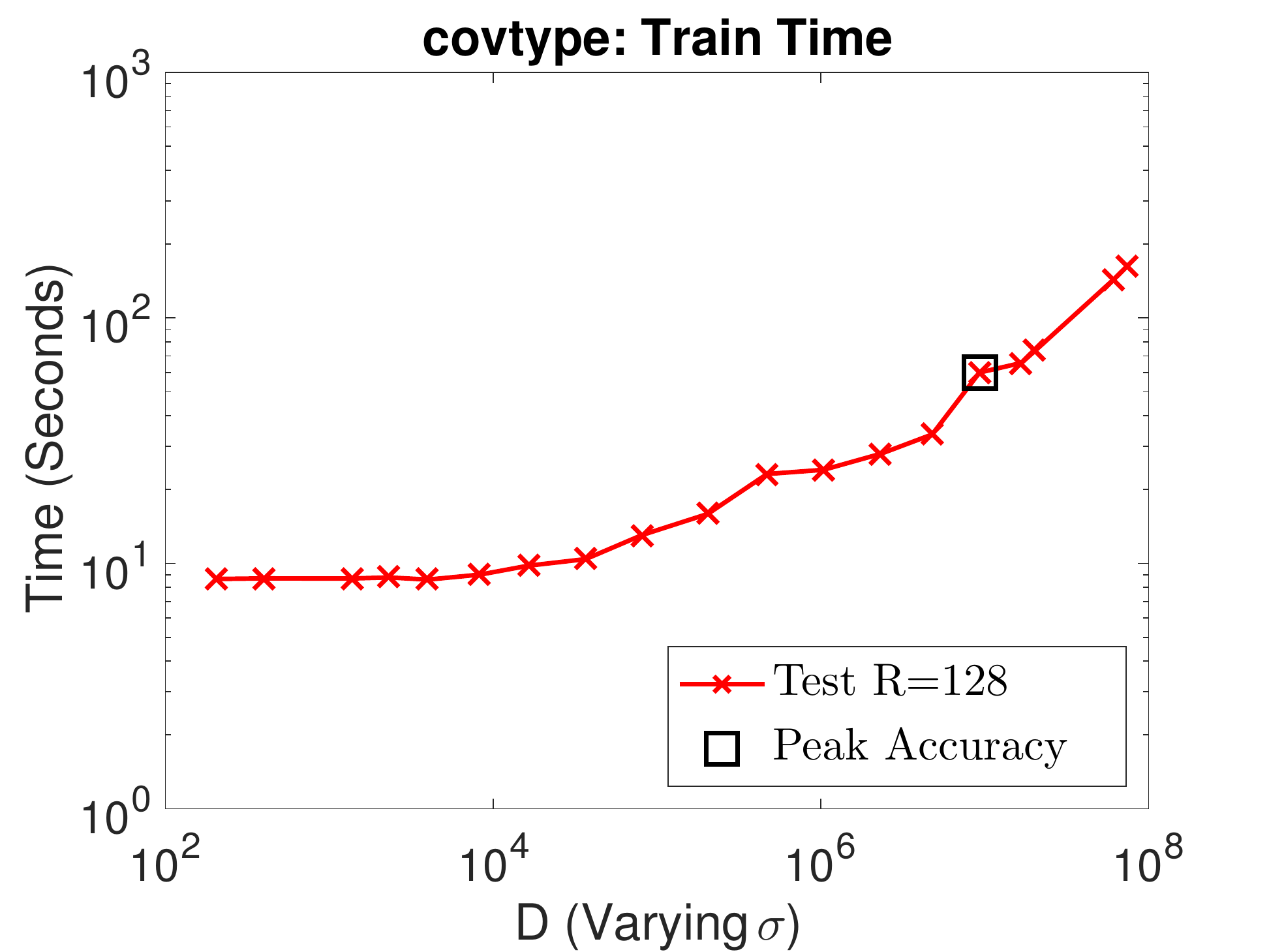}}
\subfigure[SUSY]{\includegraphics[width = 1.55in]{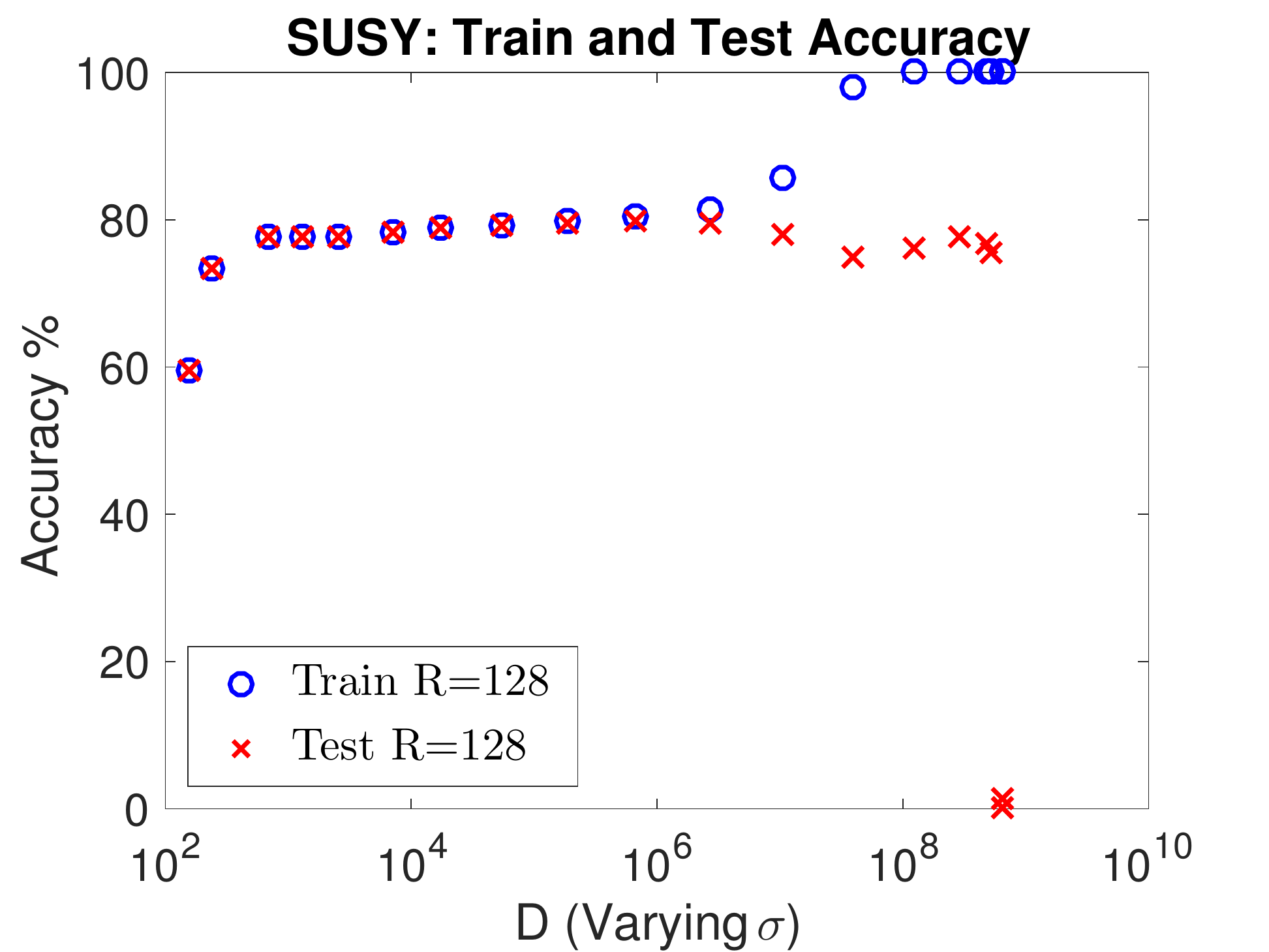}}
\subfigure[SUSY]{\includegraphics[width = 1.55in]{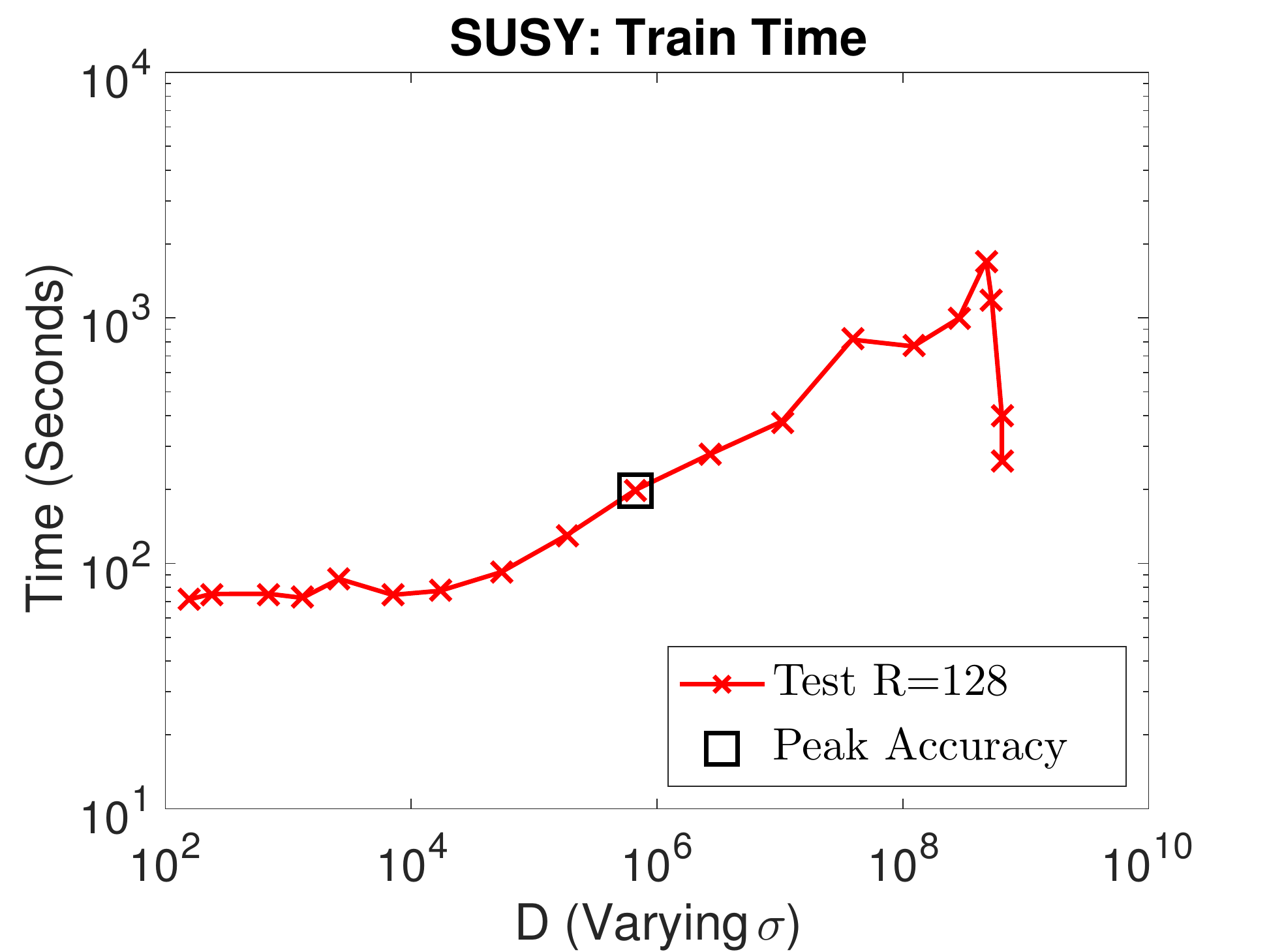}}
\subfigure[mnist]{\includegraphics[width = 1.55in]{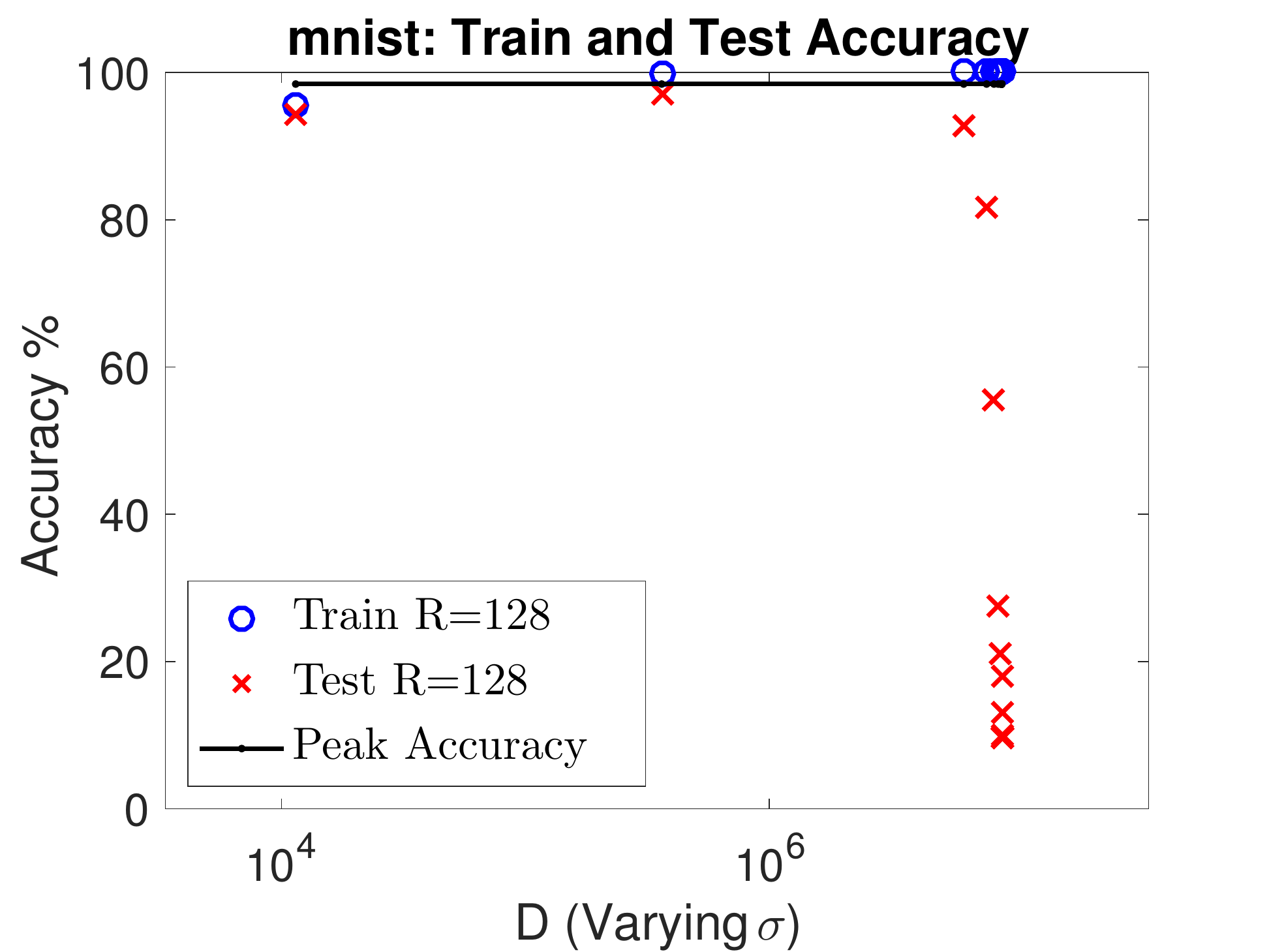}}
\subfigure[mnist]{\includegraphics[width = 1.55in]{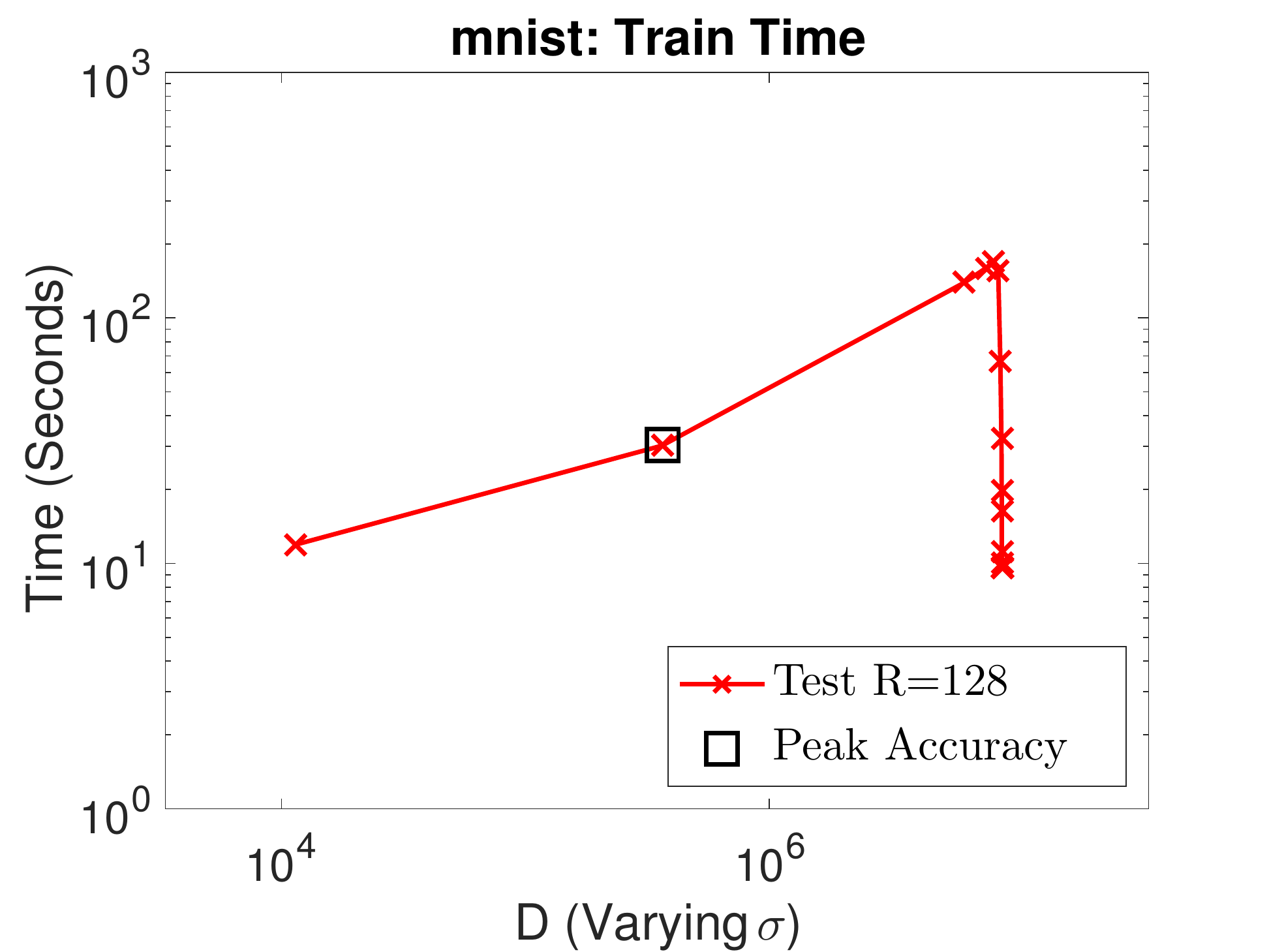}}
\subfigure[acoustic]{\includegraphics[width = 1.55in]{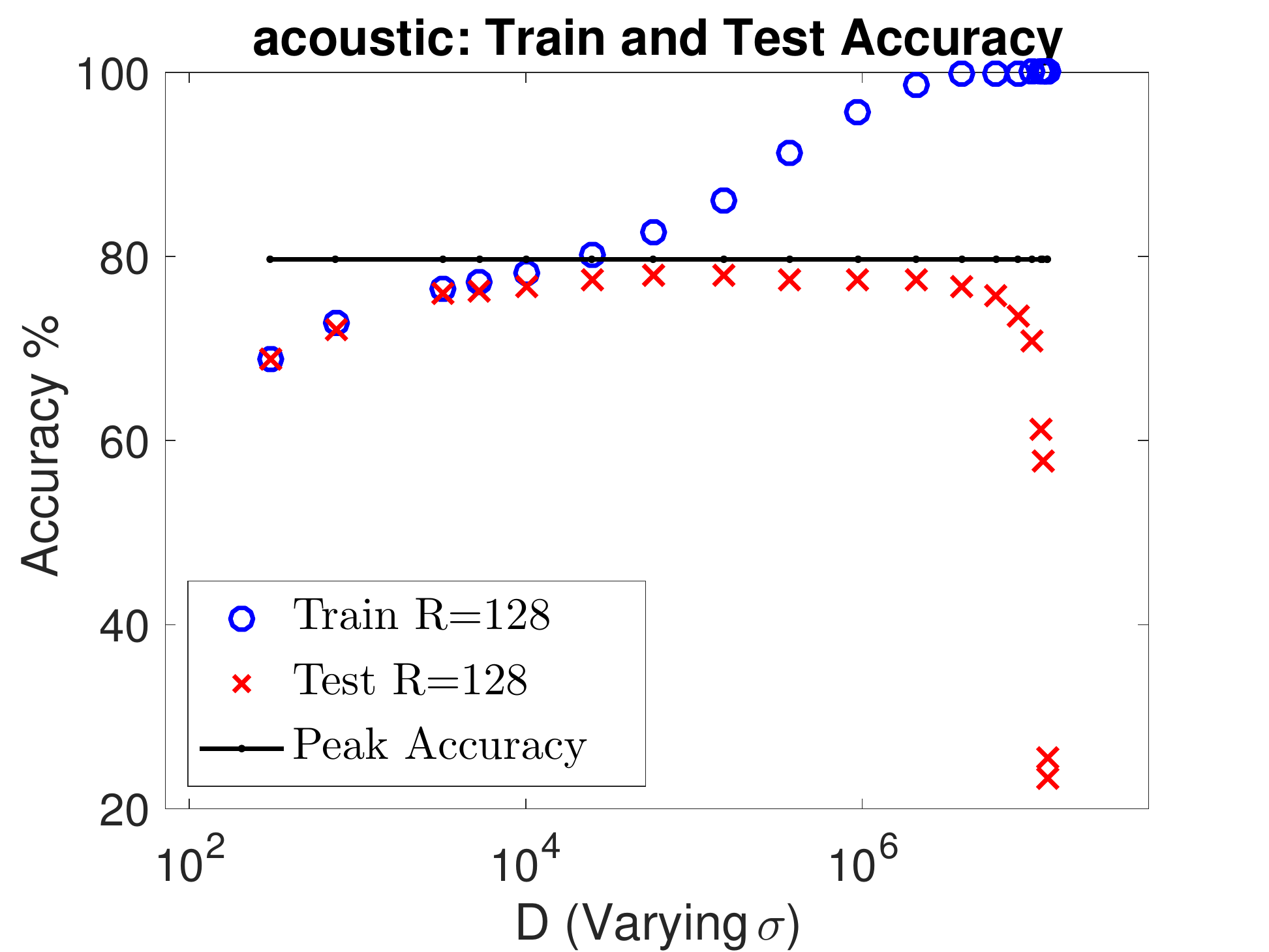}}
\subfigure[acoustic]{\includegraphics[width = 1.55in]{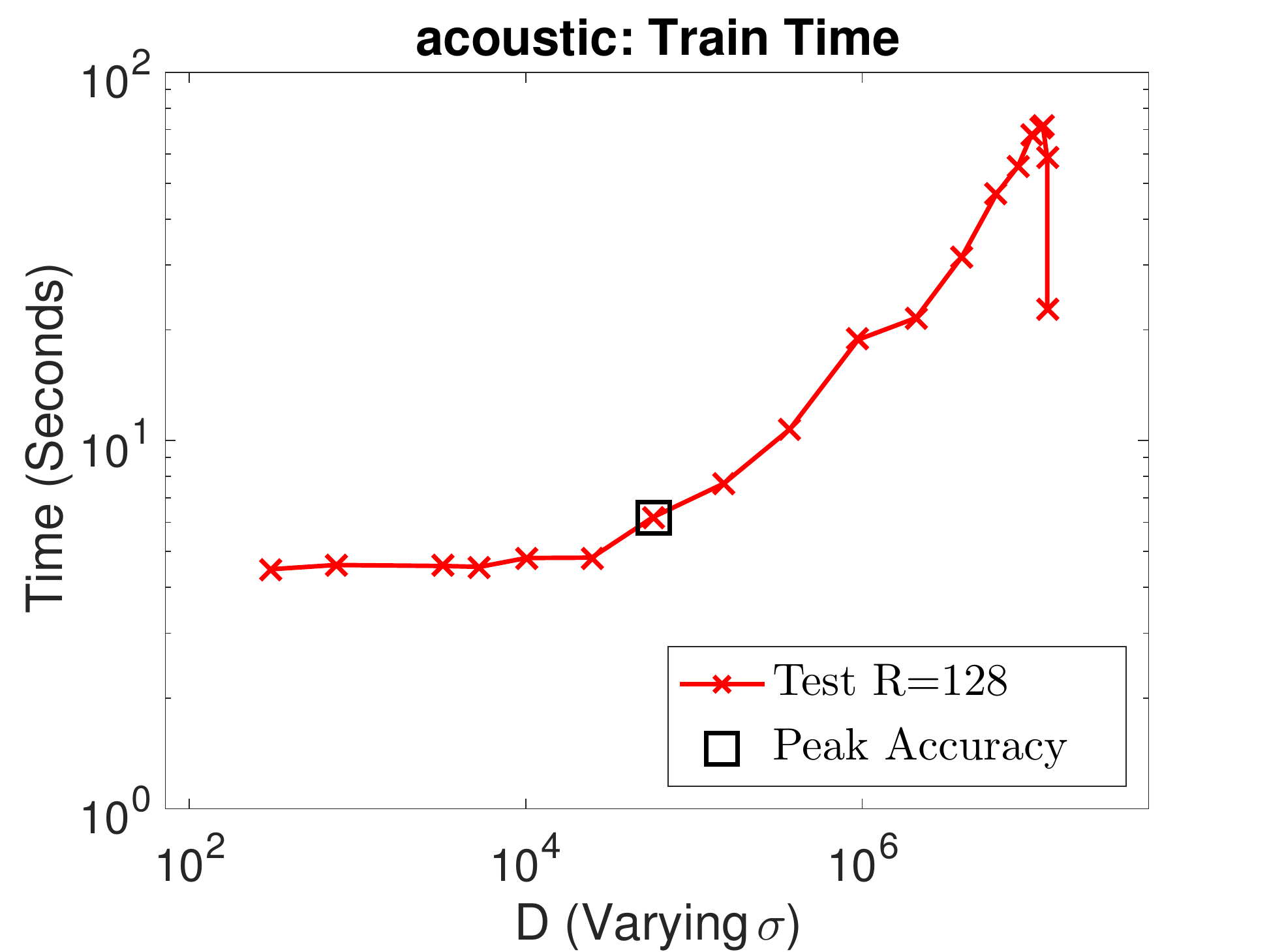}}
\subfigure[letter]{\includegraphics[width = 1.55in]{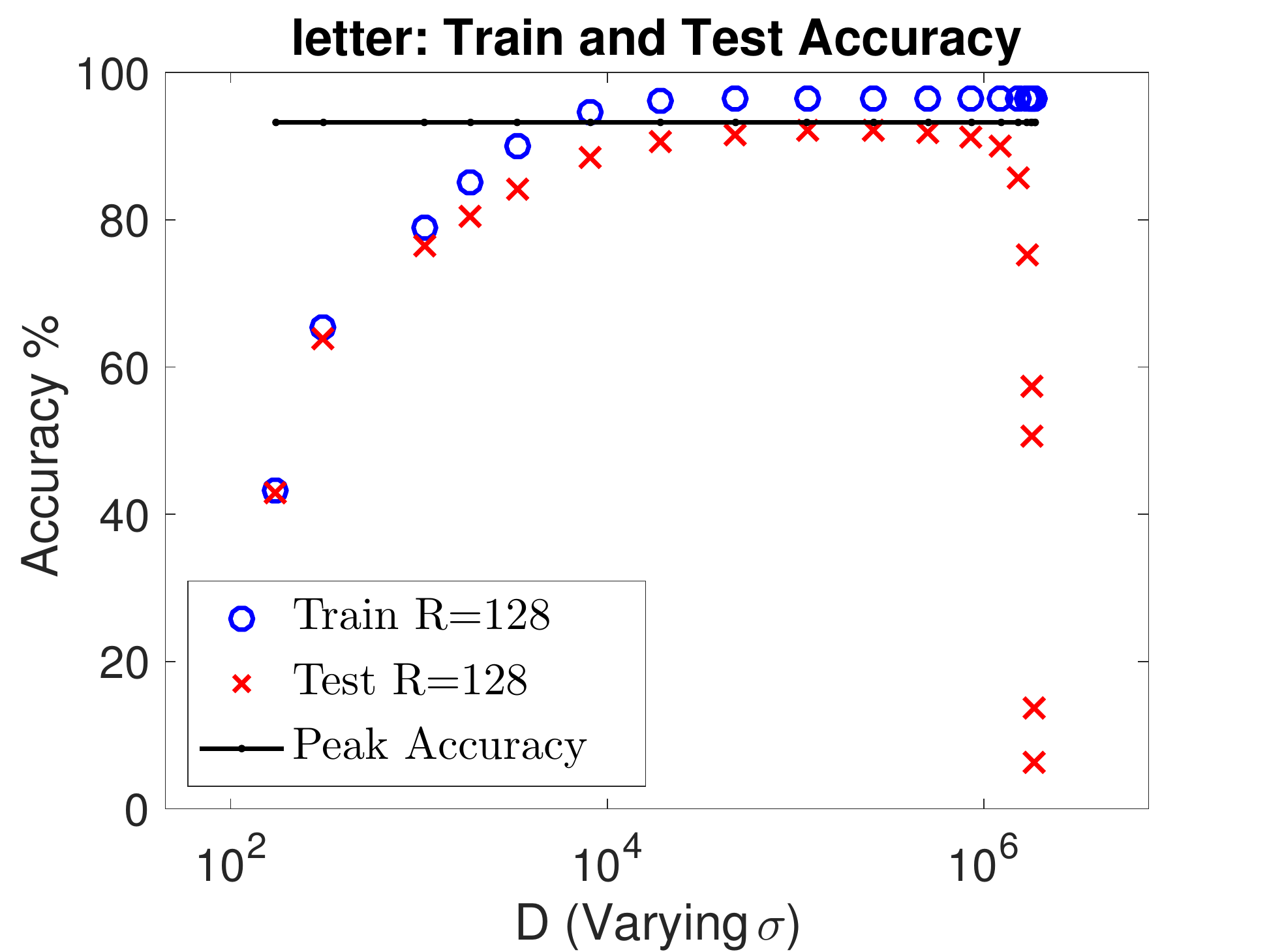}}
\subfigure[letter]{\includegraphics[width = 1.55in]{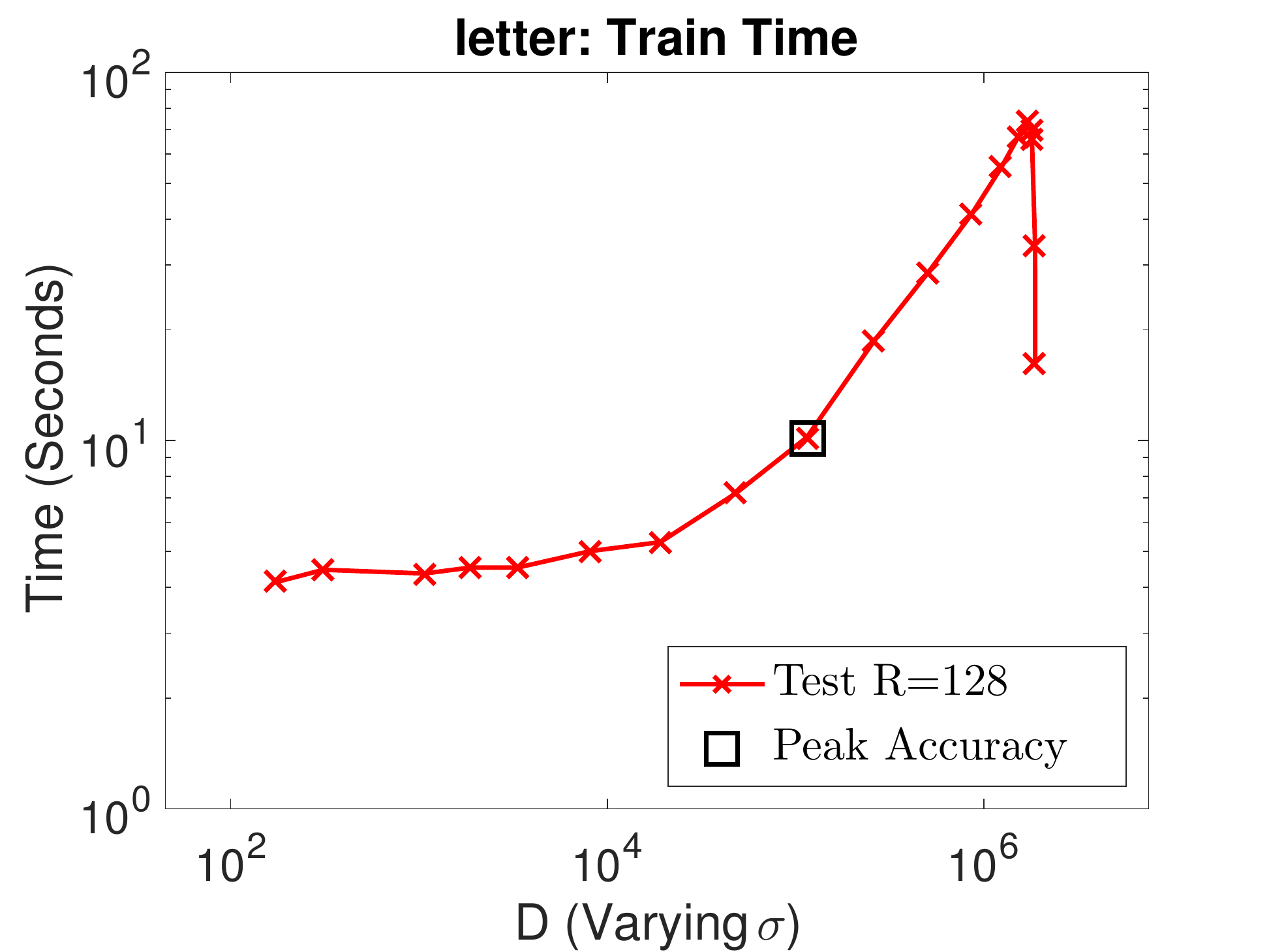}}
\caption{Train and test performance, and train time when varying $\sigma$ with fixed $R$. The black line and square box represent the best test performance of the exact kernel and RB respectively.}
\label{fig:TrTe_perf_time_lambda}
\end{figure*}

\begin{figure*}[!htb]
\centering
\subfigure[cadata]{\includegraphics[width = 1.55in]{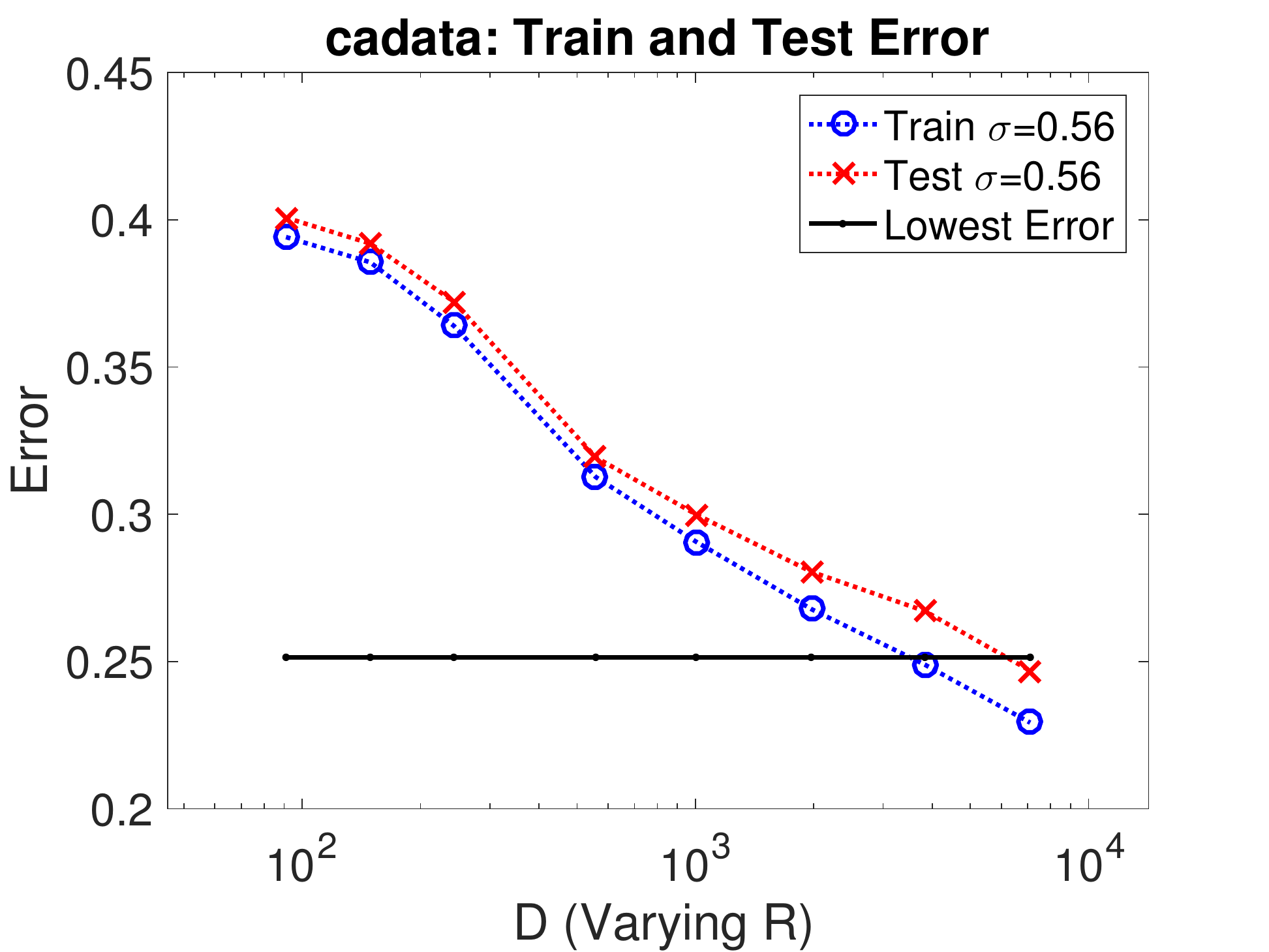}}
\subfigure[ijcnn1]{\includegraphics[width = 1.55in]{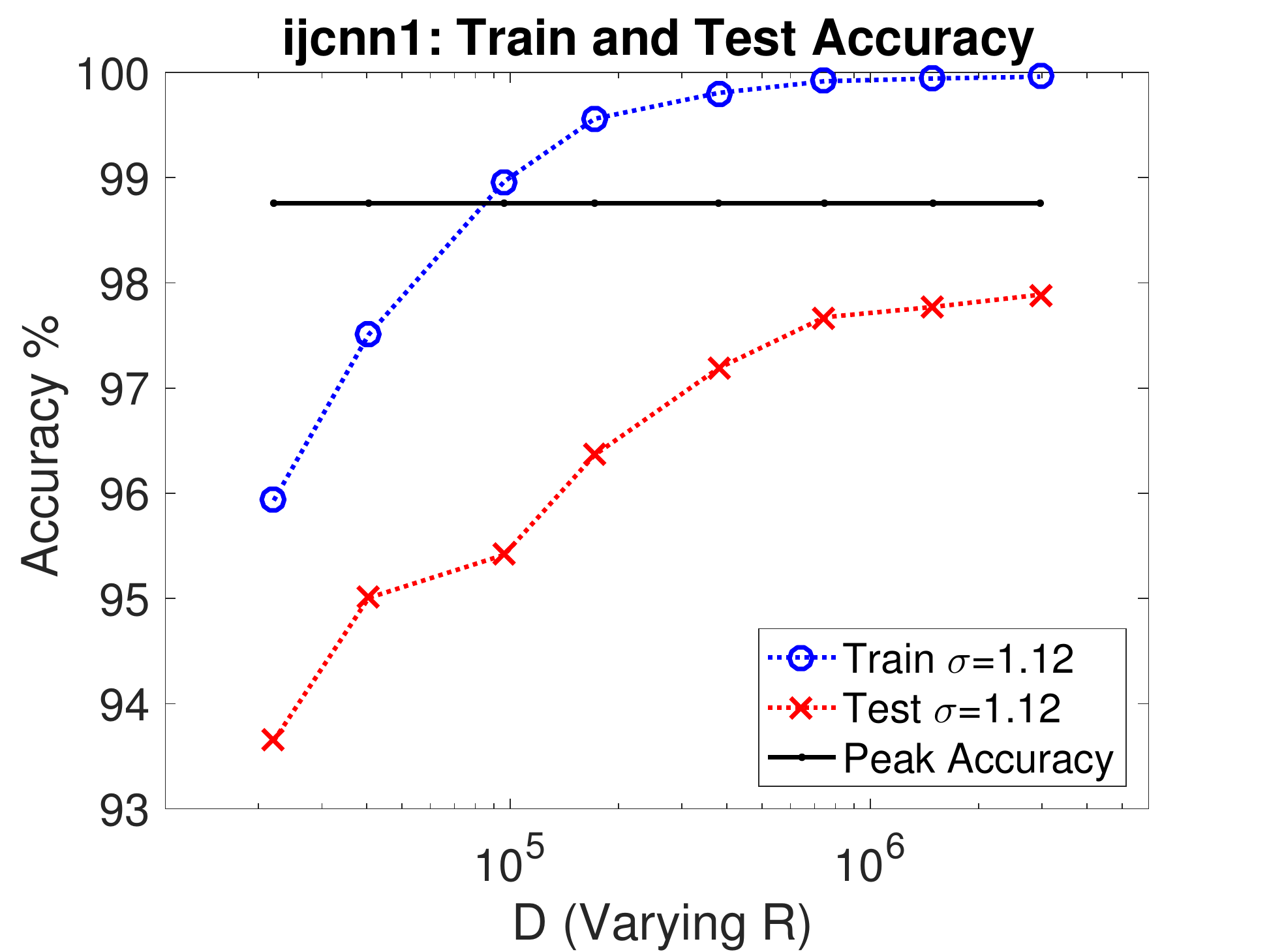}}
\subfigure[acoustic]{\includegraphics[width = 1.55in]{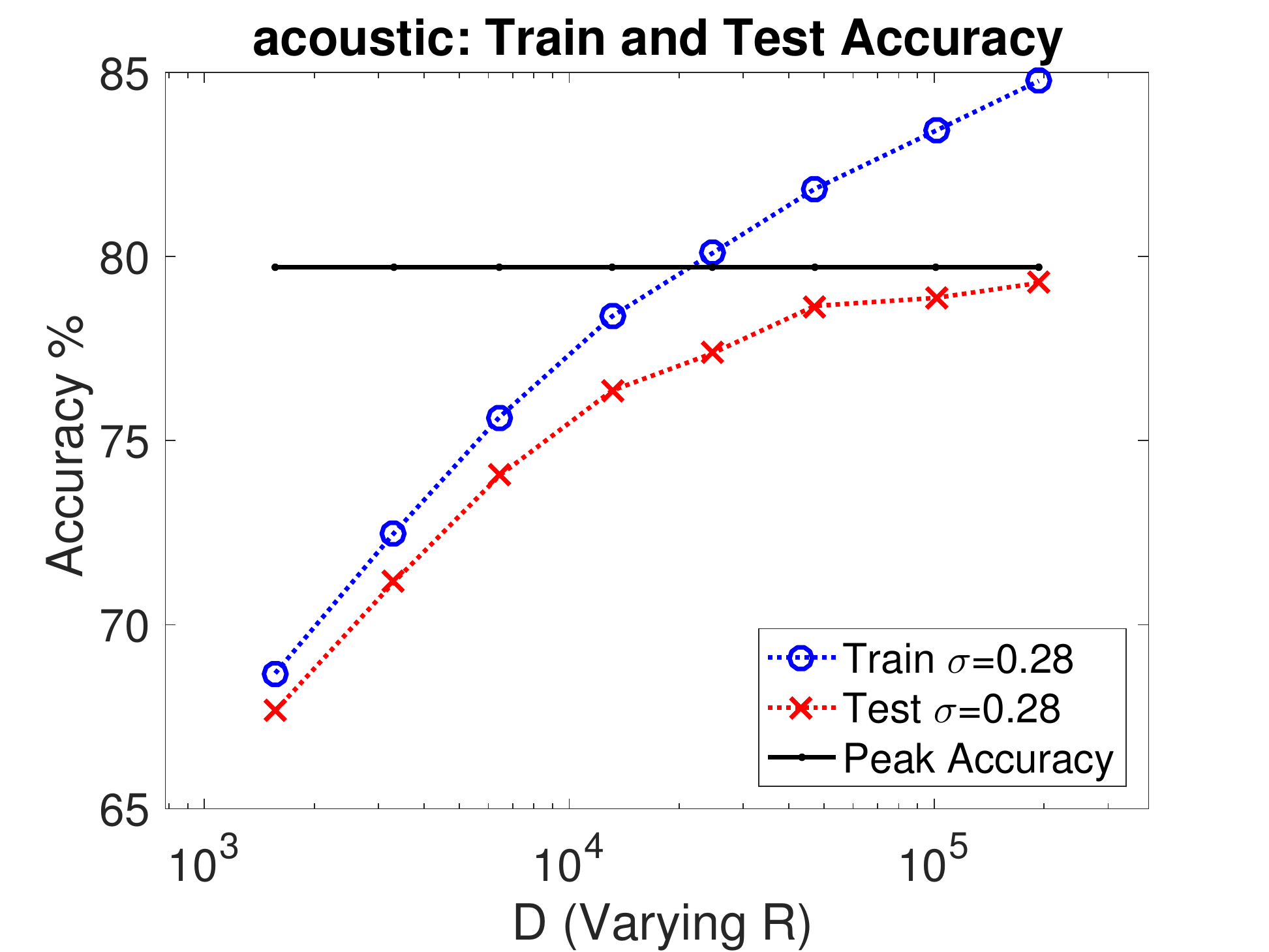}}
\subfigure[letter]{\includegraphics[width = 1.55in]{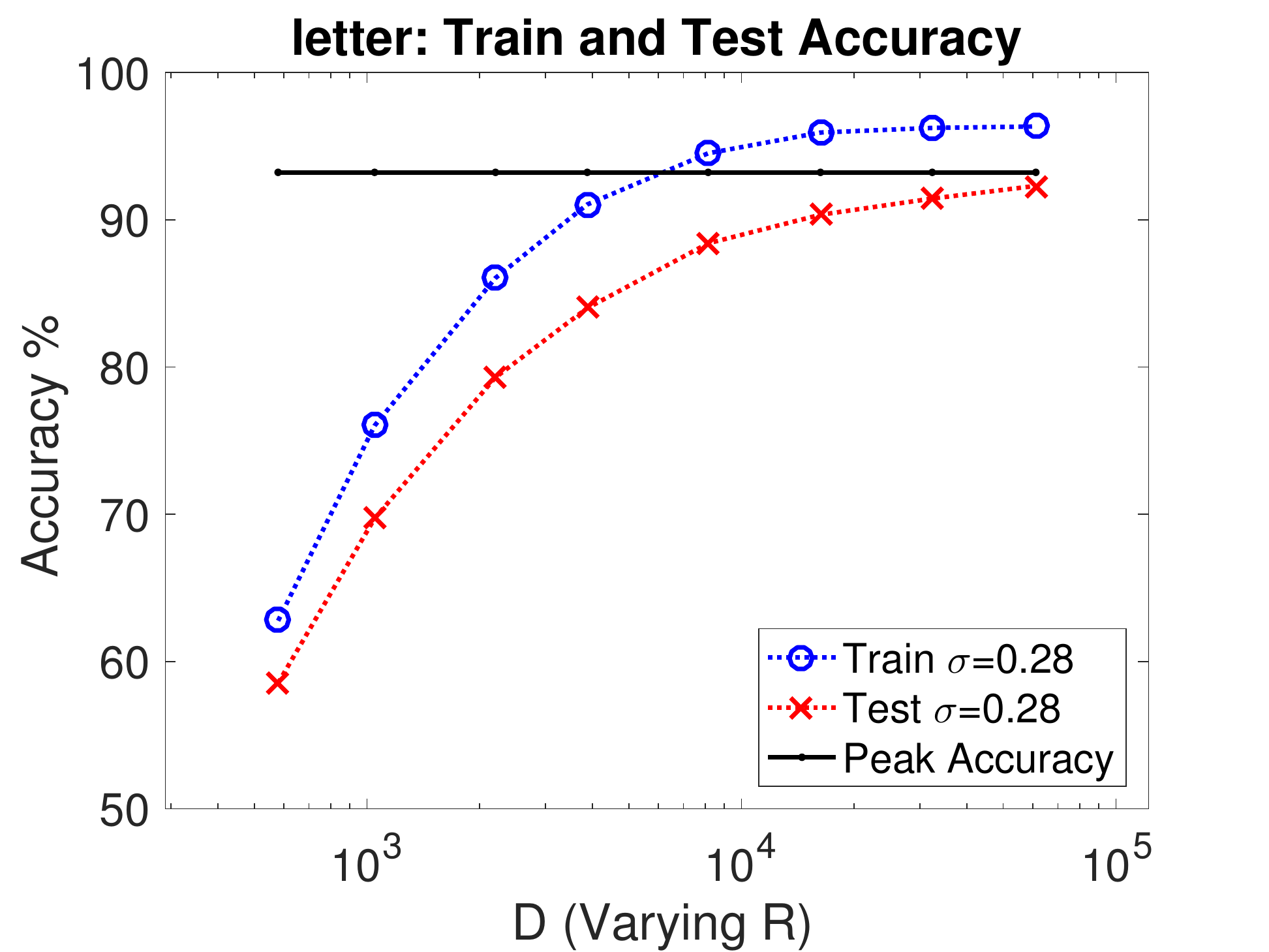}}
\caption{Train and test performance when varying $R$ with fixed $\sigma$.}
\label{fig:TrTe_perf_R}
\end{figure*}

\subsection{Performance Comparisons of All Methods}
We present a large sets of experiments to compare RB with other most popular low-rank kernel approximations, including RF \cite{rahimi2007random}, Nystr\"{o}m \cite{Seeger2000Nystrom}, and recently proposed independent block approximation \cite{Stein2014BlockDiag}. We also compare all methods with the exact kernel as a benchmark \cite{LearnKernels2001}. We do not report the results of the vanilla kernel on covtype and SUSY since the programs run out of memory. To make a fair comparison, we also apply CG on RB and Nystr\"{o}m directly on $Z$ to admit similar computational costs. Since the independent block kernel approximation approximates the kernel matrix directly, we employ direct solver of dense matrix for this method. In practice, the CG iterative solver has no need to solve in high precision \cite{JCLW2016ICCASP}, which has also been observed in our experiments. Thus, we set the tolerance to $1e-3$. 

Fig.\ref{fig:perf_time_R_mem_time_perf_group1} clearly demonstrates the superiority of RB compared to other low-rank kernels. For example, in the first column, RB significantly outperforms other methods in testing performance on all of these datasets, especially when $R$ is relatively small. This is because RB enjoys much faster convergence rate to the optimal function than other methods. The advantage generally diminishes when $R$ increases to reasonably large. However, for some large datasets such as covtype and SUSY, increasing number of random features or $R$ boosts the performance extremely slow. This is consistent with our analysis that RB enjoys its fast convergence rate of $O(1/(\kappa R))$ while other methods has slow convergence rates $O(1/\sqrt{R})$. The third and fourth columns further promote the insights about how many number of random features or how large rank $R$ that is needed for achieving similar performance of RB. In particular, RB is often between one and three orders of magnitude faster and less memory consumptions than other methods. 

In the second column, we also observe that the training time of all low-rank kernels are linear with $R$, which is expected since all these methods has computational complexity of $O(kNR)$. The difference in training time between these low-rank kernels is only within some constant factors. However, we point out that the computations of RF, Nystr\"{o}m and independent block approximation are mainly carried out by the high-optimized BLAS library since they are dense matrices. In contrast, the computations of RB are most involved in sparse matrix operations, which are self-implemented and not yet optimized. In addition, more advanced sparse matrix techniques such as preconditioning can be explored to significantly accelerate the computation, which we leave it as future work.

\begin{figure*}[!htb]
\centering
\subfigure[cadata]{\includegraphics[width = 1.55in]{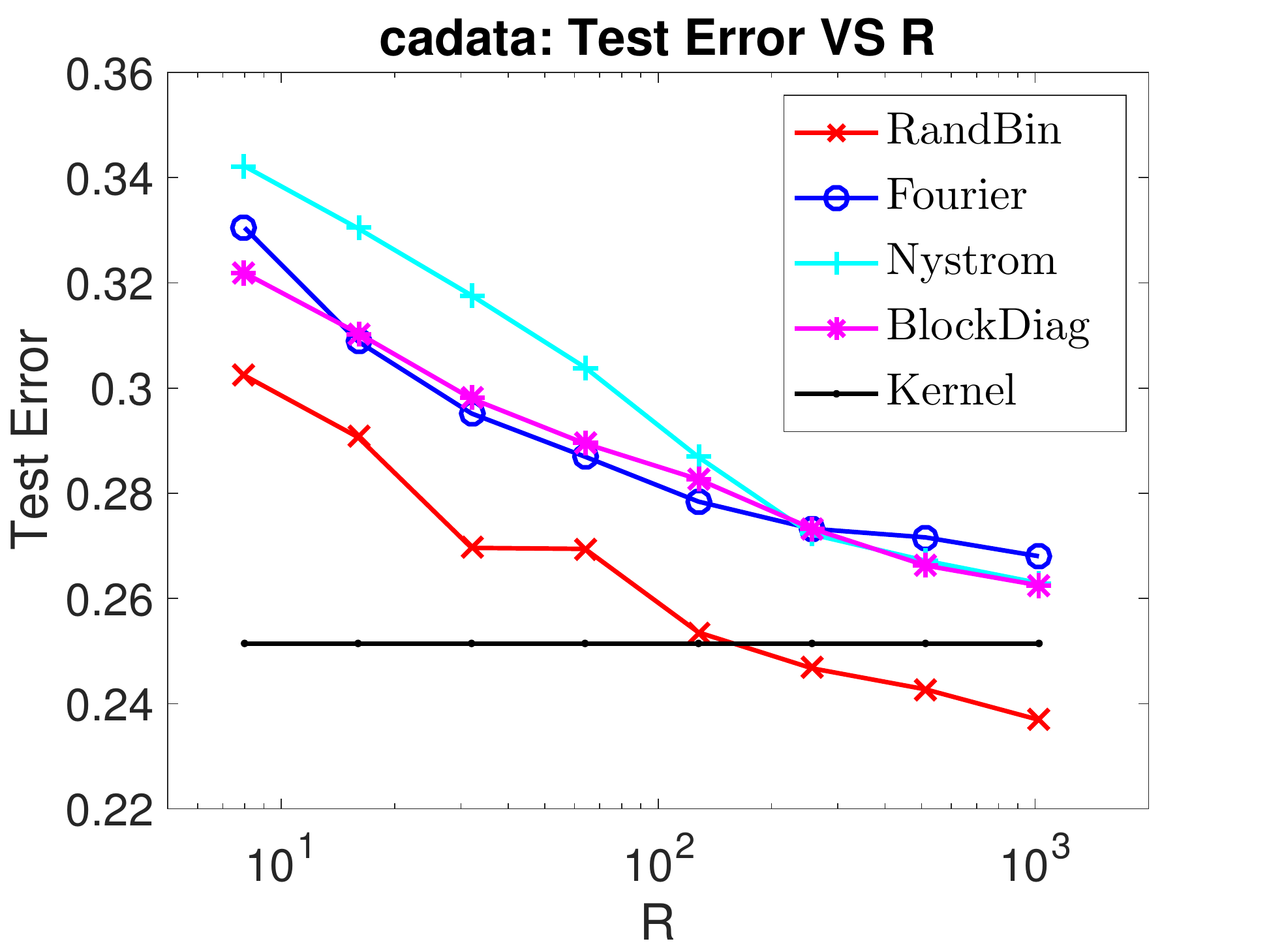}}
\subfigure[cadata]{\includegraphics[width = 1.55in]{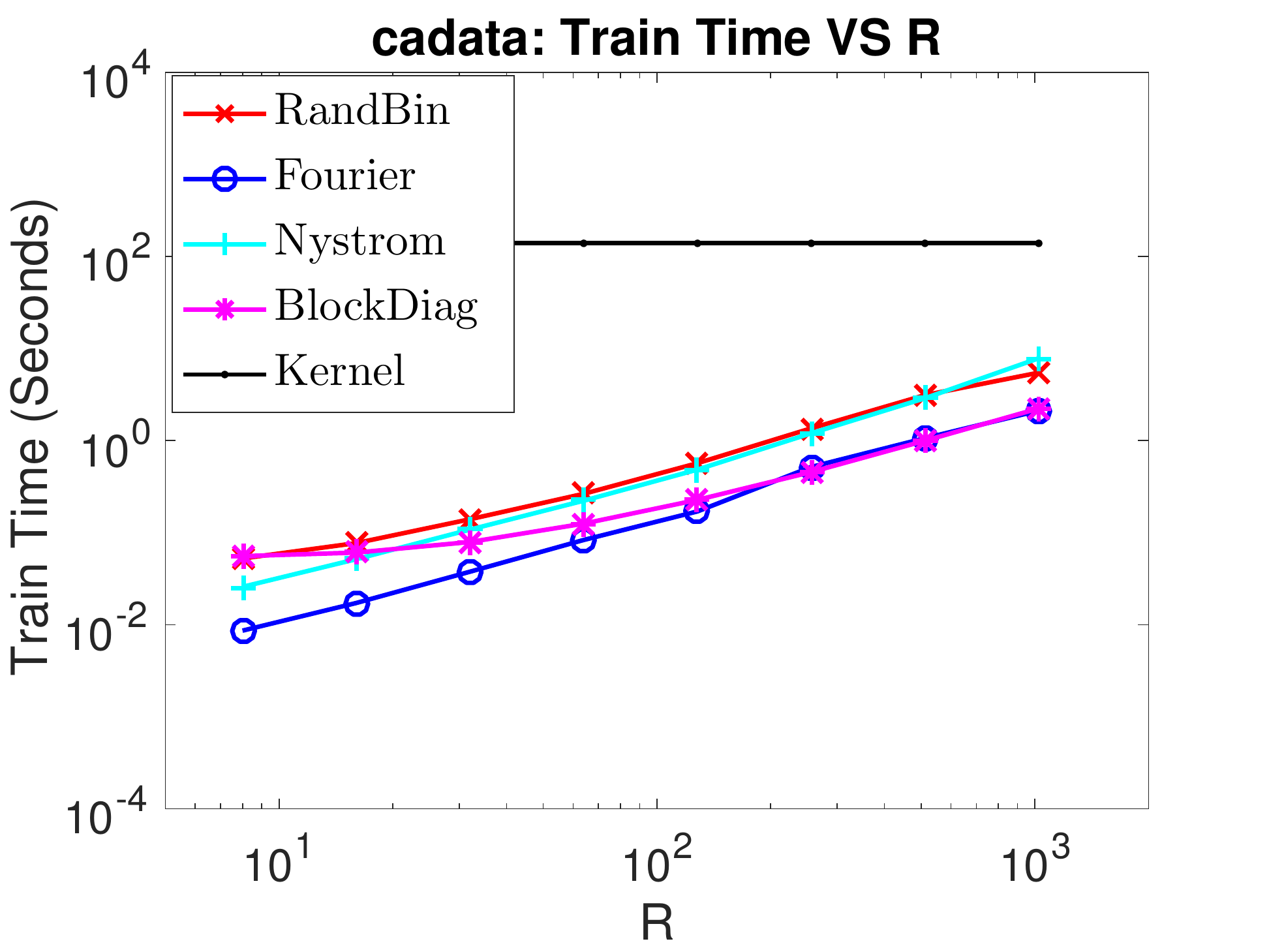}}
\subfigure[cadata]{\includegraphics[width = 1.55in]{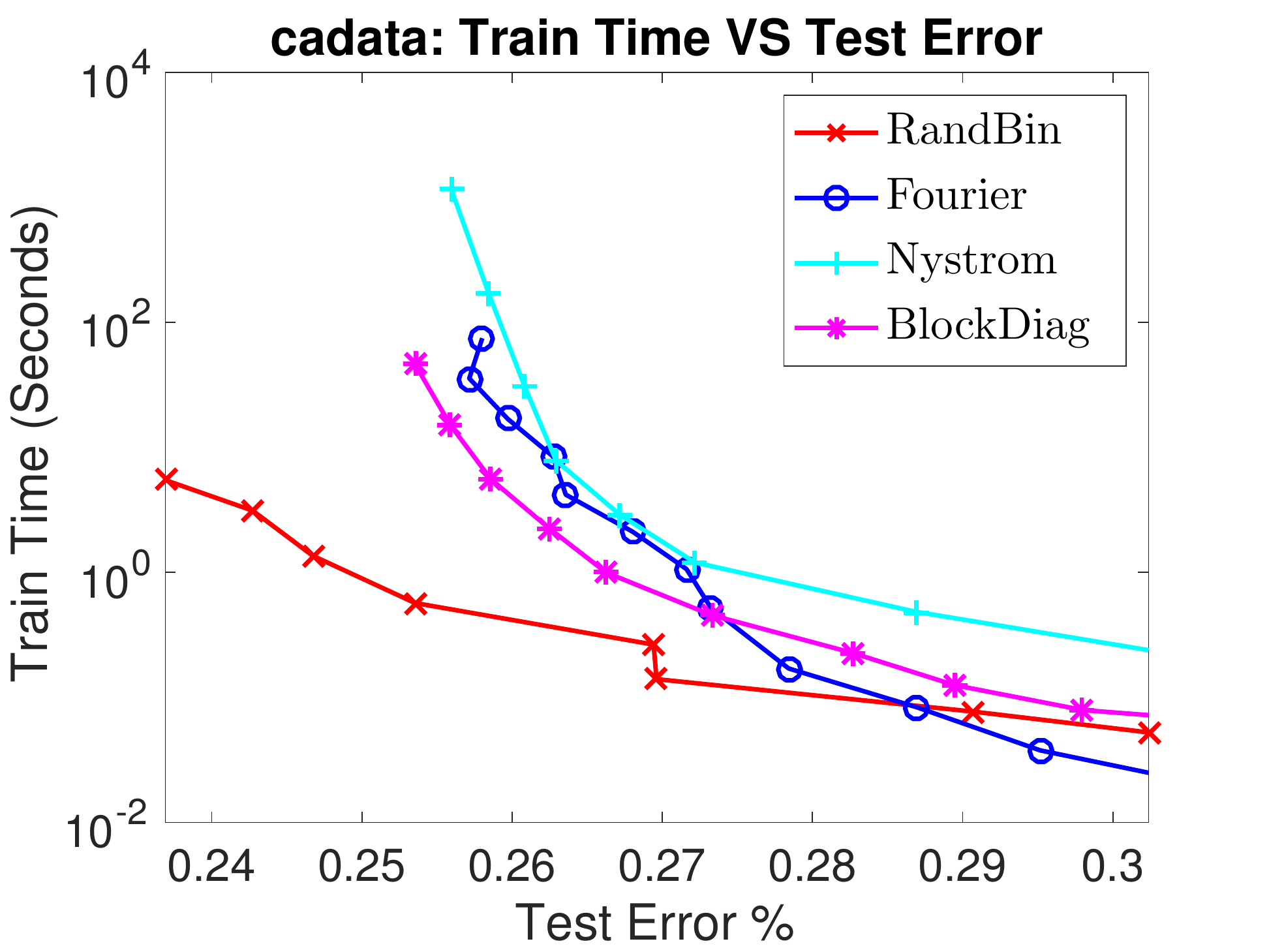}}
\subfigure[cadata]{\includegraphics[width = 1.55in]{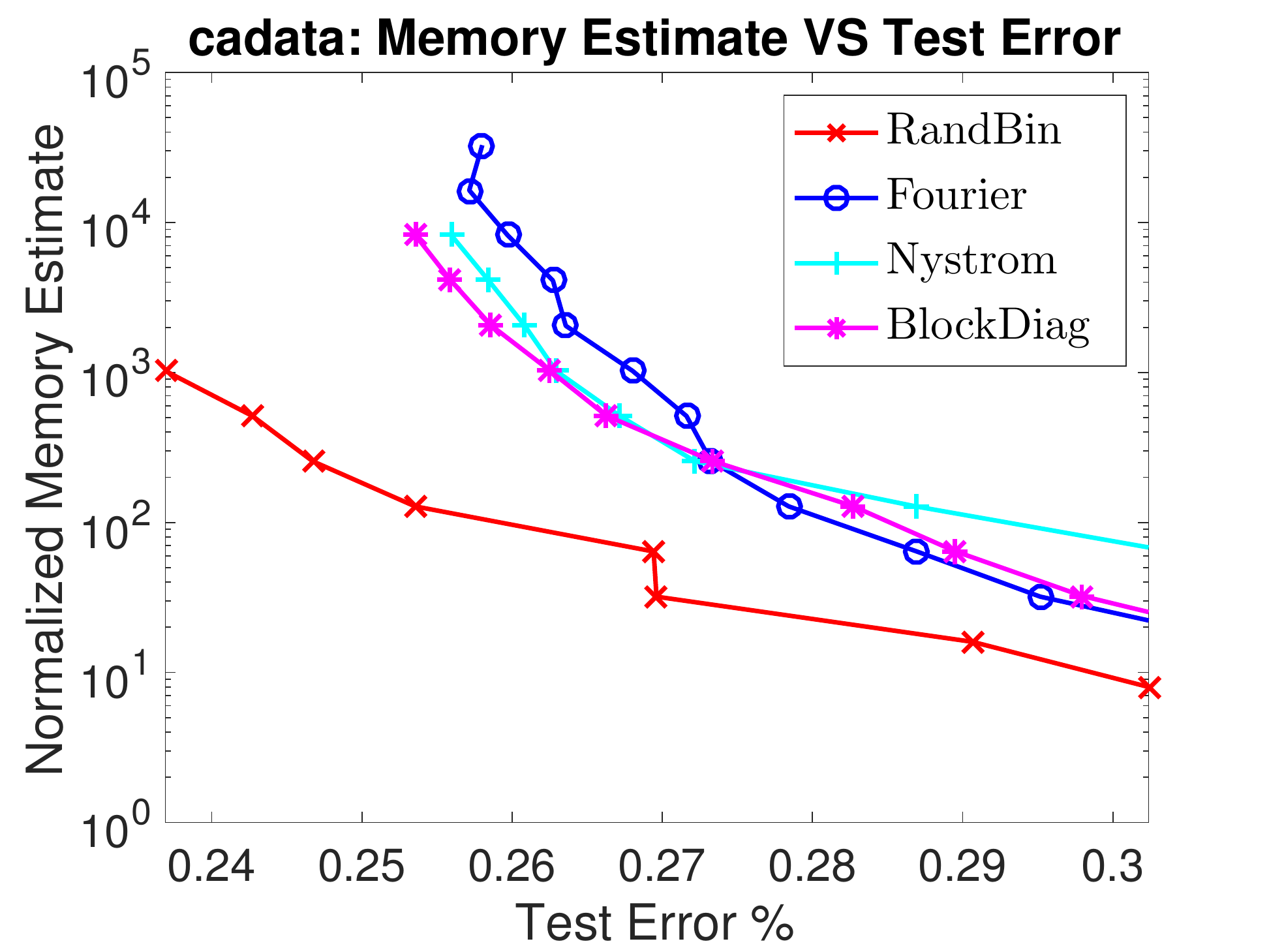}}
\subfigure[ijcnn1]{\includegraphics[width = 1.55in]{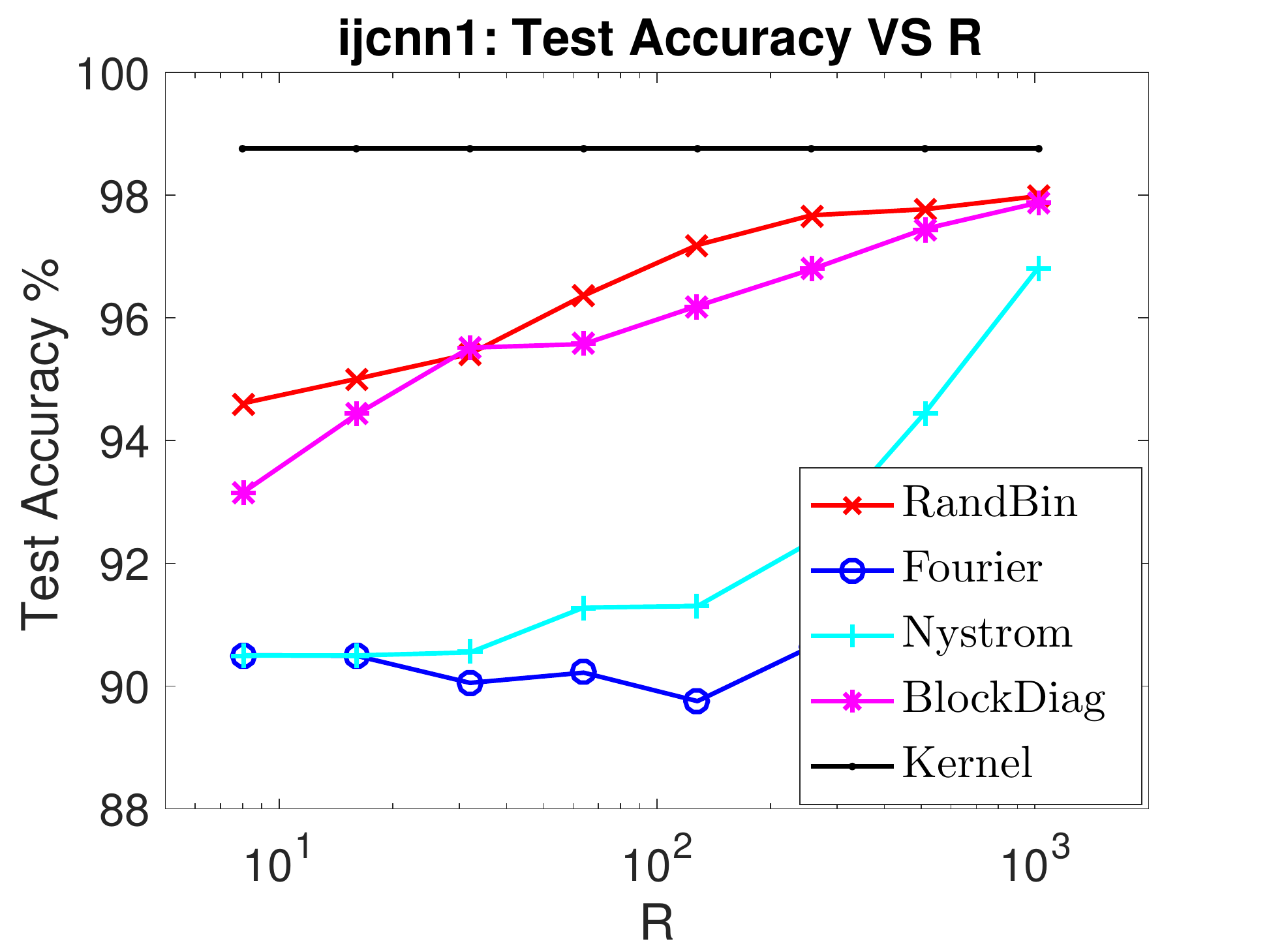}}
\subfigure[ijcnn1]{\includegraphics[width = 1.55in]{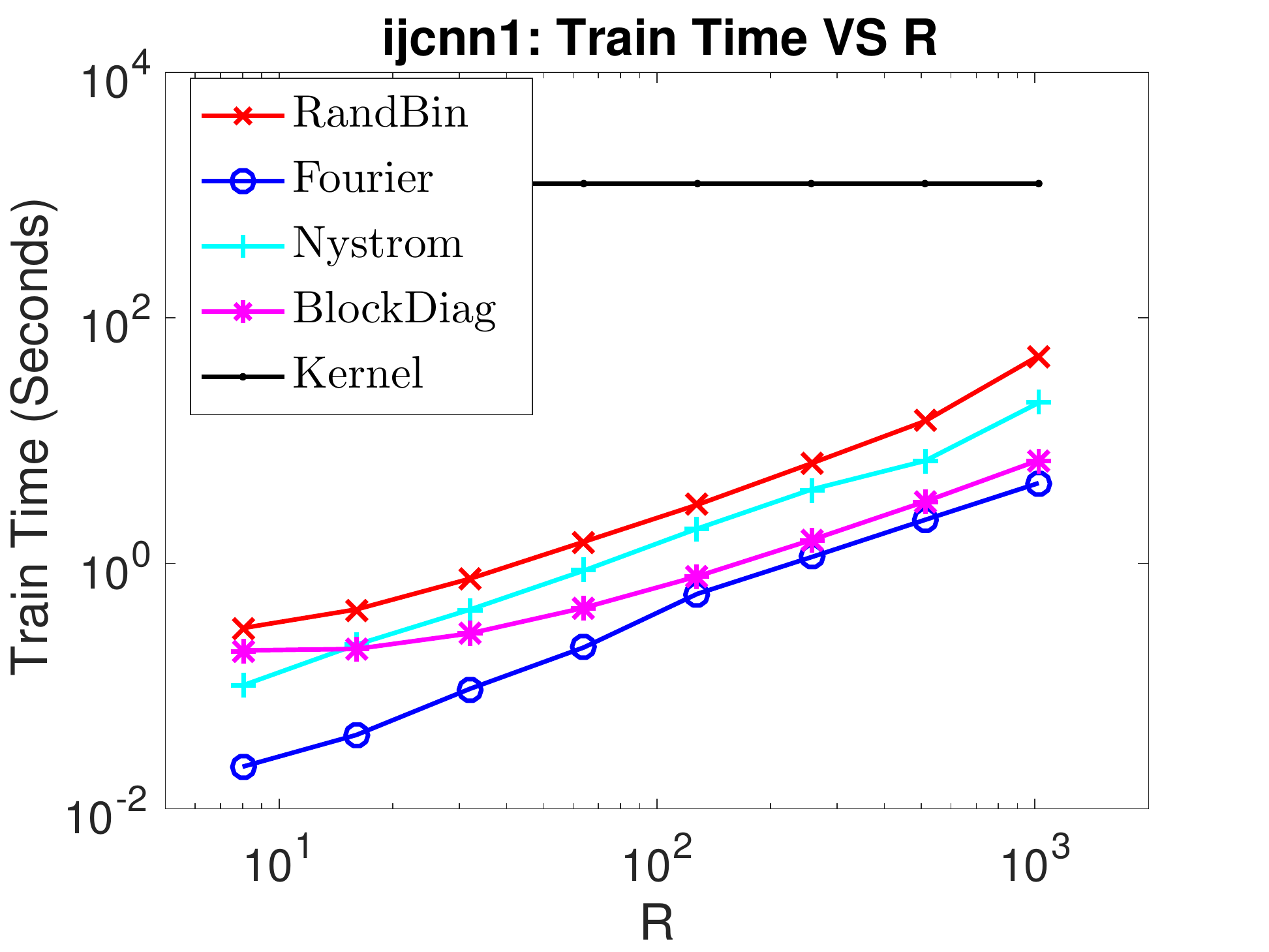}}
\subfigure[ijcnn1]{\includegraphics[width = 1.55in]{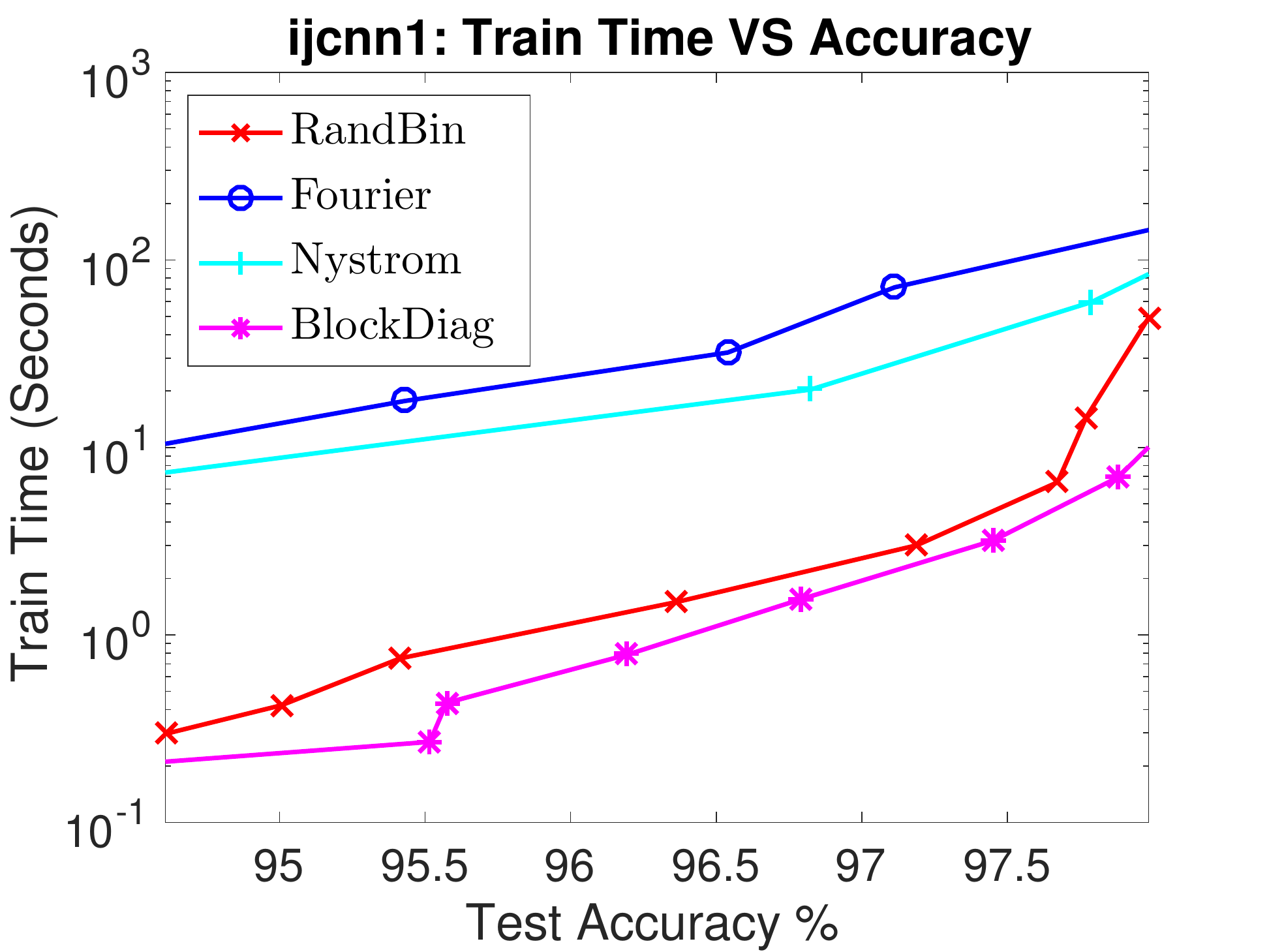}}
\subfigure[ijcnn1]{\includegraphics[width = 1.55in]{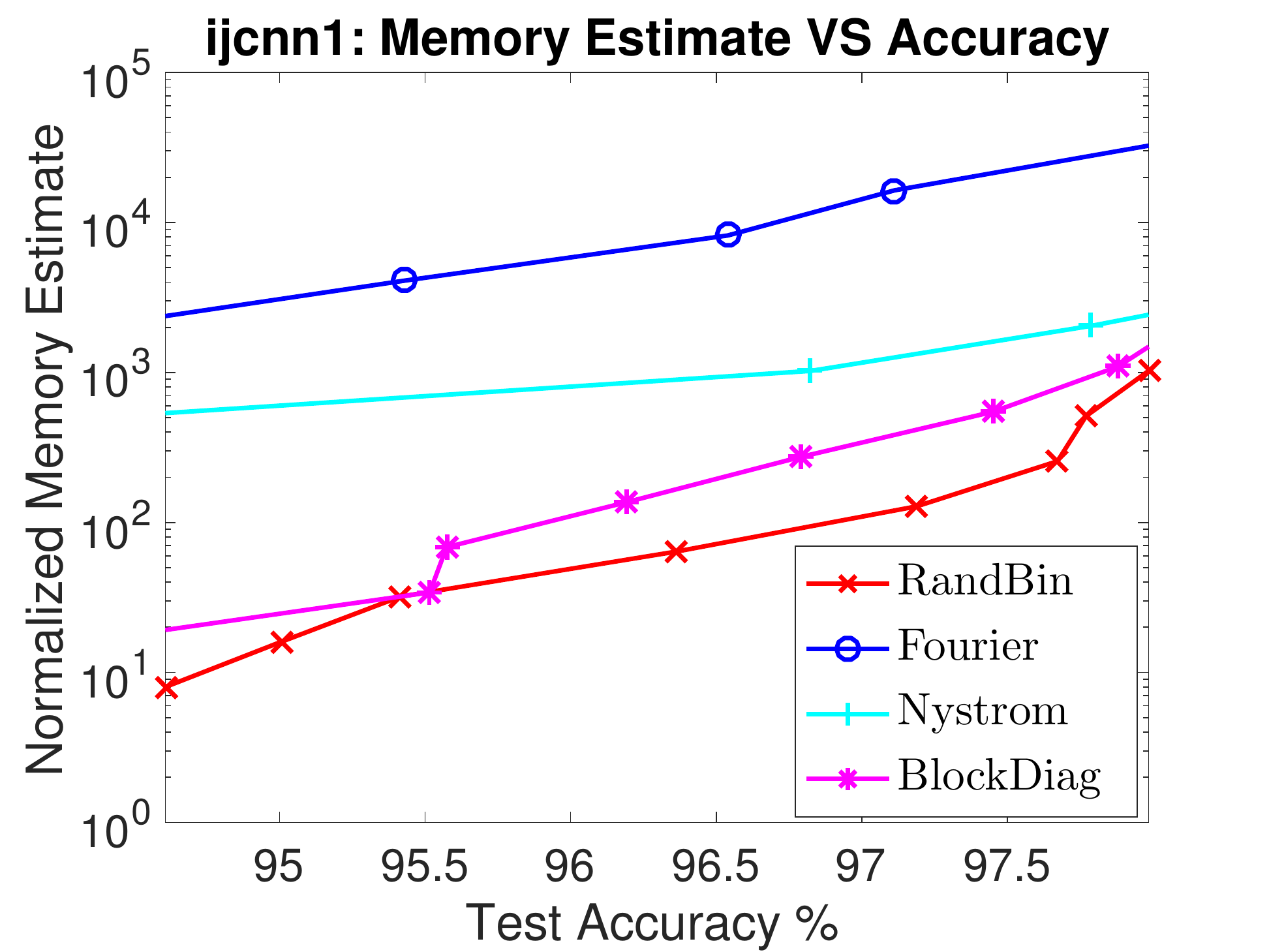}}
\subfigure[covtype]{\includegraphics[width = 1.55in]{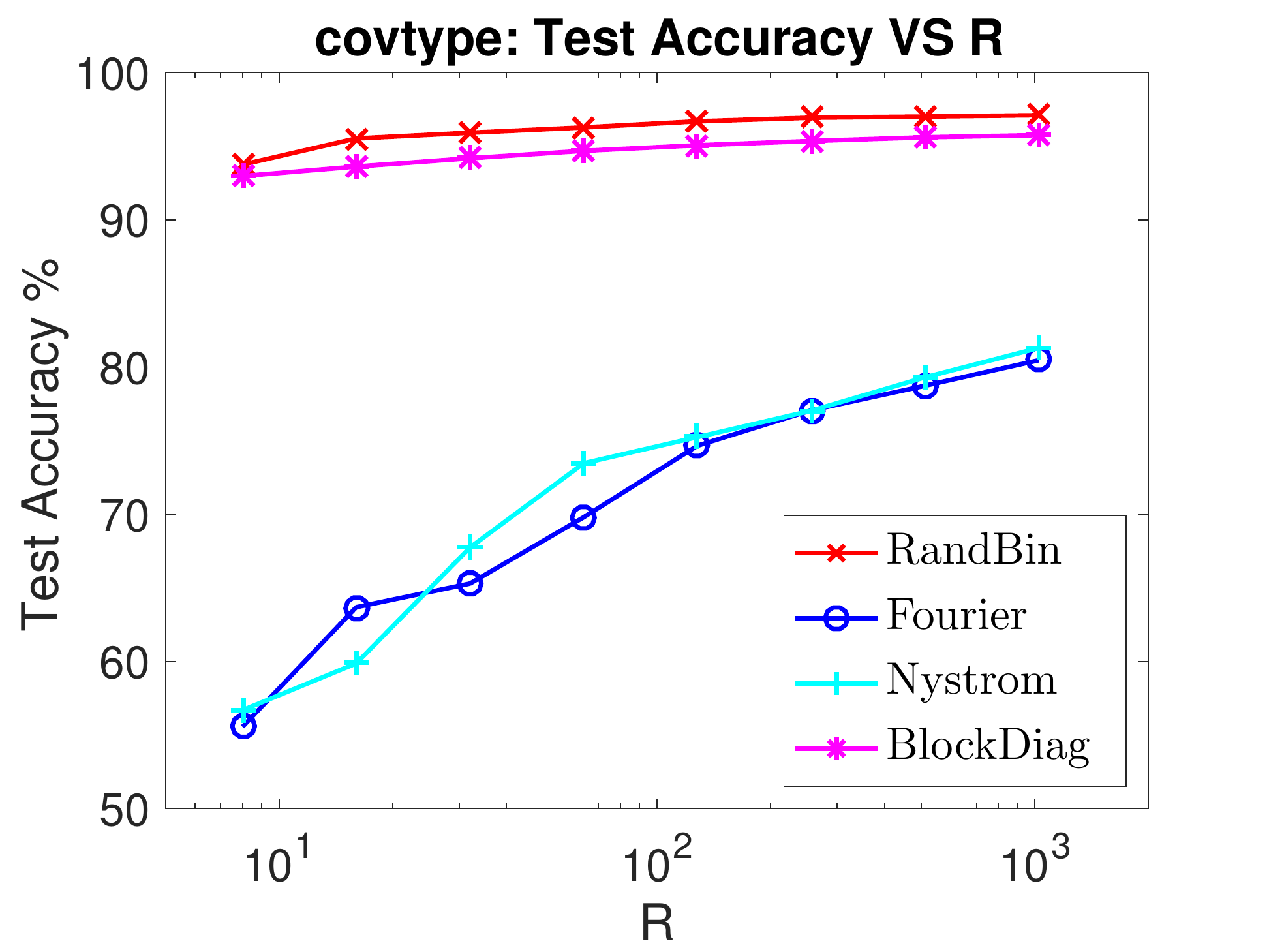}}
\subfigure[covtype]{\includegraphics[width = 1.55in]{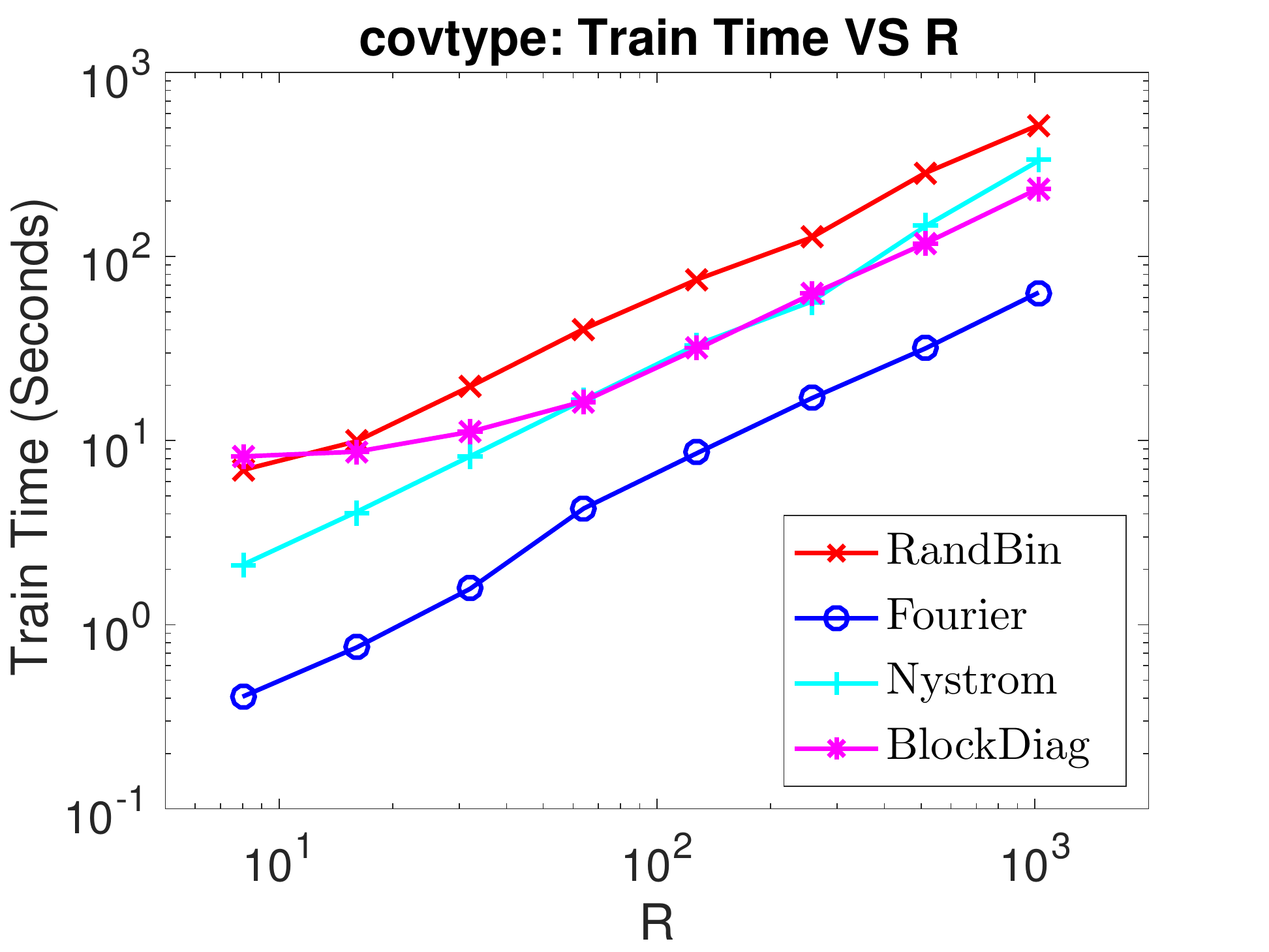}}
\subfigure[covtype]{\includegraphics[width = 1.55in]{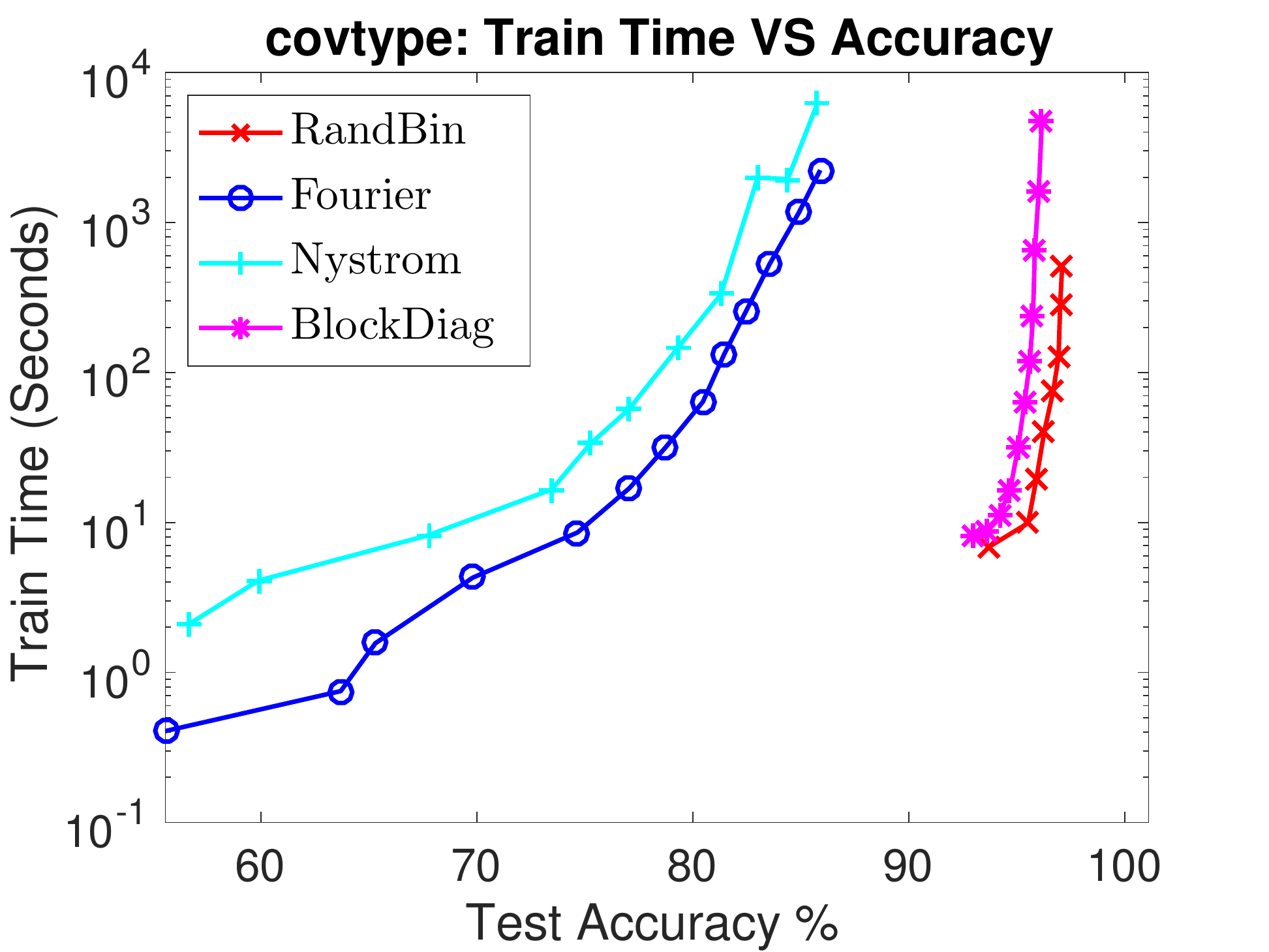}}
\subfigure[covtype]{\includegraphics[width = 1.55in]{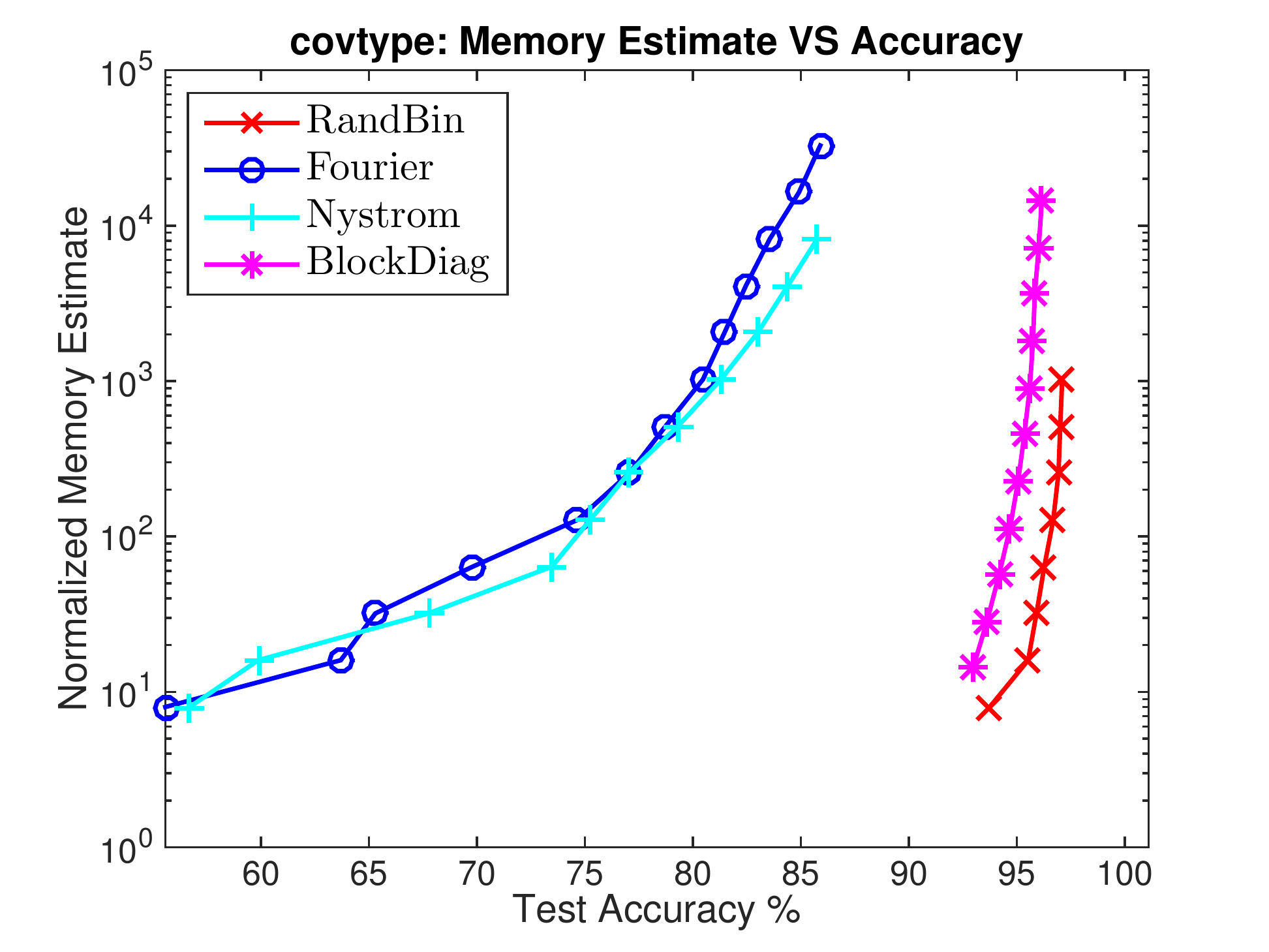}}
\subfigure[SUSY]{\includegraphics[width = 1.55in]{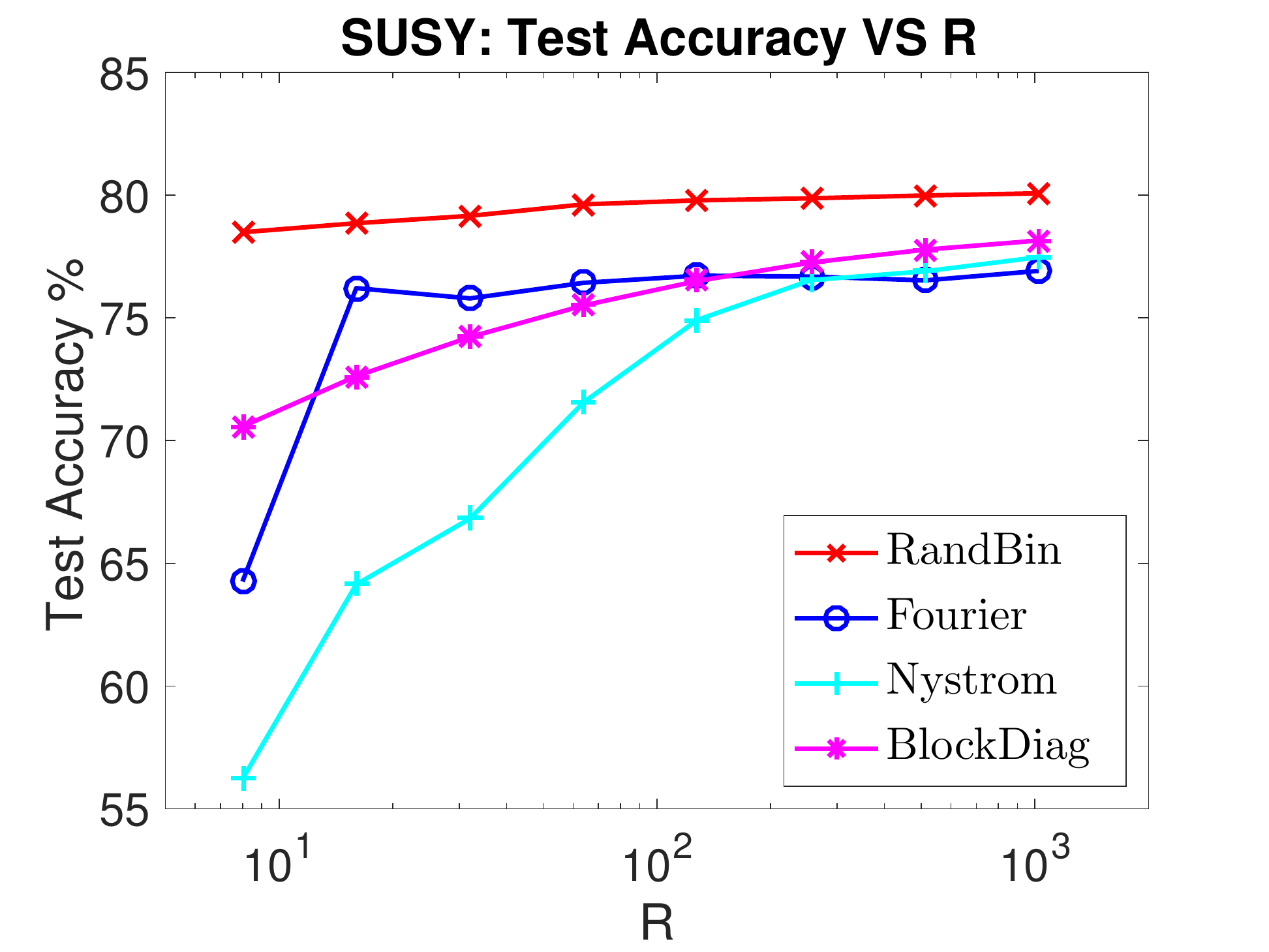}}
\subfigure[SUSY]{\includegraphics[width = 1.55in]{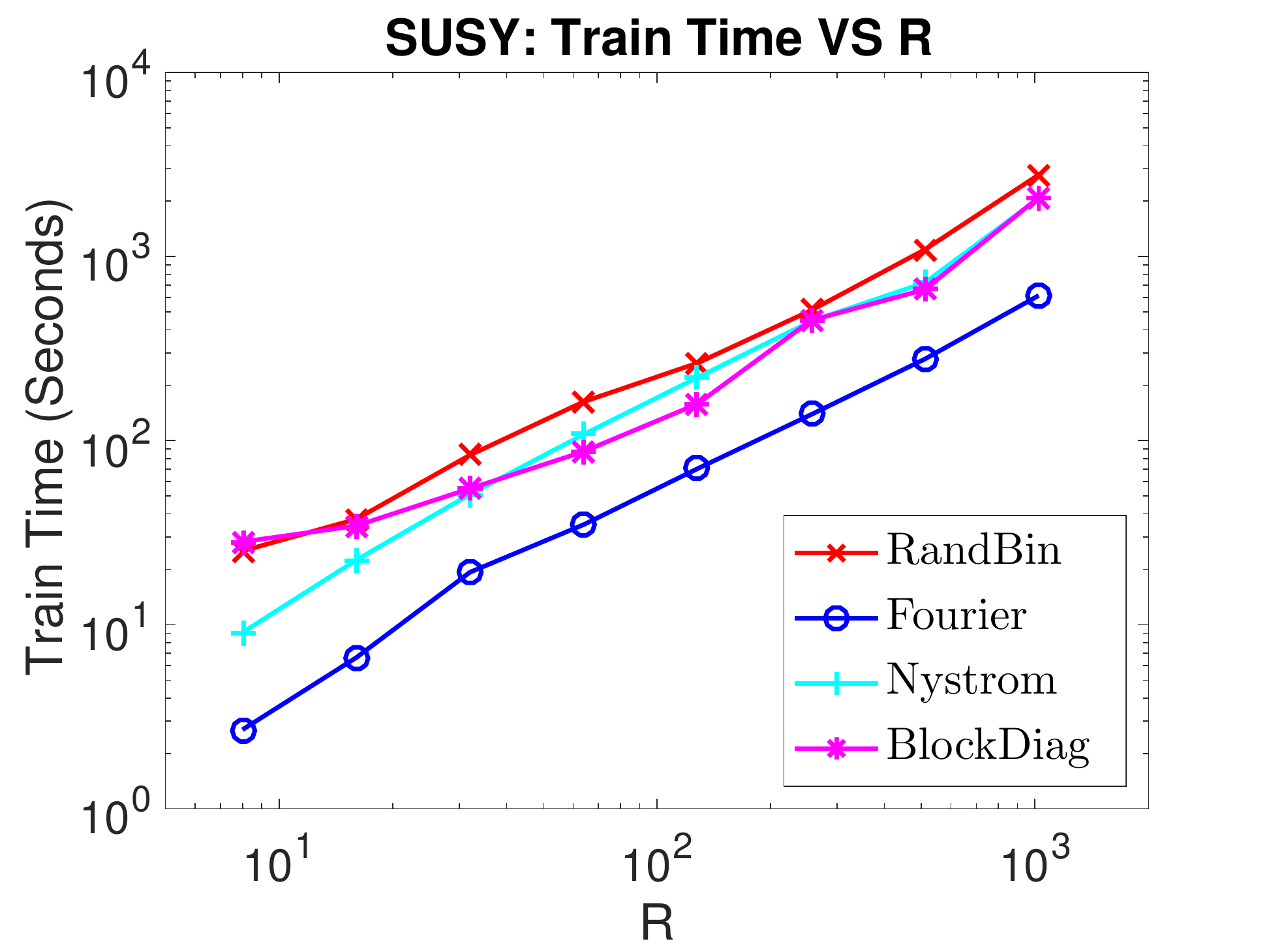}}
\subfigure[SUSY]{\includegraphics[width = 1.55in]{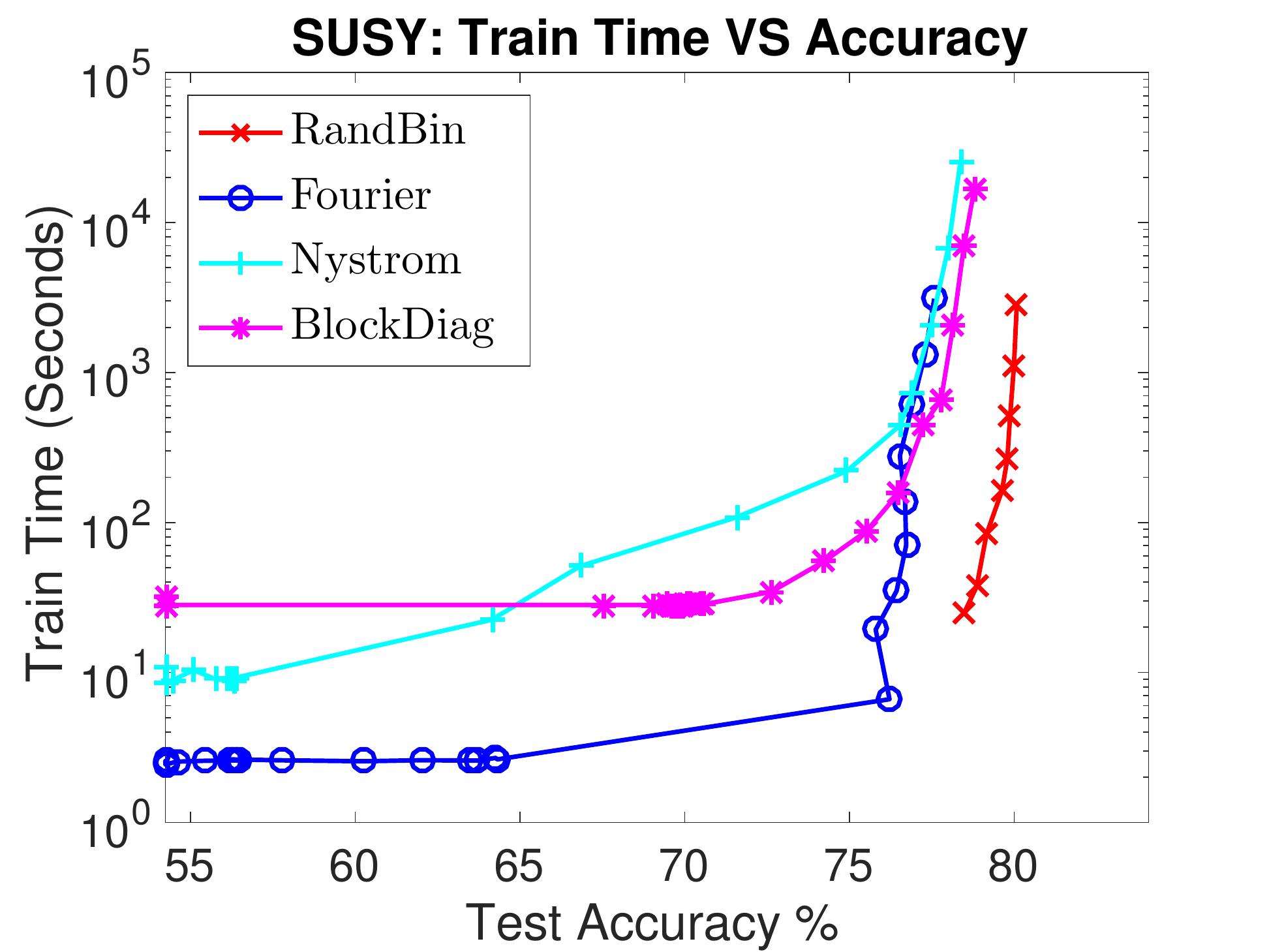}}
\subfigure[SUSY]{\includegraphics[width = 1.55in]{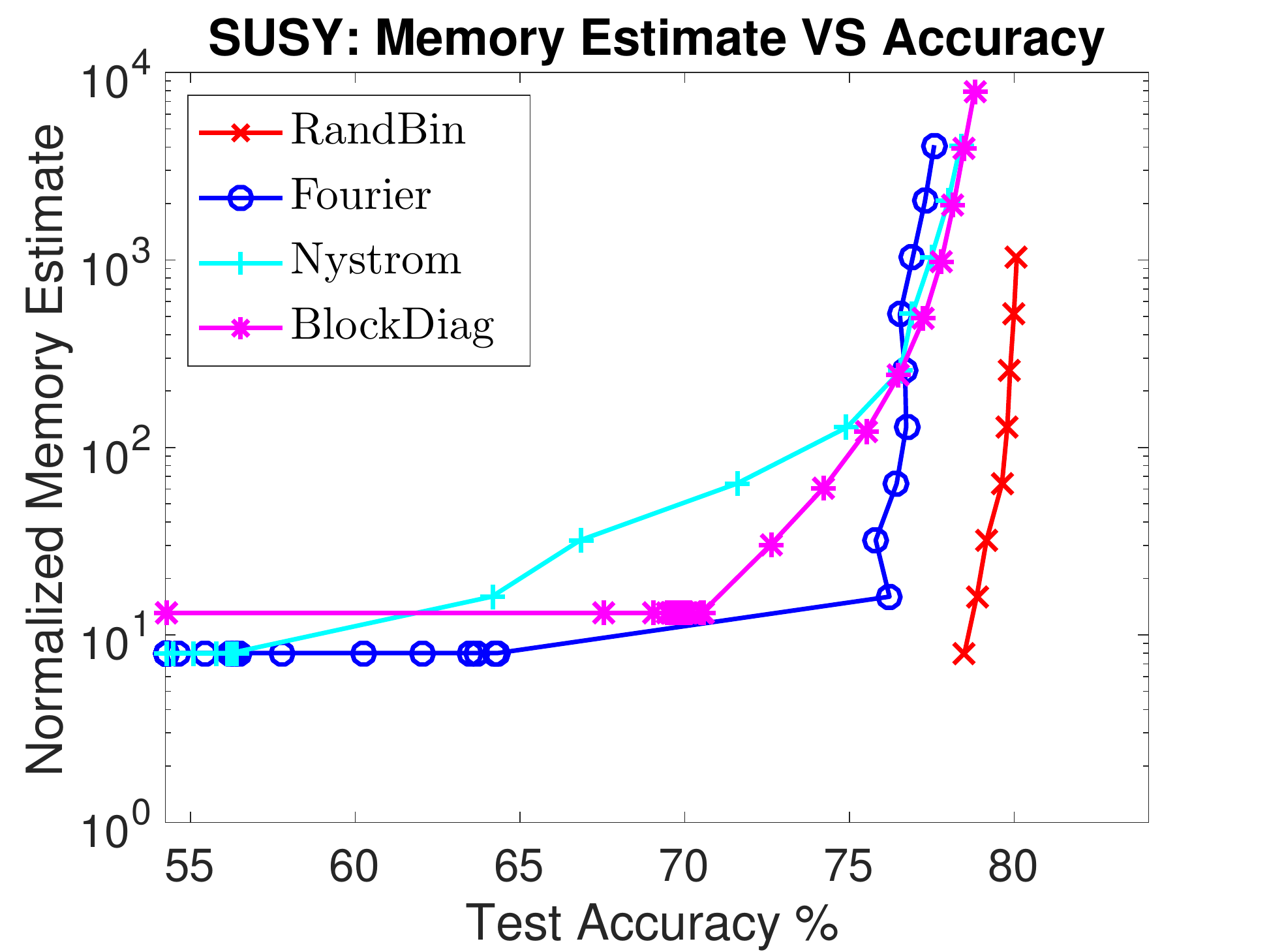}}
\subfigure[mnist]{\includegraphics[width = 1.55in]{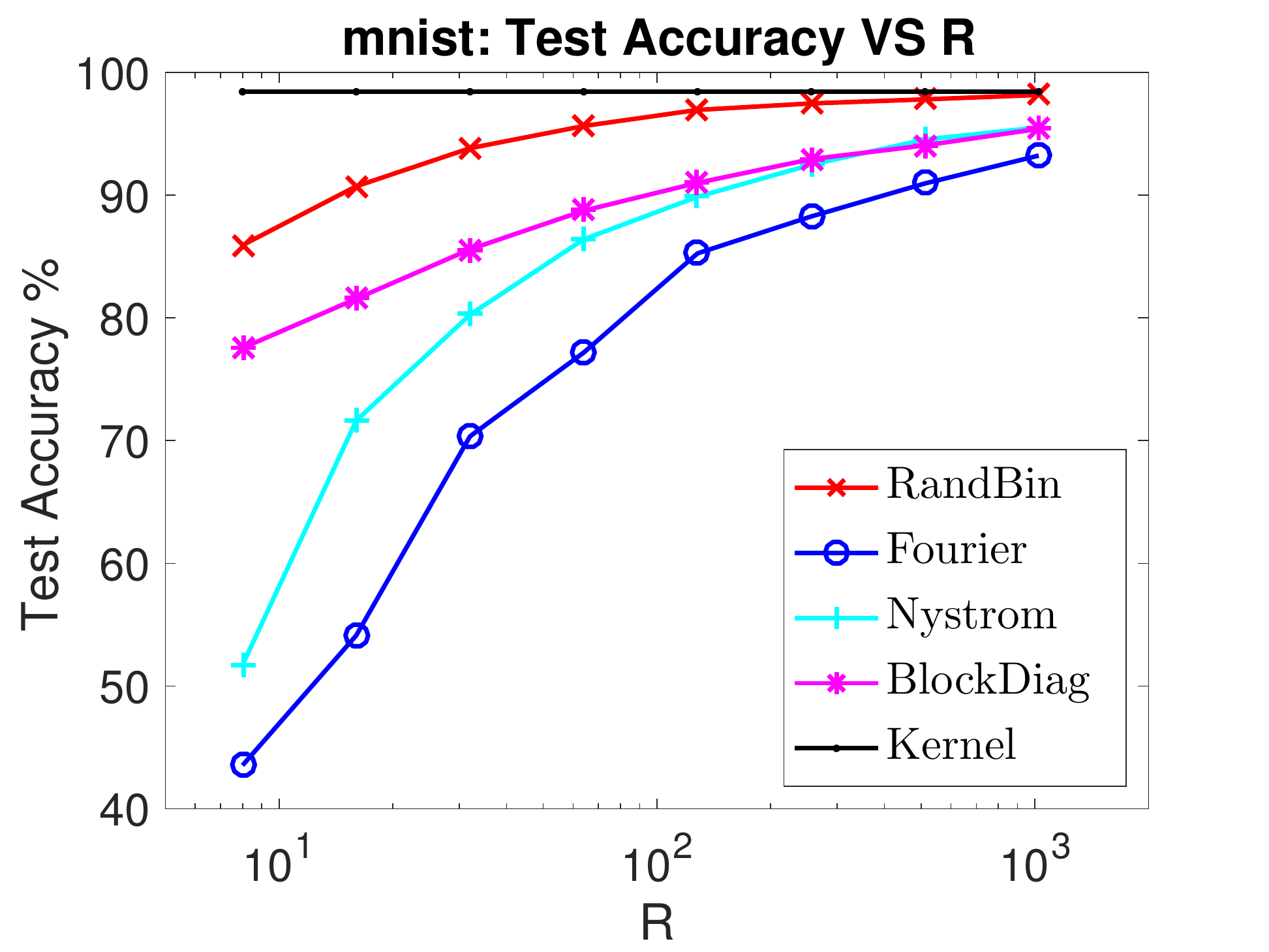}}
\subfigure[mnist]{\includegraphics[width = 1.55in]{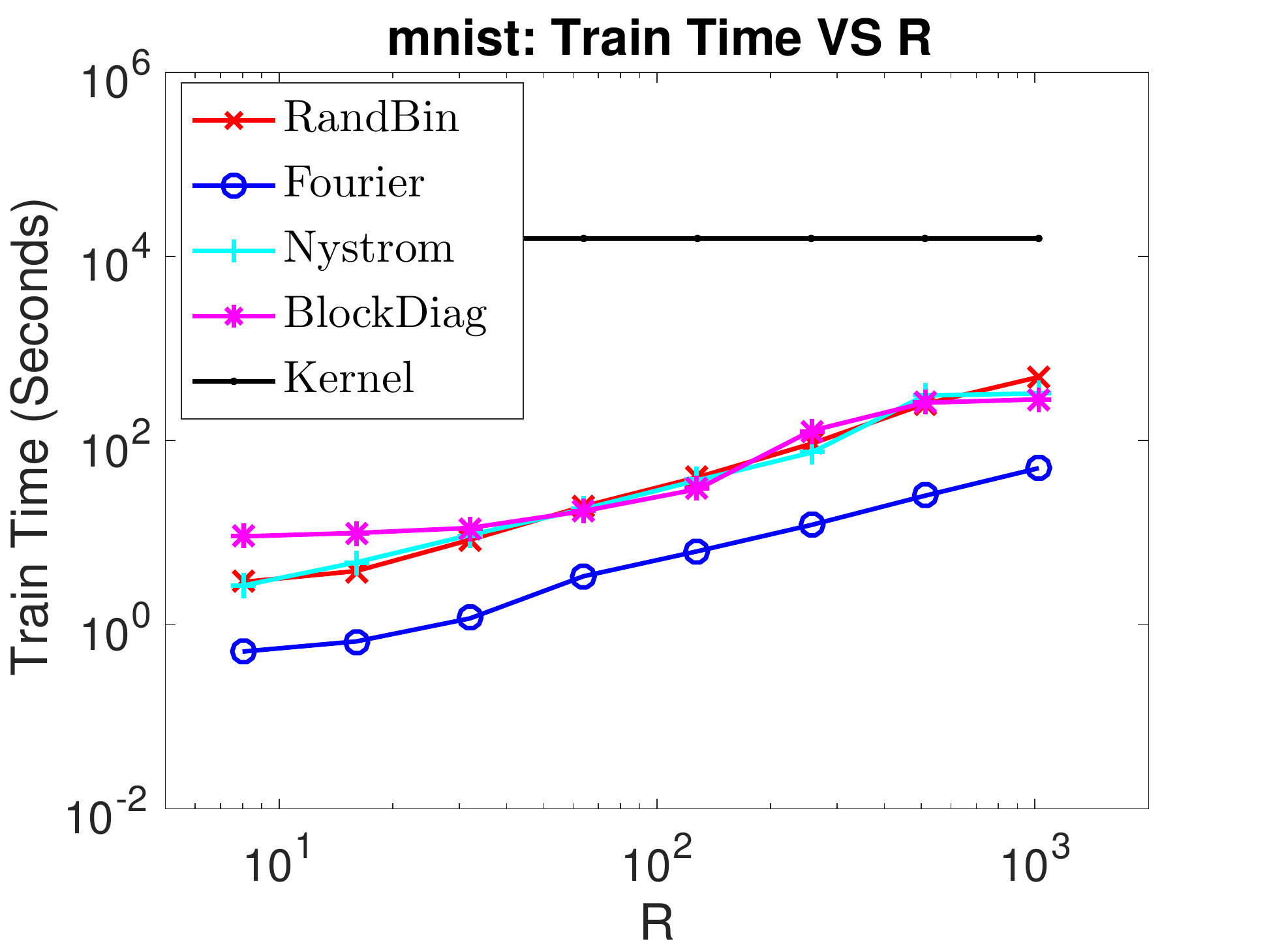}}
\subfigure[mnist]{\includegraphics[width = 1.55in]{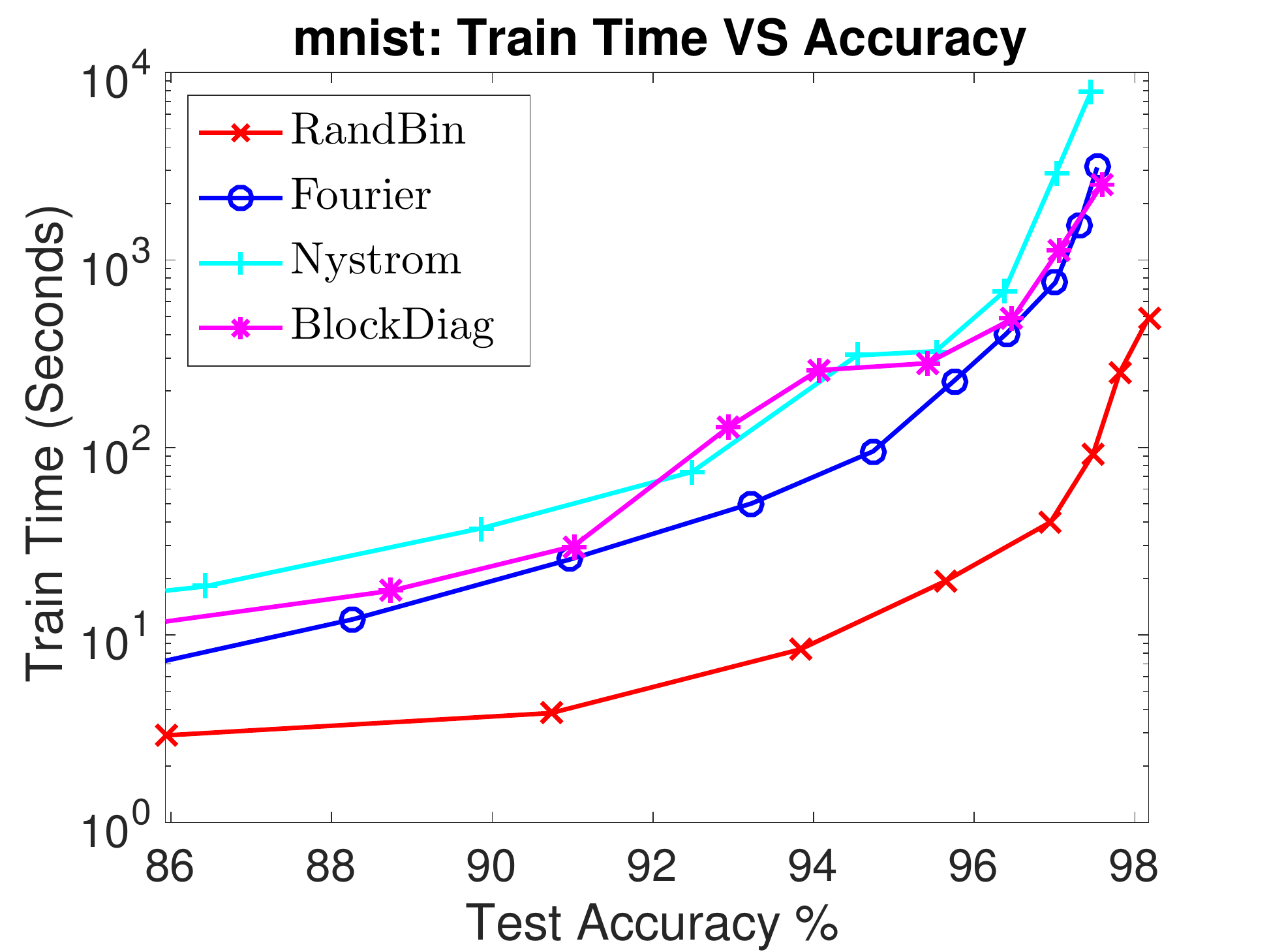}}
\subfigure[mnist]{\includegraphics[width = 1.55in]{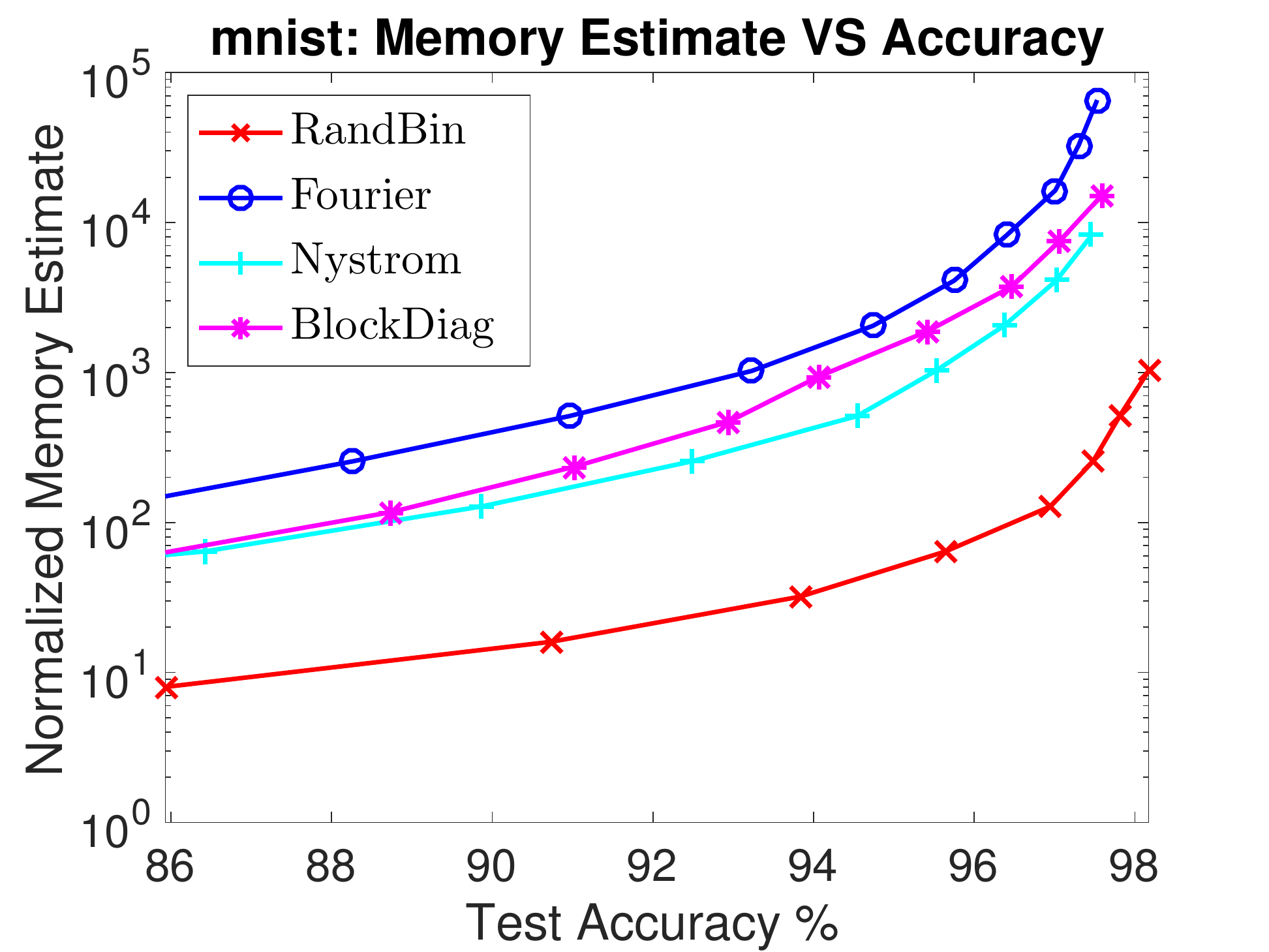}}
\subfigure[acoustic]{\includegraphics[width = 1.55in]{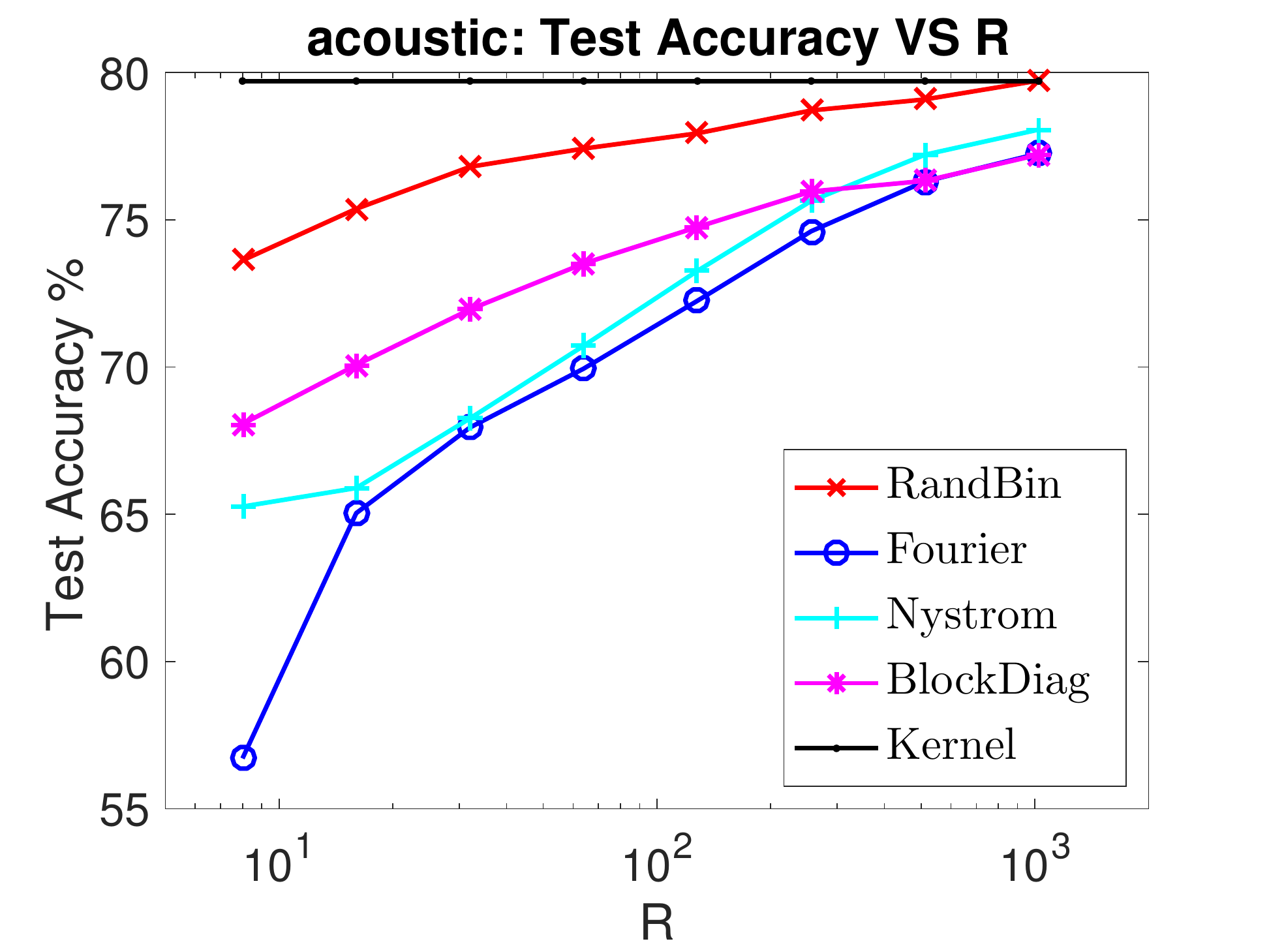}}
\subfigure[acoustic]{\includegraphics[width = 1.55in]{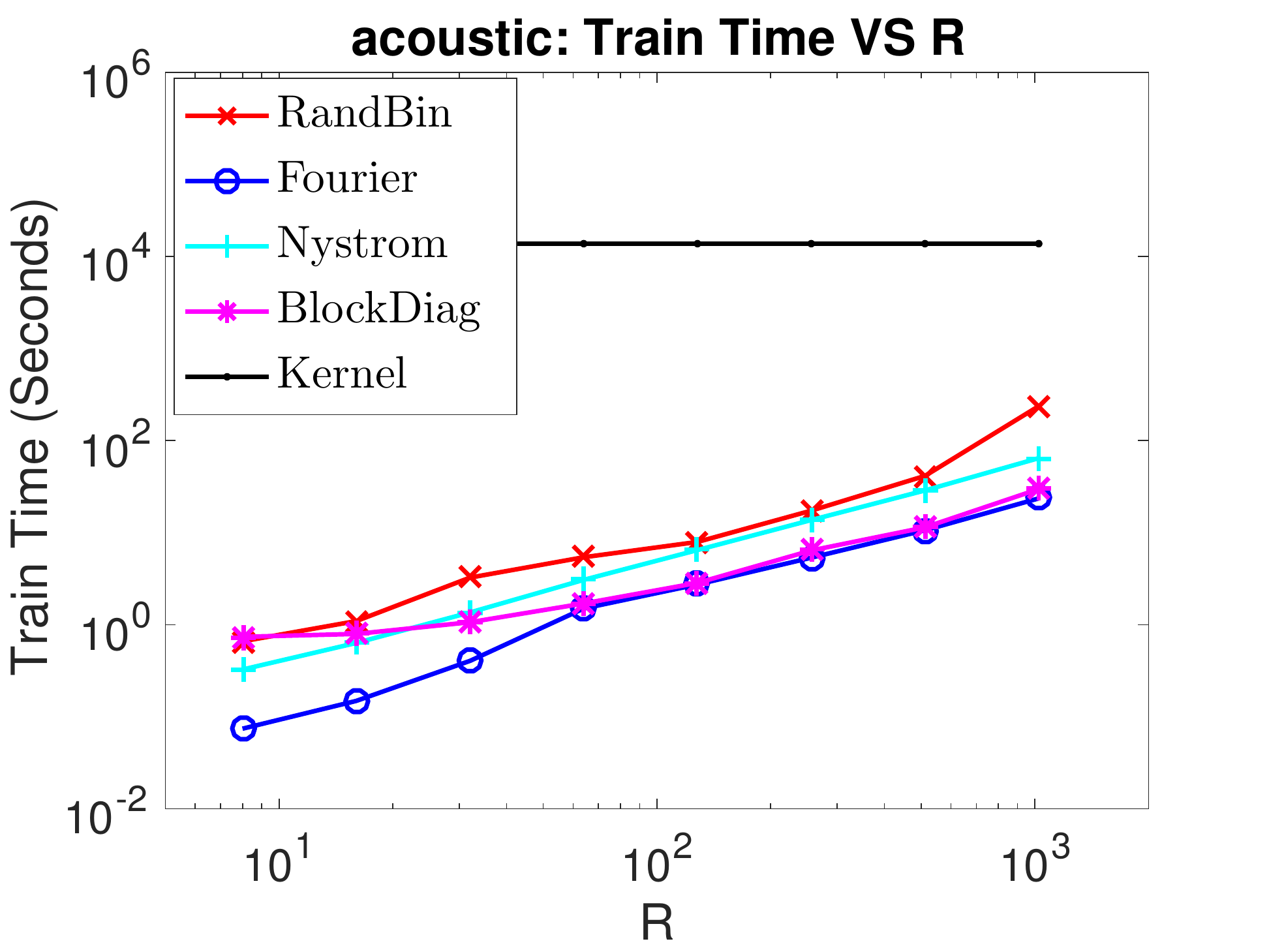}}
\subfigure[acoustic]{\includegraphics[width = 1.55in]{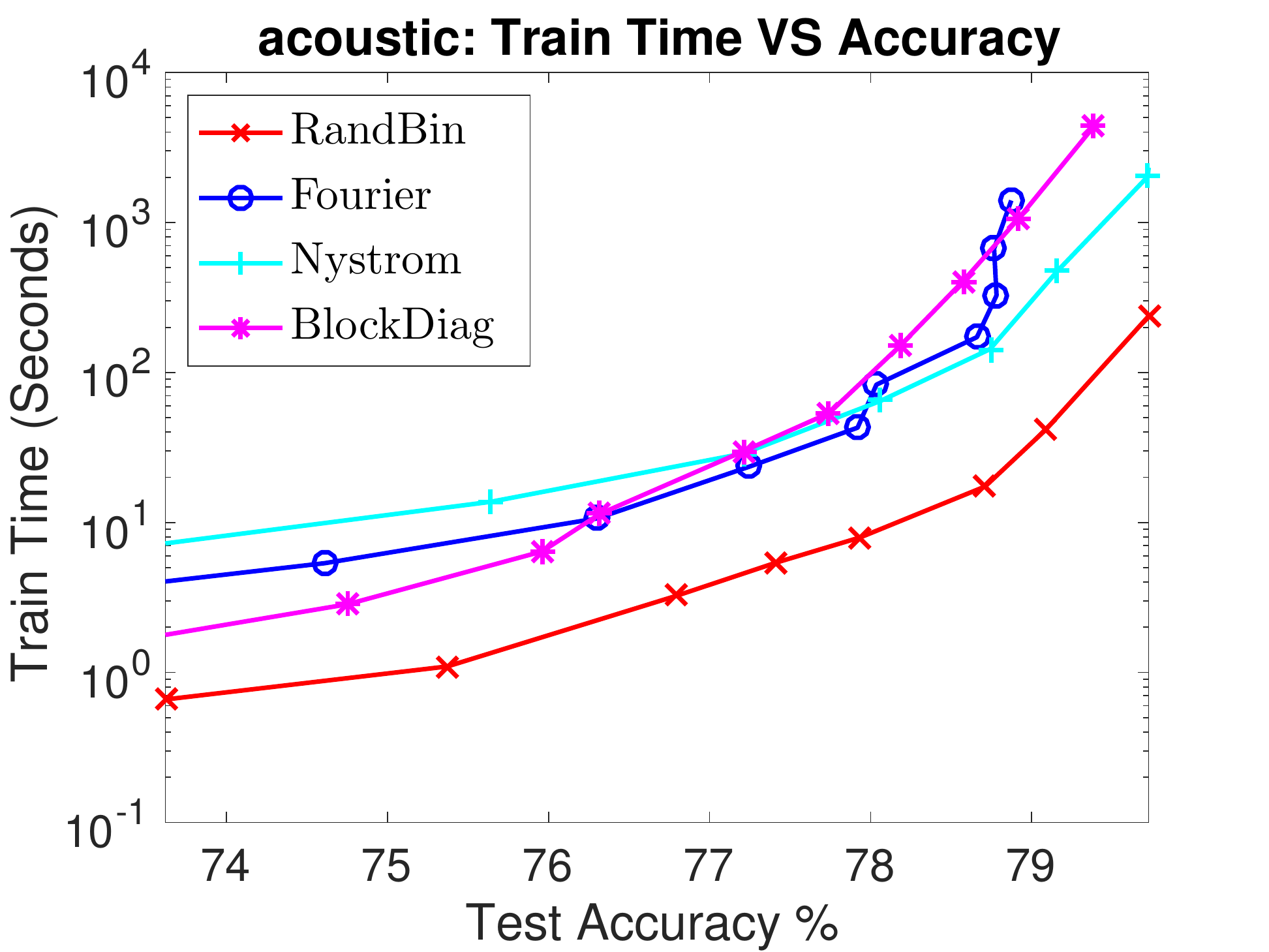}}
\subfigure[acoustic]{\includegraphics[width = 1.55in]{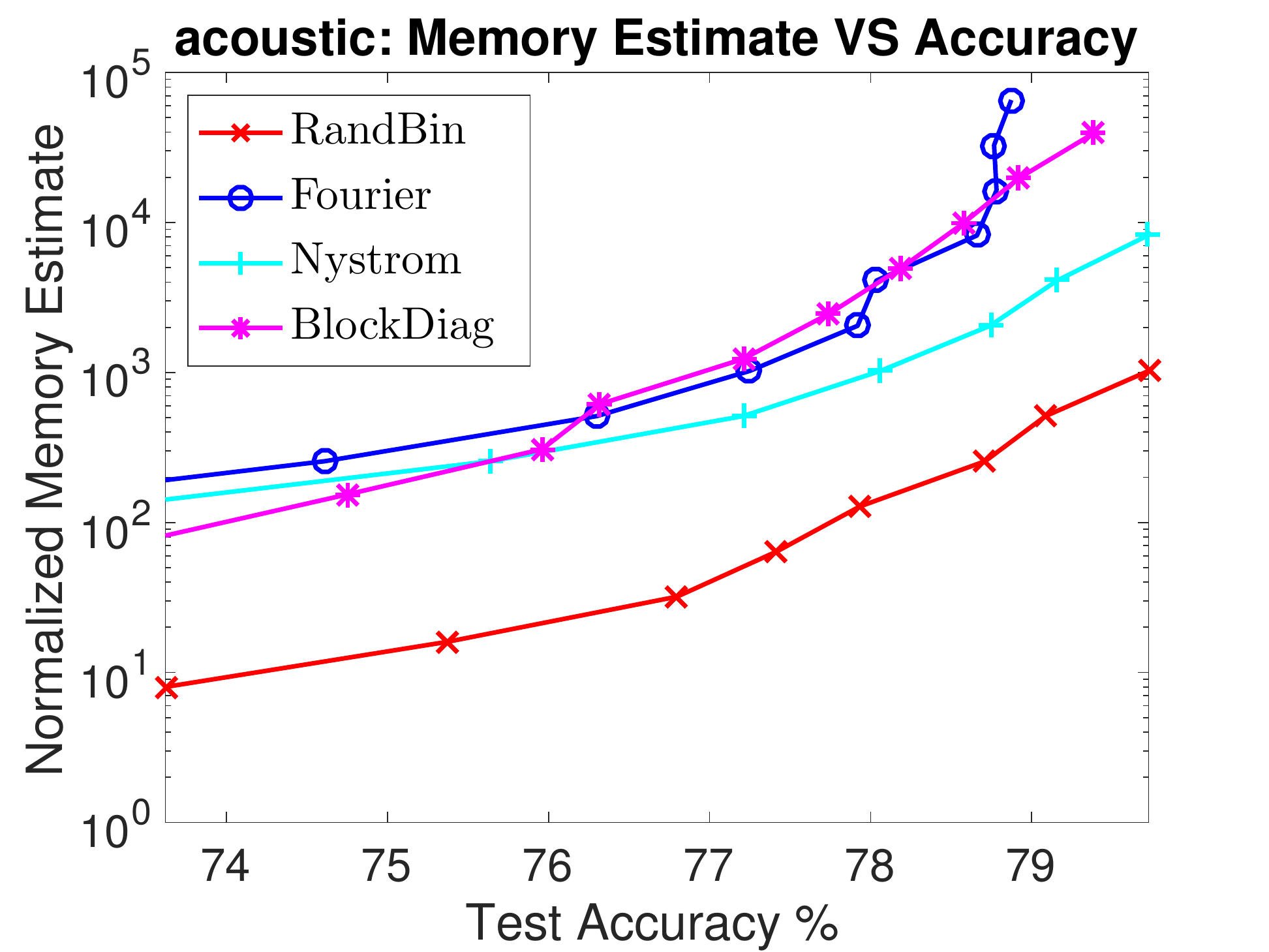}}
\caption{Comparisons among RB, RF, Nystr\"{o}m and Independent Block approximation. The first and second columns plot test performance and train time when increasing $R$. The third and fourth columns plot the train time and memory consumptions when achieving the desired test performance.}
\label{fig:perf_time_R_mem_time_perf_group1}
\end{figure*}

\subsection{Parallel Performance of Random Binning and Random Fourier}
We perform experiments to compare RB with RF when using RCD to solve L1-regularized Lasso and kernel SVM for both regression and binary classification problems. Since the goal is to demonstrate the strong parallel performance of RB, we implement the basic parallel implementation of RCD based on simple shared memory parallel programming model with OpenMP. We leave the high-performance distributed RCD implementation as one of the future works. We define the speedup of RCD on multicore implementation as follows: 
\[ speedup = 
\dfrac{\textit{runtime of RCD using single core}}{\textit{runtime using $P$ cores}}
\]
As shown in Fig.\ref{fig:parallel_rb_rf}, when the sparsity level of the feature matrix $Z$ is high, the near-linear speedup can be achieved \cite{marecek2014DBCD, liu2013asynchronous}. This is because the minimization problem can almost be separated along the coordinate axes, then higher degrees of parallelism are possible. In contrast, if $Z$ is lack of sparsity, then the penalty for data correlations slows the speedup to none. This is confirmed by no gain of parallel speedup of RF since $Z$ is always fully dense. Obviously, in order to empower strong parallel performance of RB, a very large $D$ is expected, which interestingly coincides with power of its faster convergence. Therefore, one can enjoy the double benefits of fast convergence and strong parallelizability of RB, which is especially useful for very large-scale problems.

\begin{figure*}[!htb]
\centering
\subfigure[ijcnn1]{\includegraphics[width = 1.55in]{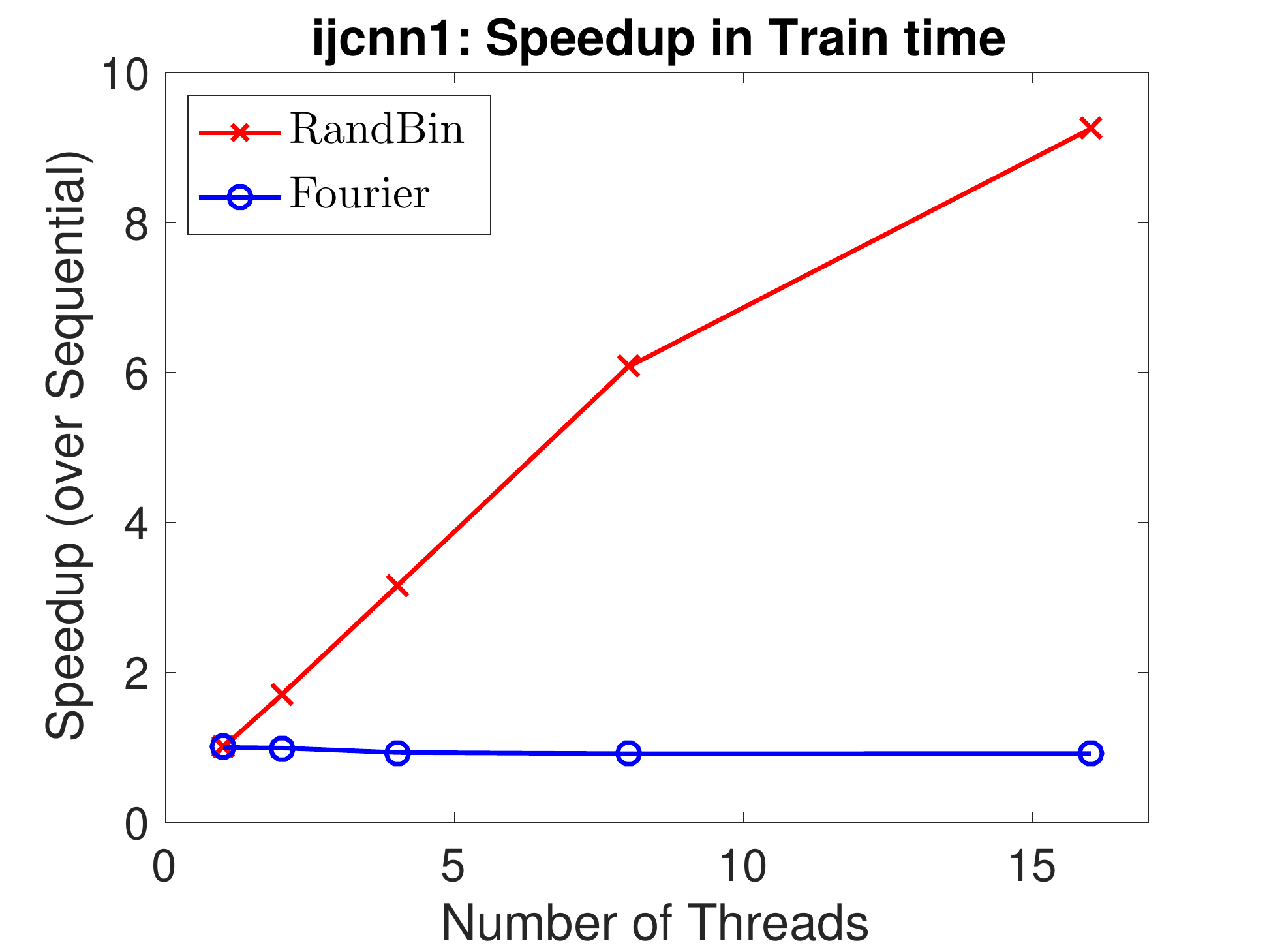}}
\subfigure[cod\_rna]{\includegraphics[width = 1.55in]{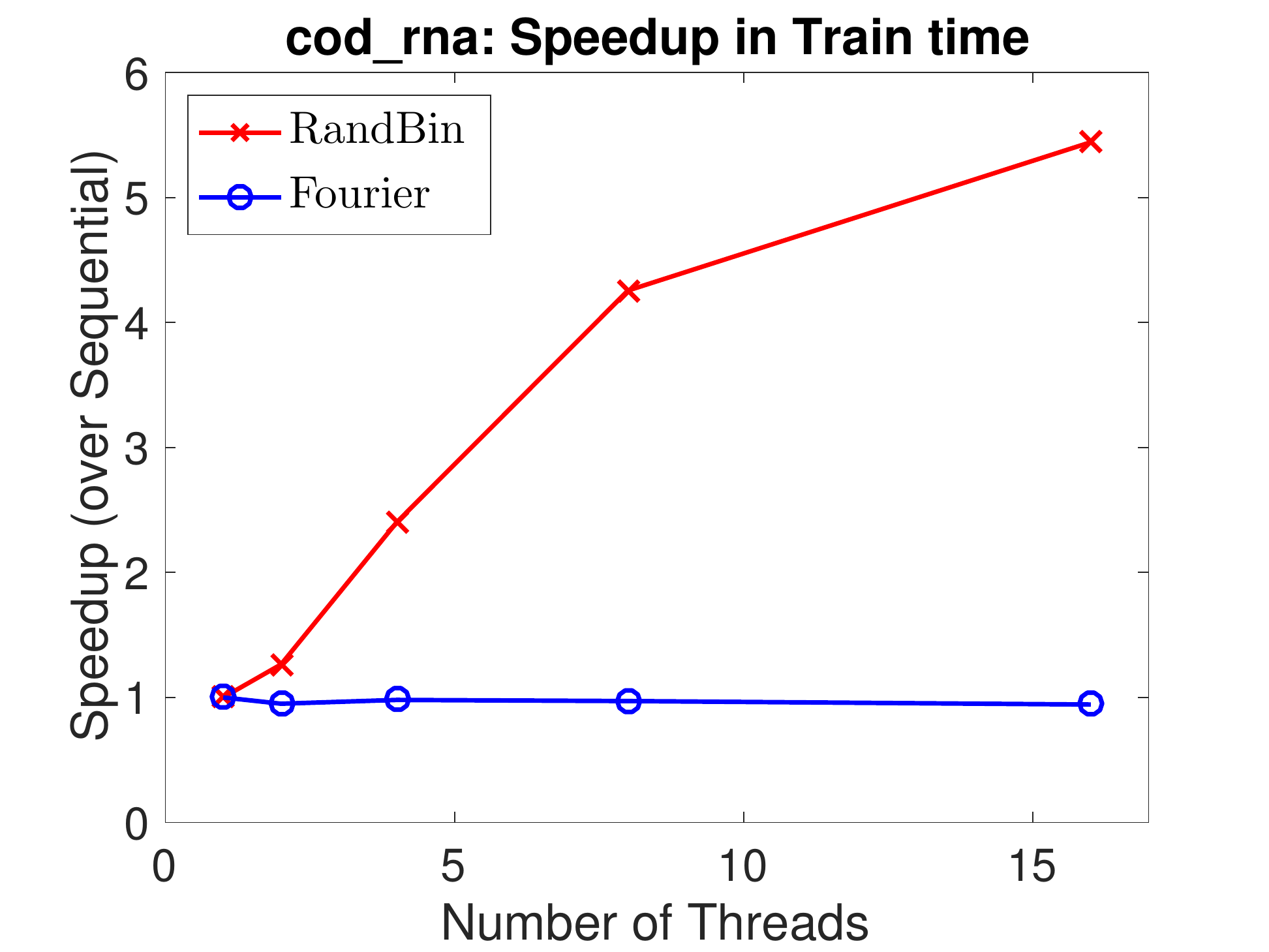}}
\subfigure[covtype]{\includegraphics[width = 1.55in]{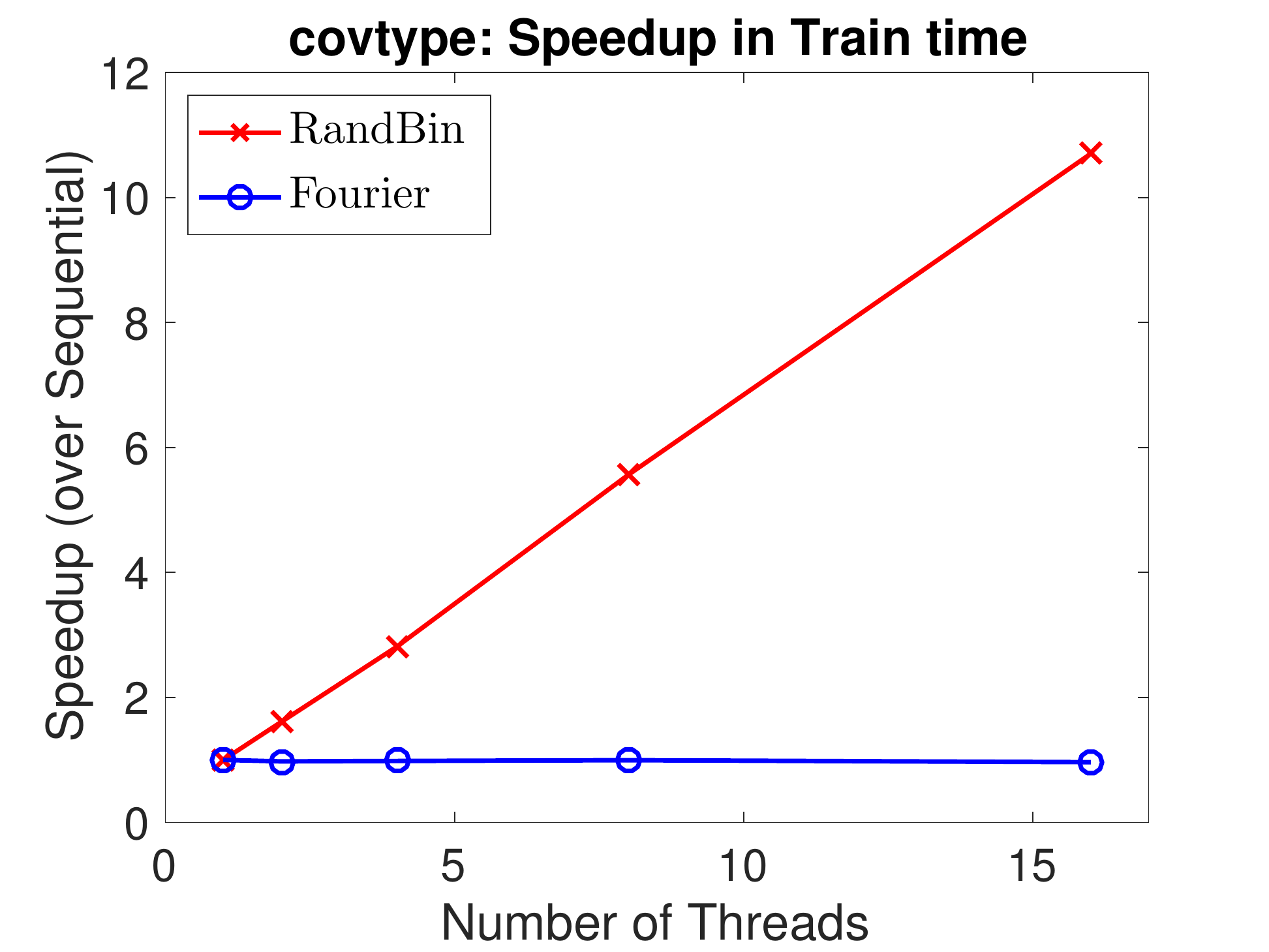}}
\subfigure[SUSY]{\includegraphics[width = 1.55in]{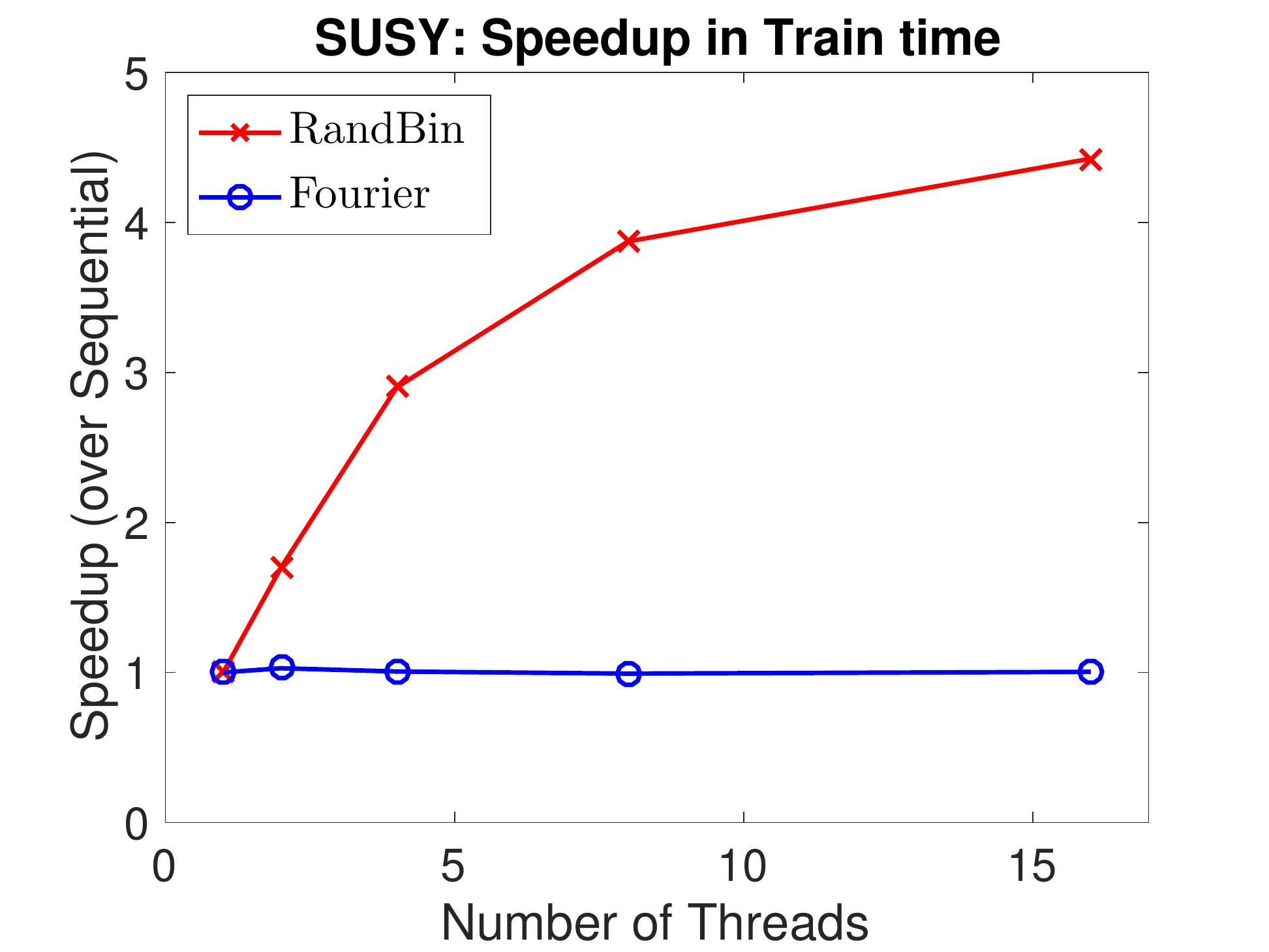}}
\caption{Comparisons of parallel performance between RB and RF using RCD when increasing the number of threads.}
\label{fig:parallel_rb_rf}
\end{figure*}

\section{Conclusions}
\label{section:conclusions}
In this paper, we revisit RB features, an overlooked yet very powerful random features, which we observe often to be orders of magnitude faster than other random features and kernel approximation methods to achieve the same accuracy. Motivated by these impressive empirical results, we propose the first analysis of RB from the perspective of optimization, to make a solid attempt to quantify its faster convergence, which is not captured by traditional Monte-Carlo analysis. By interpreting RB as a RBCD in the infinite-dimensional space, we show that by drawing $R$ grids with at least $\kappa$ expected number of non-empty bins per grid, RB achieves a convergence rate of $O(1/(\kappa R))$. In addition, in the L1-regularized setting, we demonstrate the sparse structure of RB features allows RCD solver to be parallelized with guaranteed speedup proportional to $\kappa$. Our extensive experiments demonstrate the superior performance of the RB features over other random feature and kernel approximation methods.

\section{Acknowledgement}
This work was done while L. Wu was a research intern at IBM Research. J. Chen is supported in part by the XDATA program of the Advanced Research Projects Agency (DARPA), administered through Air Force Research Laboratory contract FA8750-12-C-0323. 

\begin{flushleft}
{\small
\bibliographystyle{abbrv}
\bibliography{randomBinning}
}
\end{flushleft}

\end{document}